\documentclass[nohyperref]{article}
\pdfoutput=1
\usepackage{microtype}
\usepackage{graphicx}
\usepackage{subfigure}
\usepackage{booktabs} 

\usepackage{hyperref}

\usepackage[accepted]{icml2022}

\usepackage{amsmath}
\usepackage{amssymb}
\usepackage{mathtools}
\usepackage{amsthm}

\usepackage[capitalize,noabbrev]{cleveref}

\theoremstyle{plain}
\newtheorem{theorem}{Theorem}[section]
\newtheorem{proposition}[theorem]{Proposition}

\theoremstyle{definition}
\newtheorem{definition}[theorem]{Definition}

\theoremstyle{remark}

\newcommand{\srm}{\mathcal{L}}

 \newenvironment{myitemize}{\begin{list}{$\bullet$}
{\setlength{\topsep}{1mm}
\setlength{\itemsep}{0.25mm}
\setlength{\parsep}{0.25mm}
\setlength{\itemindent}{0mm}
\setlength{\partopsep}{0mm}
\setlength{\labelwidth}{15mm}
\setlength{\leftmargin}{4mm}}}{\end{list}}

\newcommand{\add}[1]{{ #1}}

\usepackage[textsize=tiny]{todonotes}

\icmltitlerunning{A Hierarchical Bayesian Approach to Inverse Reinforcement Learning with Symbolic Reward Machines}

\begin{document}

\twocolumn[
\icmltitle{A Hierarchical Bayesian Approach to Inverse Reinforcement Learning with Symbolic Reward Machine}




\begin{icmlauthorlist}
\icmlauthor{Weichao Zhou}{yyy}
\icmlauthor{Wenchao Li}{yyy}
\end{icmlauthorlist}

\icmlaffiliation{yyy}{Department of ECE, Boston University}

\icmlcorrespondingauthor{Weichao Zhou}{zwc662@bu.edu}
\icmlcorrespondingauthor{Wenchao Li}{wenchao@bu.edu}

\icmlkeywords{Machine Learning, ICML}

\vskip 0.3in
]



\printAffiliationsAndNotice{} 

\begin{abstract}

    A misspecified reward can degrade sample efficiency and induce undesired behaviors in reinforcement learning (RL) problems. We propose \textit{symbolic reward machines} for incorporating high-level task knowledge when specifying the reward signals. Symbolic reward machines augment existing reward machine formalism by allowing transitions to carry predicates and symbolic reward outputs.
    This formalism lends itself well to inverse reinforcement learning, whereby the key challenge is determining appropriate assignments to the symbolic values from a few expert demonstrations.
    We propose a hierarchical Bayesian approach for inferring the most likely assignments such that the concretized reward machine can discriminate expert demonstrated trajectories from other trajectories with high accuracy. Experimental results show that learned reward machines can significantly improve training efficiency for complex RL tasks and generalize well across different task environment configurations.

\end{abstract}
\section{Introduction}\label{intro}

Reinforcement Learning (RL) agents rely on rewards to measure the utility of each interaction with the environment~\cite{mnih2015human,Silver:2016aa}. 
As the complexity of RL tasks increases, it becomes difficult for the agent to grasp the intricacies of the task solely from  goal-driven reward functions -- rewarding the agent only at the end of each episode~\cite{riedmiller2018learning,parr1998reinforcement}. 
Reward machine (RM) is a formalism proposed by~\cite{DBLP:journals/corr/abs-2010-03950} for representing a reward function as a finite-state automaton (FSA)~\cite{10.5555/1373322}. 
However, the design of RMs can quickly become cumbersome as the complexity of the tasks increases~\cite{MAL-001}. 
In this paper, we draw inspiration from  \textit{symbolic finite automaton} (SFA) and \textit{symbolic finite transducer} (SFT)~\cite{d39antoni2017the} and propose \textit{symbolic reward machines} (SRMs) which augment FSA-based RMs by allowing the internal state transitions of an RM to carry predicates and functions on the trajectory. In addition to improving interpretability and conciseness of the reward model, SRMs facilitate the expression of complex task scenarios and reward design patterns. 

Given a reward structure, such as an RM, determining the appropriate reward assignments for individual conditions can be challenging and time-consuming~\cite{devidzeexplicable}. Ill-assigned rewards can undermine the effectiveness of the resulting reward functions~\cite{DBLP:journals/corr/abs-2111-00876}. 
For instance, if a learning agent is excessively awarded for the completion of a certain stage of a task, the agent may end up repeatedly completing the same stage to accumulate rewards instead of proceeding to the next stage, a phenomenon known as \textit{reward hacking}~\cite{DBLP:journals/corr/AmodeiOSCSM16}. 
We envision that, in a typical design routine of an SRM, a human engineer constructs the SRM to incorporate 
high-level task information, but leave the low-level details, such as the right amount of reward for a specific event, empty or as \textit{holes}. 
The SRM formalism also facilitates the specification of \textit{symbolic constraints} over the holes for capturing important task-specific nuances. 

Another contribution of this paper is a novel learning-based approach for 
\textit{concretizing} the holes in an SRM.
Similar to other inverse reinforcement learning settings~\cite{fu2017learning,DBLP:journals/corr/FinnLA16,Ng:2000:AIR:645529.657801,DBLP:journals/corr/FinnCAL16}, 
our approach makes use of \textit{example trajectories} demonstrated by a human expert. 
We leverage the \textit{generative adversarial} approaches from~\cite{DBLP:journals/corr/FinnCAL16,NEURIPS2018_943aa0fc} to construct a discriminator with a neural-network reward function to distinguish the expert trajectories from the trajectories of an agent policy. However, our approach works in a hierarchical Bayesian manner~\cite{MAL-001}. 
To circumvent the issue of \textit{non-differentiability} of SRMs, we employ a sampler to sample candidate instantiations of the holes to concretize the SRM. 
We introduce a stochastic reward signal as the latent variable dependent on the output of the concretized SRM and use a neural-network reward function to perform importance sampling of the stochastic rewards for trajectory discrimination. 
We summarize our contributions below.
\begin{myitemize}
\item We propose SRMs as a new structured way to represent reward functions for RL tasks.
\item We develop a hierarchical Bayesian approach that can concretize an SRM by inferring appropriate reward assignments from expert demonstrations. 
\item Our approach enables RL agents to achieve state-of-the-art performance on a set of complex environments with only a few demonstrations. In addition, we show that an SRM concretized in one environment generalizes well to other environment configurations of the same task. 
\end{myitemize}


\section{Related Work}

\noindent\textbf{Inverse Reinforcement Learning}. 
We first note that the IRL formulation proposed in \cite{Ng:2000:AIR:645529.657801,Abbeel:2004:ALV:1015330.1015430} has an infinite number of solutions. The Max-Entropy IRL from \cite{Ziebart:2008:MEI:1620270.1620297}, Max-Margin IRL from \cite{Abbeel:2004:ALV:1015330.1015430,Ratliff:2006:MMP:1143844.1143936} and Bayesian IRL from \cite{ramachandran2007bayesian} aim at resolving the ambiguity of IRL. However, those approaches restrict the reward function to be linear on the basis of human designed feature functions. 
Deep learning approaches proposed in~\cite{fu2017learning,ho2016generative,NEURIPS2018_943aa0fc,DBLP:journals/corr/FinnCAL16} have substantially improved the scalablity of IRL by drawing a connection between IRL and Generative Adversarial Networks (GANs) introduced by \cite{goodfellow2014generative}. Our work, while embracing the data-driven and generative-adversarial ideologies, further extends IRL to cope with symbolically represented human knowledge. 

\noindent\textbf{Reward Design}. {There have been substantial efforts on enriching the information in reward functions. Reward shaping proposed by \cite{Ng:1999:PIU:645528.657613} adds state-based potentials to the reward in each state. Exploration driven approaches such as \cite{bellemare2016unifying,pathak2017curiosity,DBLP:journals/corr/abs-1708-08611,flet-berliac2021adversarially} incentivize agents with intrinsic rewards. Compared with these methods, we do not seek to generate reward functions densely ranging over the entire state space but rather design intepretable ones that selectively or even sparsely produce non-zero rewards.}
Reward machines from~\cite{DBLP:journals/corr/abs-2010-03950} directly represent the reward functions as FSAs. The symbolic reward machine in our work is also automata-based but augments RMs in a similar way to SFA for FSA~\cite{10.1145/2103621.2103674}.  
There have been efforts on learning a so-called perfect RM as termed in~\cite{toro2019learning} from the experience of an RL agent in partially observable environment. However, the RM is still based on FSA and the rewards are still manually assigned.  
Regarding leveraging human demonstrations, inverse reward design (IRD) proposed in \cite{NIPS2017_32fdab65} is analogous to IRL  but aims at inferring a true reward function from some proxy reward function perceived by a RL agent.  
Safety-aware apprenticeship learning from~\cite{zhou2018safety} pioneers the incorporation of formal verification in IRL. {However, those works confine the reward functions to be linear of features as the generic IRL does. Our work does not have such limitations.} 
The paradigm proposed in \cite{zhou2021programmatic} is the first to express reward functions using programs.
Our work differs from theirs in two aspects: first, our reward design is based on automata; second, we propose a hierarchical Bayesian approach for inferring reward assignments in the automata.  

\noindent\textbf{Hierarchical and Interpretable Reinforcement Learning}.
Hierarchical RL (HRL)~\cite{DBLP:journals/corr/abs-1803-00590} combines high-level and low-level policies to handle sub-goals in complex RL tasks. 
Our work is similar to HRL in terms of the level of human efforts involved. 
However, one key difference between our work and HRL is that we train a single policy for the entire task instead of multiple policies for each of the sub-goals.
Researches in interpretable RL have been focused on designing interpretable policies~\cite{andre2001programmable,andre2002state,verma2018programmatically,DBLP:journals/corr/abs-1907-07273,DBLP:journals/corr/abs-2102-11137,tian2020learning}. This paper concerns the design of interpretable reward functions rather than interpretable policies. Our motivation is that a well-designed reward function is transferable and can be a powerful complement to the vast literature on RL policy learning. 
\section{Background}\label{prelim}

An RL environment is a tuple $\mathcal{M}=\langle \mathcal{S, A, P}, d_0\rangle$ where $\mathcal{S}$ is the state space; $\mathcal{A}$ is an action space; $\mathcal{P}(s'|s, a)$ is the probability of reaching a state $s'$ by performing an action $a$ at a state $s$; $d_0$ is an initial state distribution. A \textit{policy} $\pi(a|s)$ determines the probability of an RL agent performing an action $a$ at state $s$. By successively performing actions for $T$ steps after initializing from a state $s^{(0)}\sim d_0$, a \textit{trajectory} $\tau=s^{(0)}a^{(0)}s^{(1)}a^{(1)}\ldots s^{(T)}a^{(T)}$ is produced. A state-action based \textit{reward function} is a mapping $f:S\times A\rightarrow \mathbb{R}$ to the real space. With a slight abuse of notations, we denote the total reward along a trajectory $\tau$ as $f(\tau)=\sum^{T}_{t=0} f(s^{(t)}, a^{(t)})$ and similarly for the joint probability $p(\tau|\pi)=\prod^{T-1}_{t=0}\mathcal{P}(s^{(t+1)}|s^{(t)},a^{(t)})\pi(a^{(t)}|s^{(t)})$ of generating a trajectory $\tau$  by following $\pi$.  The objective of entropy-regularized RL is to maximize $J_{RL}(\pi)=\mathbb{E}_{\tau\sim\pi}[f(\tau)]+\mathcal{H}(\pi)$ where $\tau\sim\pi$ is an abbreviation for $\tau\sim p(\tau|\pi)$ and $\mathcal{H}(\pi)$ is the expected entropy of $\pi$.
When the reward function is unknown but a set of expert trajectories $\tau_E$ is sampled with some expert policy $\pi_E$,
{GAIL~\cite{ho2016generative} trains an agent policy $\pi_A$ as a generator to match $\pi_E$ by minimizing $J_{adv}$ in Eq.\ref{eq1_1} via RL algorithms such as PPO ~\cite{DBLP:journals/corr/SchulmanWDRK17}. Adversarially, GAIL optimizes a discriminator $D:\mathcal{S}\times\mathcal{A}\rightarrow [0,1]$ to accurately identify $\tau_E$'s from $\tau_A\sim \pi_A$ by maximizing $J_{adv}$. From a probabilistic inference perspective, Bayesian GAIL from~\cite{NEURIPS2018_943aa0fc} labels any expert trajectory $\tau_E$ with $1_E$ and $0_E$ to respectively indicate $\tau_E$ as being sampled from an expert demonstration set $E$ and from some agent policy $\pi_A$. Likewise, the trajectory $\tau_A$ of $\pi_A$ is labeled with $1_A$ and $0_A$ for the same indications. Assuming that the labels $0_A, 1_E$ are known \textit{a priori}, Bayesian GAIL solves the most likely discriminator $D$  by maximizing $p(D|0_A, 1_E; \pi_A; E)\propto p(D)p(0_A, 1_E|\pi_A, D; E)\propto\sum_{\tau_A} p(\tau_A|\pi_A)p(0_A|\tau_A; D)\sum_{\tau_E} p(\tau_E|E) p(1_E|\tau_E;D)$ of which the logarithm as in Eq.\ref{eq1_0} is lower-bounded due to Jensen's inequality by Eq.\ref{eq1_1}.} It is further proposed in \cite{fu2017learning} that by representing $D(s,a)=\frac{\exp(f(s,a))}{\exp(f(s,a)) + \pi_A(a|s)}$ with a neural network $f$,  when Eq.\ref{eq1_1} is maximized, it holds that $f\equiv \log \pi_E$ and $f$ equals the expert reward function which $\pi_E$ is optimal w.r.t, given that $\forall \tau.p(\tau|E)\approx p(\tau|\pi_E)$. {Hence, by representing the $D$ in Eq.\ref{eq1_0} and Eq.\ref{eq1_1} with $f$, an objective of solving the most likely expert reward function $f$ is obtained.} 
\begin{eqnarray}
&\log \sum\limits_{\tau_A} p(\tau_A|\pi_A)p(0_A|\tau_A; D)\sum\limits_{\tau_E} p(\tau_E|E) p(1_E|\tau_E;D)&\ \label{eq1_0}\\
&\geq\underset{{\tau_E\sim E}}{\mathbb{E}}\Big[\log \prod\limits^T_{t=0}D(s_E^{(t)},a^{(t)}_E)\Big]+&\nonumber\\
&\underset{{\tau_A\sim\pi_A}} {\mathbb{E}}\Big[\log \prod\limits^T_{t=0}(1-D(s_A^{(t)},a^{(t)}_A))\Big]:=J_{adv}(D)&\label{eq1_1}
\end{eqnarray}

\section{Motivating Example}
    
\begin{figure*}[ht]
     \centering
    \subfigure[]{
        \includegraphics[height=2.2cm, width=1.4cm]{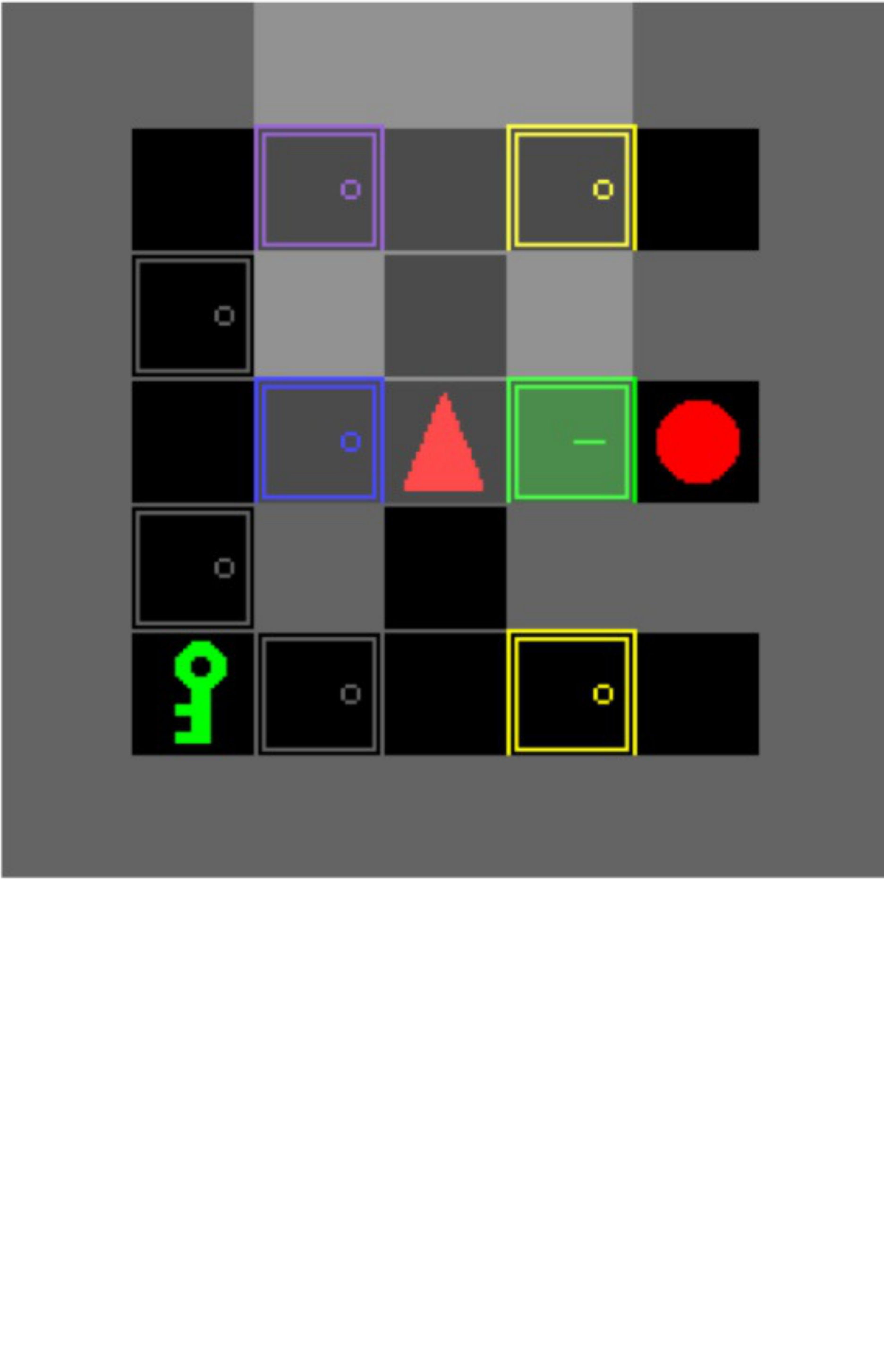}
            }
        \subfigure[]{
         \includegraphics[height=3.5cm,width=11.cm]{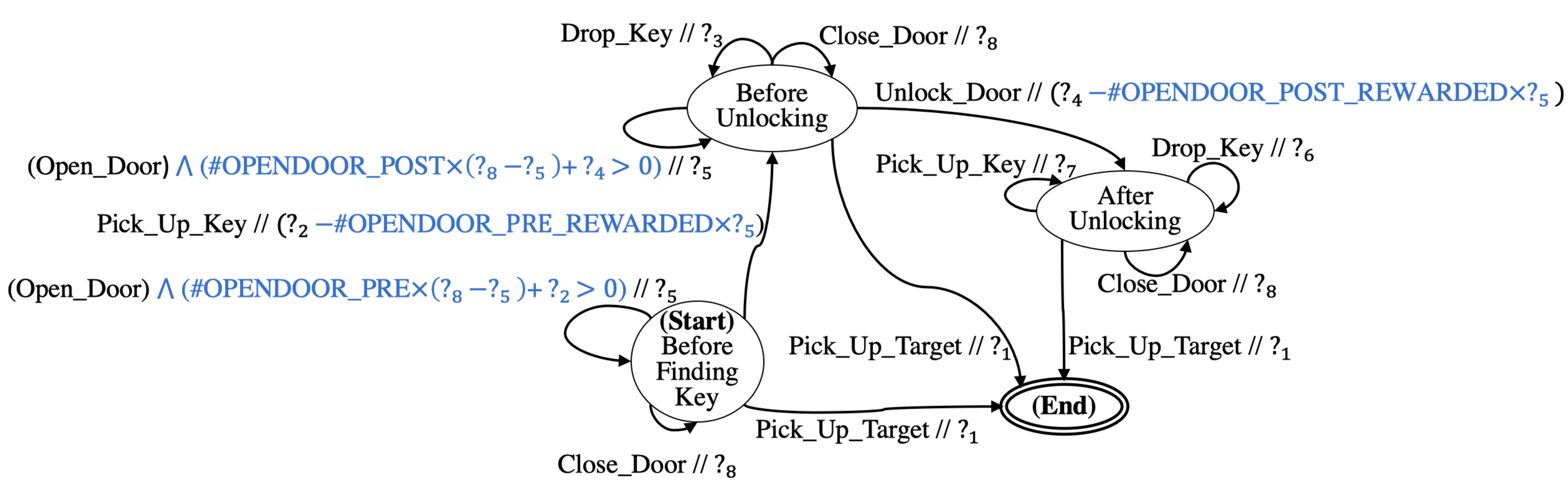}
        }
     \subfigure[]{
        \includegraphics[height=3.cm, width=3.8cm]{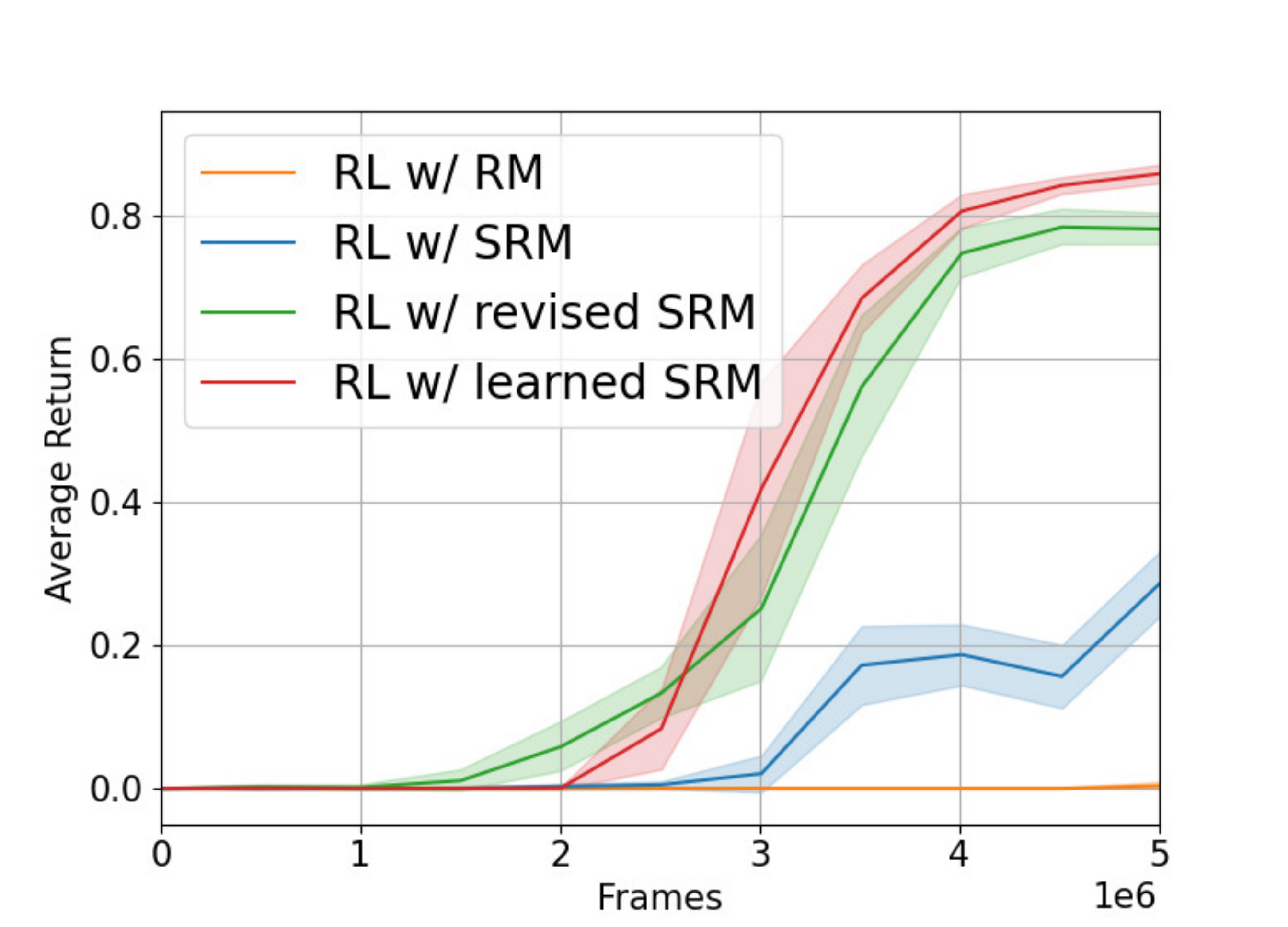}
            }
     \caption{a) A $7\times7$ KeyCorridorS3R3 task; b) The diagram excluding the blue terms illustrates an RM designed for the KeyCorridor task. The ellipses indicate the states of the RM.
The directed edges indicate the state transitions in the RM. 
 The initial state is the one that contains ``\textbf{(Start)}". When the state ``\textbf{(End)}" is reached, the task is finished. By adding the blue terms, the RM becomes an SRM; c) The y-axis is the average return measured in terms the default goal-driven reward, and the x-axis is the number of interactions with the environment. The orange curve is for the RL agent trained with RM; the blue curve is for the SRM after adding the blue terms; the green curve is for the revised SRM; the red curve is for the SRM induced by our learning approach. 
     }     \label{fig1_1}\label{fig1_2}
\end{figure*}

We first motivate the use of SRMs with a task from a Mini-Grid environment introduced in \cite{gym_minigrid}. We highlight that the human insights incorporated in the SRM can hardly be realized with conventional goal-driven reward mappings and reward machines designed based on sub-tasks.

Fig.\ref{fig1_1}(a) shows an $7\times7$ KeyCorridor task in the Mini-Grid environment. An agent needs to find a grey key that is hidden in a room, unlock the grey door of another room and drop the key to pick up the blue target object. In every step, the agent can observe at most the $7\times 7$ tiles in front of it if the tile is not obstructed by walls and doors. By default, the environment only returns a reward when the agent reaches the goal tile. 
Fig.\ref{fig1_2}(b) shows an RM (as an FSA) for this task and the modifications (in blue) for an SRM built on top of this RM. 
\add{
Each edge in Fig.\ref{fig1_2}(b) is annotated as $\mathtt{\psi//r}$ where $\psi$ is a transition predicate and $r$ is a reward value. 
When the predicate $\psi$ on the transient step is $\mathtt{True}$, a reward $r$ is returned to the agent and the (S)RM transitions to the next state. If none of the depicted transitions is enabled in a state, the (S)RM stays in the same state and the output reward is $0$. 
For the SRM, the symbols $\mathtt{\{?_{id}\}^8_{id=1}}$ indexed by $\mathtt{id}$'s are free variables, which we dub \textit{holes}, whose values are to be determined. 
The terms such as $\mathtt{Reach\_Goal}$ are atomic propositions 
over events occurring at the current time step.
Note that RM only allows atomic propositions and constant rewards. 
With an handcrafted reward assignment $\mathtt{?_1=1,?_5=?_6=0.1,?_4= 0.5,?_3=?_2=0}$ and $\mathtt{ ?_7=?_8=-0.1}$ 
which rewards opening door, unlocking door, then dropping key and picking up target, while penalizing closing door and picking up key after unlocking the door.
RL agents trained using this RM cannot achieve any performance after 5E6 steps
as shown in  Fig.\ref{fig1_2}(c).
We then construct an SRM by augmenting this RM with functional terms, which 
count the number of occurrence of certain events (as described next to the $\#$-sign), and additional predicates built using these  terms.
For instance, 
$\mathtt{\#OPPENDOOR\_PRE}$ counts the number of times that \text{Open\_Door} is evaluated to $\mathtt{True}$ before finding the key. 
Using the same set of handcrafted reward assignments, we can see that agents trained using the SRM reward can achieve better performance than training with the RM reward. 
If further revising the SRM by letting $?_2=0.1$ and $?_6=0$, we show in Fig.\ref{fig1_1}(c) that the agents trained using the revised SRM achieve even better performance. 
Manually determining the reward assignments can be difficult and time-consuming.
In this paper, we propose a learning approach to infer appropriate assignments for the $\mathtt{?_{id}}$'s from expert demonstrated trajectories.
Fig.\ref{fig1_1}(c) shows that the same SRM but with inferred rewards can efficiently train the RL agent to attain the highest performance.
We formalize the concept of SRM and formulate the reward inference problem in the next section.}


\section{Symbolic Reward Machines (SRMs)}
In this section, we give a formal definition of SRM and formulate the problem of concretizing SRMs.
\subsection{Definition}
The definition of SRM is inspired from those of SFA and SFT in~\cite{10.1145/2103621.2103674}. To adapt them to the RL setting, we assume a background theory equipped with fixed interpretations on the state and actions in the RL environment as well as a language of functions. Following \cite{10.5555/509043}, a $\lambda$-term is a function written in the form of $\lambda x.\rho$. The type of a $\lambda x.\rho$ is a mapping from the type of input argument $x$ to the type of function body $\rho$. The free variable set $FV(\rho)$ of $\rho$ is the set of symbolically denoted variables appearing in $\rho$ while not appearing in the input of any $\lambda$-term inside $\rho$. Given some concrete input $\hat{x}$, the evaluation of $\lambda x.\rho$ is written as $[\![\lambda x.\rho]\!](\hat{x})$ or $[\![\rho[\hat{x}/x]]\!]$ where $[\hat{x}/x]$ represents the replacement of $x$ with $\hat{x}$ in $\rho$. The denotation stays the same if $x$ and $\hat{x}$ are not unary. A predicate is a specific set of $\lambda$-terms mapping to Boolean type $\mathbb{B}=\{\top, \bot\}$ where $\top, \bot$ mean \textit{True} and \textit{False} respectively. 
The set of predicates is closed under Boolean operations $\wedge,\vee,\neg$. 
\begin{definition}
Given an RL environment $\mathcal{M}=\langle \mathcal{S, A, P}, d_0\rangle$, a \textbf{symbolic reward machine (SRM)} is a tuple $\srm=\langle\mathcal{Q}, \Psi, \mathcal{R}, \delta,  q_0, Acc\rangle$ where $\mathcal{Q}$ is a set of internal states; $\Psi\subseteq (\mathcal{S\times A})^*\rightarrow \mathbb{B}$ is a set of predicates on trajectories in $\mathcal{M}$; $\mathcal{R}$ is a set of $\lambda$-terms of type $(\mathcal{S\times A})^*\rightarrow \mathbb{R}$; $\delta$ is a set of transition rules $(q, \psi, r, q')$, where $q,q'\in \mathcal{Q},\psi\in\Psi, r\in \mathcal{R}$; $q_0\in \mathcal{Q}$ is an initial state; $Acc\subseteq\mathcal{Q}$ is a set of accepting states; the free variables set of $\mathcal{L}$ is defined as $FV(\mathcal{L})=\bigcup\limits_{\rho\in\mathcal{R}\cup{\Psi}}FV(\rho)$.
\end{definition}

We use the notation $q\xrightarrow{\psi//r}q'$ for a rule $(q,\psi,r,q')\in \delta$ and call $\psi$ its guard. The input to an SRM is a trajectory $\tau\in (\mathcal{S\times A})^*$. A rule $q\xrightarrow{\psi//r} q'$ is applicable at $q$ iff $[\![\psi]\!](\tau)=\top$, in which case $\srm$ outputs a reward $[\![r]\!](\tau)$ for the last state-action pair in $\tau$ while the state $q$ transitions to $q'$. If no rule is applicable, the state $q$ does not transition and $\srm$ outputs a fixed constant reward such as $0$, in which case we dub a dummy transition $q\xrightarrow{\underset{\exists q'\in\mathcal{Q}.(q,\psi,r,q')\in \delta}{\bigwedge(\neg\psi)}//0}q$. An SRM $\srm$ is called deterministic if given an input trajectory $\tau$ at any state $q\in\mathcal{Q}$, $|\{(q,\psi, r, q')\in\delta|\psi(\tau)=\top\}|\leq 1$ in which case with a little abuse of notations we write the next state as $q'=\delta(q, \tau)$ either obtained from the uniquely applicable transition rule $(q,\psi, r,q')\in \delta$ s.t. $\psi(\tau)=\top$, or from a dummy transition such that $q'=q$. To deploy a deterministic SRM in an RL task is to construct a synchronous product as defined below.


\begin{definition}
A synchronous product between an $\mathcal{M}$ and a deterministic $\srm$ is a tuple $\mathcal{M}\otimes \srm=\langle (\mathcal{S}\times\mathcal{A})^*\times \mathcal{S}\times \mathcal{Q}, \mathcal{A}, \Psi, \mathcal{R}, \mathcal{P}\odot\delta,d_0, q_0, Acc\rangle$ where a product state in $(\mathcal{S}\times\mathcal{A})^* \times\mathcal{S}\times \mathcal{Q}$ is a pair $(\tau::s, q)$ where $::$ means concatenation; $\tau::s$ means concatenating a trajectory $\tau\in(\mathcal{S}\times\mathcal{A})^* $ with an $\mathcal{M}$ state $s\in\mathcal{S}$; $q\in\mathcal{Q}$ is an $\srm$ state; the product transition rule $\mathcal{P}\odot\delta$ follows Eq.\ref{eq4_1} where $\tau::s::a$ is a trajectory resulted from further concatenating $\tau::s$ with an action $a\in\mathcal{A}$; the initial product state is $(s_0, q_0)$ where $s_0\sim d_0$; the rest follows the definitions in $\mathcal{M}$ and $\srm$.
\begin{eqnarray}
&&(\mathcal{P}\odot\delta)((\tau::s, q), a, (\tau::s::a::s', q'))\nonumber\\
&=&\begin{cases}
\mathcal{P}(s'|s,a)& if\ q'=\delta(q, \tau::s::a)\\
0& Otherwise
\end{cases}\label{eq4_1}
\end{eqnarray}
\end{definition}
\begin{proposition}
Suppose that $\srm$ is deterministic. Starting from any $q^{(0)}\in\mathcal{Q}$, as the input trajectory extends from $\tau=s^{(0)}a^{(0)}$ to $\tau=s^{(0)}a^{(0)}\ldots s^{(T)}a^{(T)}$, there can only be at most one path $\sigma=q^{(0)} q^{(1)}\ldots q^{(T+1)}$ obtained by applying the uniquely applicable or a dummy transition successively to $\tau$ from $t=0$ to $t=T$, i.e., $\forall t=0,1,\ldots,T.q^{(t+1)}=\delta( q^{(t)}, s^{(0)}a^{(0)}\ldots s^{(t)}a^{(t)})$.
\end{proposition}
Suppose that at time step $t$ there is a transition from $(\tau::s^{(t)}, q^{(t)})$ to $(\tau::s^{(t)}::a^{(t)}::s^{(t+1)}, q^{(t+1)})$ with $(\mathcal{P}\odot\delta)((\tau::s^{(t)}, q^{(t)}),a^{(t)}, (\tau::s^{(t)}::a^{(t)}::s^{(t+1)}, q^{(t+1)}))>0$. While $\mathcal{L}$ witnesses such product state transition, the RL agent only observes the transition $s^{(t)}a^{(t)}s^{(t+1)}$ in $\mathcal{M}$. Furthermore, $\mathcal{L}$ returns a reward $[\![r^{(t)}]\!](\tau::s^{(t)}::a^{(t)})$ at time step $t$, where $r^{(t)}$ is either $0$ or the $\lambda$-term associated with the rule $q^{(t)}\xrightarrow{\psi^{(t)}//r^{(t)}}_\srm q^{(t+1)}$. We further inductively define $[\![\srm]\!](\tau)$ as $[\![\srm]\!](\tau::s^{(t)}::a^{(t)})=[\![\srm]\!](\tau)::[\![r^{(t)}]\!](\tau::s^{(t)}::a^{(t)})$. For simplicity, we also write $\tau::s::a$ as $\tau::(s,a)$ as if $\tau$ is a list of $(s,a)$'s, if it does not raise ambiguity in the context. We denote $\srm(\tau)=\sum^{T-1}_{t=0}[\![\srm]\!](\tau[:t])$ where $\tau[:t]$ is the partial trajectory from initialization up until step $t$. 

For clarification, we remark that an SRM is an SFT only under stricter conditions, in which case the SRM retains all the properties of SFT, e.g., composability and decidability. While SFT emphasizes the acceptance of inputs, SRM emphasizes more on computing the rewards for the trajectories. Besides, SRM-based rewards are non-Markov but recent researches~\cite{DBLP:journals/corr/abs-2111-00876} prove that there exist various types of tasks that no Markov reward function can capture.

\subsection{Problem Formulation}
For a predicate $\lambda x.\rho$, if $\rho$ does not include trajectory $\tau$ but includes holes $\mathtt{?_{id}}$'s in its free variable set $FV(\rho)$, e.g., $\rho:=\mathtt{?_6+?_7\leq 0}$ in the previous example, such $\rho$'s can be potentially used as \textit{symbolic constraints}. When having a concrete value $\mathtt{h_{id}}$ for a hole $\mathtt{?_{id}}$, one can concretize $\mathcal{L}$ by replacing $\mathtt{?_{id}}$ with $\mathtt{h_{id}}$ in $\mathcal{L}$, written as $\mathcal{L}\mathtt{[h_{id}/?_{id}]}$. We define the problem of concretizing an SRM below.
\begin{definition}[\textbf{Symbolic Reward Machine Concretization}]
The concretization problem of a symbolic reward machine is a tuple $\langle \srm, \textbf{H},c\rangle$ where $\srm$ is an SRM with holes $\mathtt{\textbf{?}}=\{\mathtt{?_1, ?_2,\ldots}\}\subseteq FV(\srm)$; ${\textbf{H}}={H}_1\times {H}_2\ldots$ with each $H_{\mathtt{id}}$ being the assignment space of $\mathtt{?_{id}}$; $c$ is a \textit{symbolic constraint} subject to $ FV(c)\subseteq\mathtt{\textbf{?}}$. An SRM $\srm$ can be concretized by any $\mathtt{\textbf{h}}\in \mathcal{\textbf{H}}$ into an $l:=\srm[\mathtt{\textbf{h}}/\mathtt{\textbf{?}}]$ iff  $[\![c[\mathtt{\textbf{h}}/\mathtt{\textbf{?}}]]\!]=\top$.
\end{definition}

Concretizing an SRM does not readily mean that the resulting reward function will be effective for the RL task. 
Hence, we further assume that a set $E$ of demonstrated trajectories is provided by the expert, thus inducing a learning from demonstration (LfD) version of the SRM concretization problem $\langle \srm, \textbf{H},c, E\rangle$. The solution $\mathtt{\textbf{h}}$ of this problem not only concretizes the SRM $\srm$ but also satisfies $\forall \pi.\mathbb{E}_{\tau\sim E}[\srm[\mathtt{\textbf{h}}/\mathtt{\textbf{?}}](\tau)]\geq \mathbb{E}_{\tau\sim\pi}[\srm[\mathtt{\textbf{h}}/\mathtt{\textbf{?}}](\tau)]$, which inherits the definition of generic IRL  in~\cite{Ng:2000:AIR:645529.657801}.

\section{A Hierarchical Bayesian Learning Approach To SRM Concretization}
In this section, we propose an approach to concertize SRM with human demonstrations. Unlike the generic IRL problems, \textit{an SRM such as the one in Fig.\ref{fig1_2} may be parameterized not only in the output reward but also in the transition conditions.} Our approach is inspired by Bayesian GAIL as mentioned in the \textit{Section 3}. However, directly using $\srm$ to substitute $f$ in the discriminator $D$ in Eq.\ref{eq1_1} is not practical due to the following challenges: \textit{a) $\srm$ is not differentiable w.r.t the holes; b) $\srm$ is trajectory based}. In short, stochastic gradient descent with batched data is not readily applicable. Hence, we propose a hierarchical inference approach to circumvent this issue.  

Given a $\pi_A$, we seek the best SRM concretization $l:=\mathtt{\mathcal{L}[\textbf{h/?}]}$ by maximizing the log-likelihood $\log p(l|0_A, 1_E;\pi_A, E)= \log p(0_A,1_E|\pi_A, E, l)p(l) + constant$ where the prior $p(l)$ can be an uniform distribution over some allowable SRM set; $ \log p(0_A,1_E|\pi_A, E, l)$ can be factorized to Eq.\ref{eq1_1} by substituting $f$ with $l$ in the discriminator $D$. Our objective is to learn a distribution $q$ of $l$ by minimizing $D_{KL}\Big[q (l)|| p(l|0_A,1_E; \pi_A, E)\Big]= \underset{l\sim q}{\mathbb{E}}\Big[\log q(l) - \log \frac{p(0_A, 1_E|\pi_A, E, l) p(l)}{p(0_A, 1_E|\pi_A, E)}\Big]$, and by maximizing its evidence lower-bound $ELBO(q)= D_{KL}\Big[q (l)|| p(l)\Big] +\underset{\mathclap{l\sim q}}{\mathbb{E}}\ \Big[ \log p(0_A, 1_E|\pi_A, E, l)\Big]$. When a symbolic constraint is considered, the prior $p(l)$ can be viewed as being uniform only among those $l$'s satisfying the symbolic constraint while being zero everywhere else. We let $J_{con}(q):=D_{KL}[q(l)||p(l)]$ be a supervised learning loss.  
\begin{eqnarray}
&&\log p(0_A,1_E|\pi_A, E, l)\nonumber\\
&:=&\log \sum\limits_{\tau_A,\tau_E} p(\tau_A|\pi_A)p(\tau_E|E)\iint\limits_{f_{\tau_A} f_{\tau_E}}p(0_A|\tau_A; \pi_A, f_{\tau_A})\nonumber\\
&&\qquad p(1_E|\tau_E; \pi_A, f_{\tau_E}) p(f_{\tau_E}| \tau_E; l)p(f_{\tau_A}|\tau_A; l)\label{eq4_3_0}\\
&\geq&\underset{f}{\max} \underset{{\epsilon\sim {N}(0,1)}}{\mathbb{E}}\Big[J_{adv}(D_\epsilon)\Big]-\nonumber\\
&&\qquad\qquad \underset{{\tau\sim \pi_A, E}}{\mathbb{E}}\Big[D_{KL}(p_f(\tau)||p_l(\tau))\Big]\label{eq4_3_3}
\end{eqnarray}
Regarding the remaining part in the $ELBO(q)$, the log-likelihood $\log  p(0_A,1_E|\pi_A, E, l)$ can be factorized as in Eq.\ref{eq4_3_0} by introducing two latent factors, $f_{\tau_A}$ and $f_{\tau_E}$, which are two sequences of stochastic rewards for the state-action pairs along $\tau_A$ and $\tau_E$. On one hand, each element of $f_\tau$ constitutes a discriminator for the labels $0_A$ or $1_E$ in the same way as the reward $f(s,a)$ does in the discriminator $D$ of Eq.\ref{eq1_1}. On the other hand, $f_\tau$ is viewed as a noisy observation of $[\![l]\!](\tau)$ in that the latent distribution $p(f_\tau|\tau; l)$ is interpreted as the likelihood of observing $f_\tau$ given $[\![l]\!](\tau)$. Here, we adopt a tractable model such as Gaussian noise $\epsilon\sim \mathcal{N}(0,1)$ to simulate $f_\tau= [\![l]\!](\tau)+ \epsilon$ which means adding the same $\epsilon$ to each reward in the reward sequence $[\![l]\!](\tau)$. 
To measure the integrals in Eq.\ref{eq4_3_0}, we re-introduce a neurally simulated reward function $f$ for the importance sampling of the stochastic $f_\tau$'s. We define $p(f_\tau|\tau; f)$ in the same way as $p(f_\tau|\tau; l)$ except for replacing $f_\tau=[\![l]\!](\tau)+\epsilon$ with $f_\tau=f(\tau[t]) + \epsilon$. With the sampled $f_\tau\sim p(f_\tau|\tau; f)$, we obtain a lower-bound Eq.\ref{eq4_3_3} where the GAIL objective $J_{adv}$ as defined in Eq.\ref{eq1_1} is embedded but with $D_\epsilon(s,a):=\frac{\exp(f(s,a)+\epsilon)}{\exp(f(s,a)+\epsilon) + \pi_A(a|s)}$ in place of $D$. We prove in \textit{Theorem 9.1} in \textit{Appendix A.4} that the stochastic version $\mathbb{E}_{\epsilon\sim\mathcal{N}(0,1)}[J_{adv}(D_\epsilon)]$ has the same optimal condition as that of $J_{adv}(D)$ in Eq.\ref{eq1_1}. We also abbreviate $p(\cdot|\tau; l)$ and $p(\cdot|\tau; f)$ as $p_l(\tau)$ and $p_f(\tau)$ in the KL-divergence $D_{KL}(p_f(\tau)||p_l(\tau))$, which can be viewed as a regularization term and turns out to be proportional to the squared error $\sum^{T-1}_{t=0}([\![l]\!](\tau[:t])-f(\tau[t]))^2$. 
Then \eqref{eq4_3_3} can be viewed as optimizing $f$ to maximize the weighted sum of two components: $J_{adv}(D)$ and this aforementioned regularization term. We maximize the expectation $\mathbb{E}_{q}[\log  p(0_A,1_E|\pi_A, E, l)]$ by maximizing the expectation of its lower-bound \eqref{eq4_3_3}. Note that $J_{adv}$ is irrelevant to $l\sim q$. Hence, we optimize $q$ only to minimize a supervised loss $J_{soft}(q, f):= \underset{l\sim q}{\mathbb{E}}\Big[\underset{{\tau\sim \pi_A, E}}{\mathbb{E}}\big[D_{KL}(p_f(\tau)||p_l(\tau))\big]\Big]$.

In our implementation, we consider the case when the holes are all real numbers, i.e., $\forall id.H_{id}=\mathbb{R}$. We construct a neurally simulated sampler $q_\varphi$ to output the mean and diagonal variance matrix of a multivariate Gaussian distribution of which the dimension equals the number of holes. As each hole assignment $\mathtt{\textbf{h}}$ sampled from this Gaussian corresponds to a $l:=\mathcal{L}[\mathtt{\textbf{h/?}}]$, we still denote by $q_{\varphi}(l)$ the distribution of $l$'s. Besides $q_\varphi$, we let $f_\theta$ be a neurally simulated $f$ and use it to denote $D_\epsilon$. To calculate the gradients of $J_{soft}(q_\varphi, f_\theta)$ w.r.t $\varphi$, we use the logarithmic trick from \cite{peters2008reinforcement} to handle $\mathbb{E}_{l\sim q_{\varphi}}[\cdot]\approx \frac{1}{K} \nabla_{\varphi_i}\log q_{\varphi_i}(l_k)[\cdot]$ with $K$ samples of concretized SRMs. 
Reparameterization trick~\cite{kingma2013auto} is also used to optimize the stochastic adversarial objective $\mathbb{E}_{\epsilon\sim\mathcal{N}(0,1)}[J_{adv}(D_\epsilon)]$ w.r.t $f_\theta$. We note that $J_{con}$ is infinitely large once the support of the uniform distribution $p(l)$ induced by symbolic constraint $c$ does not match $q_\varphi(l)$. Since $q_\varphi$ specifies a Gaussian distribution, we relax $J_{con}$ to only penalize $q_\varphi$ if the output mean violates $c$. Especially, we only consider the case where the symbolic constraints are all conjunctions of atomic predicates that only involve linear arithmetic, e.g., $\mathtt{?_i+?_j}\leq 0$. We evaluate a binary cross-entropy loss for the mean output  by $q_\varphi$ violating the linear constraints extracted from the symbolic constraint, adding an entropy loss for the variance. This relaxed $J_{con}$ is differentiable w.r.t $\varphi$. We use a neural network $\pi_\phi$ to simulate the agent policy $\pi_A$ and train it with the most likely $l^*=\arg\max_l[q_\varphi(l)]$ which can also be readily obtained from the mean output by $q_\varphi$. We summarize the algorithm in Algorithm 1 and illustrate the flow chart of Algorithm 1 in Fig.\ref{fig4_2}(b).
\begin{figure}
\centering
 
    \subfigure { 
  
        \includegraphics[height=2.5cm, width=8.cm]{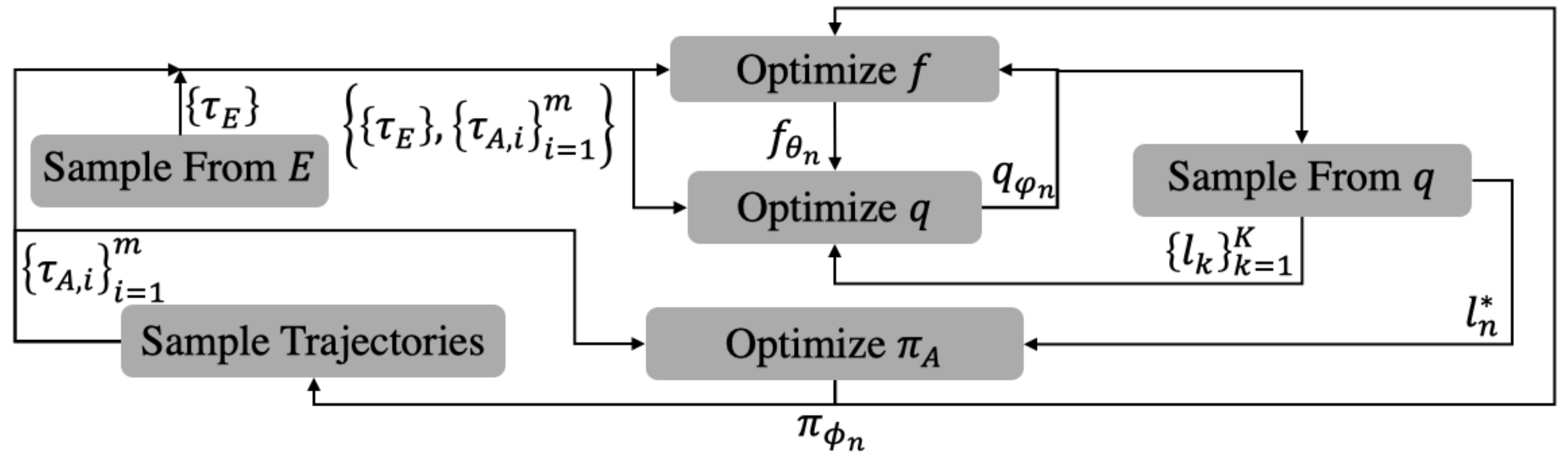}
       }
        
    \caption{The flow chart of Algorithm 1} 
    \label{fig4_2}
    \label{fig4_1}
\end{figure}
\begin{algorithm}[tb]
\caption{Hierarchical Bayesian Inference For SRM Concretization}
\label{alg:algorithm}
\textbf{Input}: Expert demonstration $E$, initial agent policy $\pi_{\phi_0}$, reward function $f_{\theta_0}$, sampler $q_{\varphi_0}$, iteration number $i=0$, maximum iteration number $N$ \\
\textbf{Output}: $\pi_{\phi_N}$ and $q_{\varphi_N}$
\begin{algorithmic}[1] 
\WHILE {iteration number $i<N$}
\STATE Sample trajectory set $\{\tau_{A,i}\}^m_{i=1}$ by using policy $\pi_{\phi_i}$
\STATE Computing reward $\{l^*(\tau_{A,i})\}^m_{i=1}$ with the most likely $l^*={\arg\max}_l\ q_{\varphi_i}(l)$
\STATE Update $\phi_i\rightarrow \phi_{i+1}$ using policy optimization, e.g., PPO
\STATE Sample $K$ samples $\{l_k\}^K_{i=1}$ by using $q_{\varphi_i}$
\STATE Sample $\{f_{\theta_i}(s,a)+ \epsilon|\epsilon\sim\mathcal{N}(0, 1)\}$ respectively with $(s,a)\in E$ and  $\{\tau_{A,i}\}^m_{i=1}$
\STATE Update $\theta_{i+1}\leftarrow \theta_{i} + \alpha \nabla_{\theta_i} (J_{soft}(q_{\varphi_i}, f_{\theta_i}) + \mathbb{E}_{\epsilon\sim\mathcal{N}}[J_{adv}(D_\epsilon)])$ with a step size parameter $\alpha$
\STATE Update $\varphi_{i+1}\leftarrow \varphi_{i} + \beta \nabla_{\varphi_i} J_{soft}(q_{\varphi_i}) + \beta\eta \nabla_{\varphi_i}J_{con}(q_{\varphi_i})$ with step size parameters $\beta, \eta$
\ENDWHILE
\STATE \textbf{return} $\pi_{\phi_N}$ and $q_{\varphi_N}$
\end{algorithmic}
\end{algorithm}

\section{Experiments} 
\begin{figure*} 
\centering
     \subfigure[DoorKey-16x16]{

         \includegraphics[height=2.5cm,width=2.8cm]{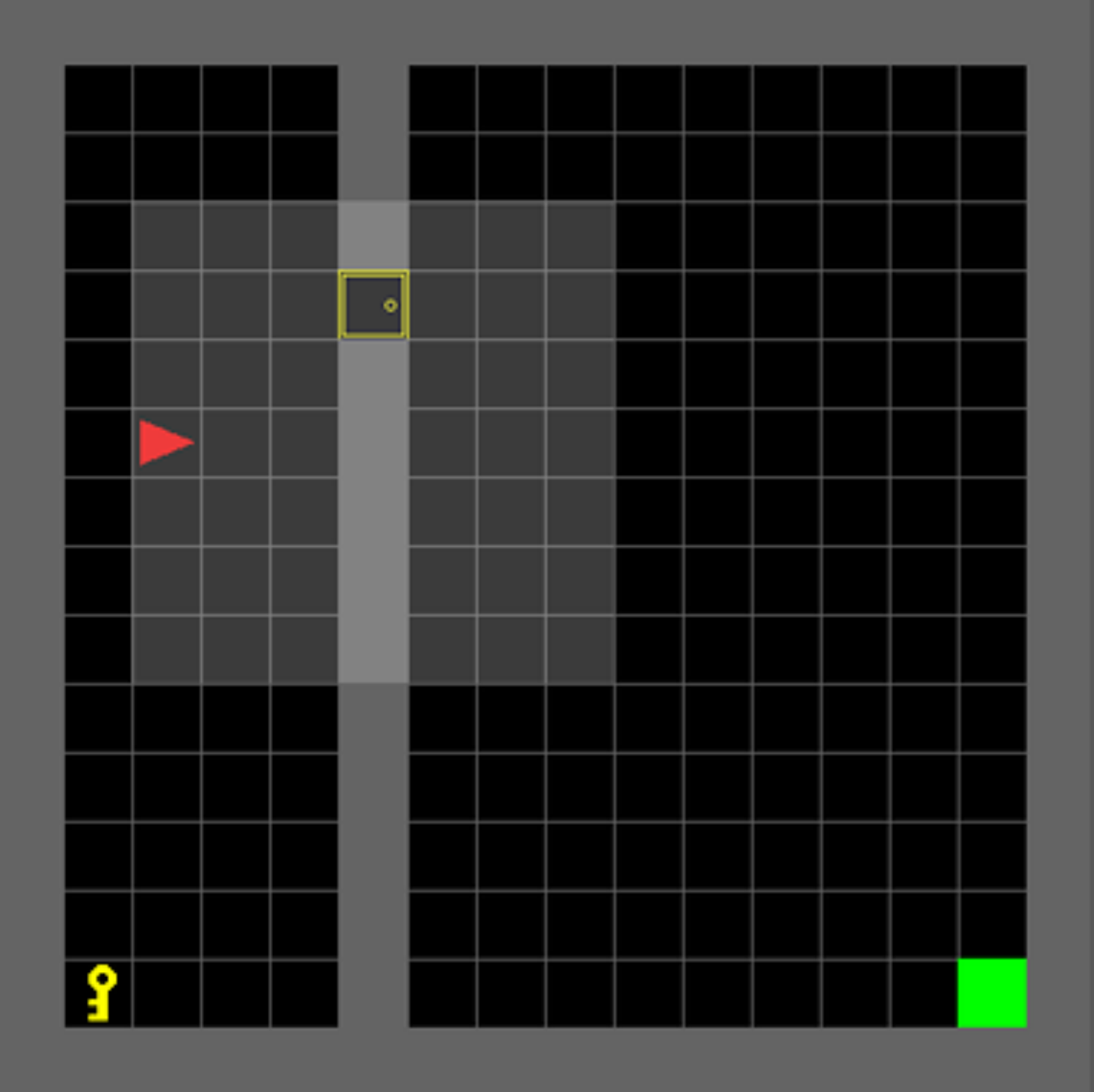}
        }\qquad\qquad\quad
    \subfigure[KeyCorridorS6R3]{

         \includegraphics[height=2.5cm,width=2.8cm]{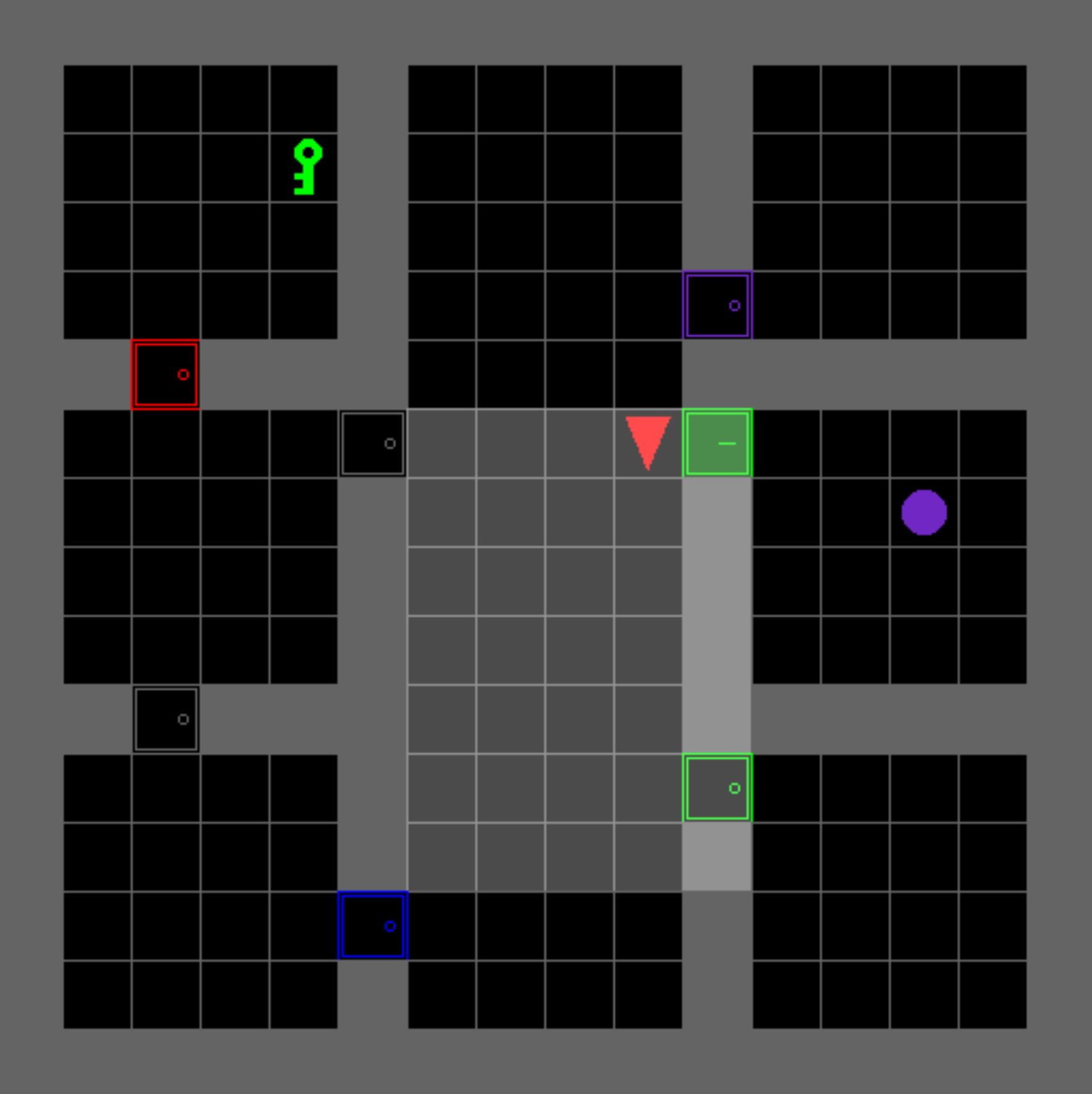}
         
   } \qquad\quad\quad\quad
     \subfigure[ObstructedMaze-Full]{

          \includegraphics[height=2.5cm,width=2.8cm]{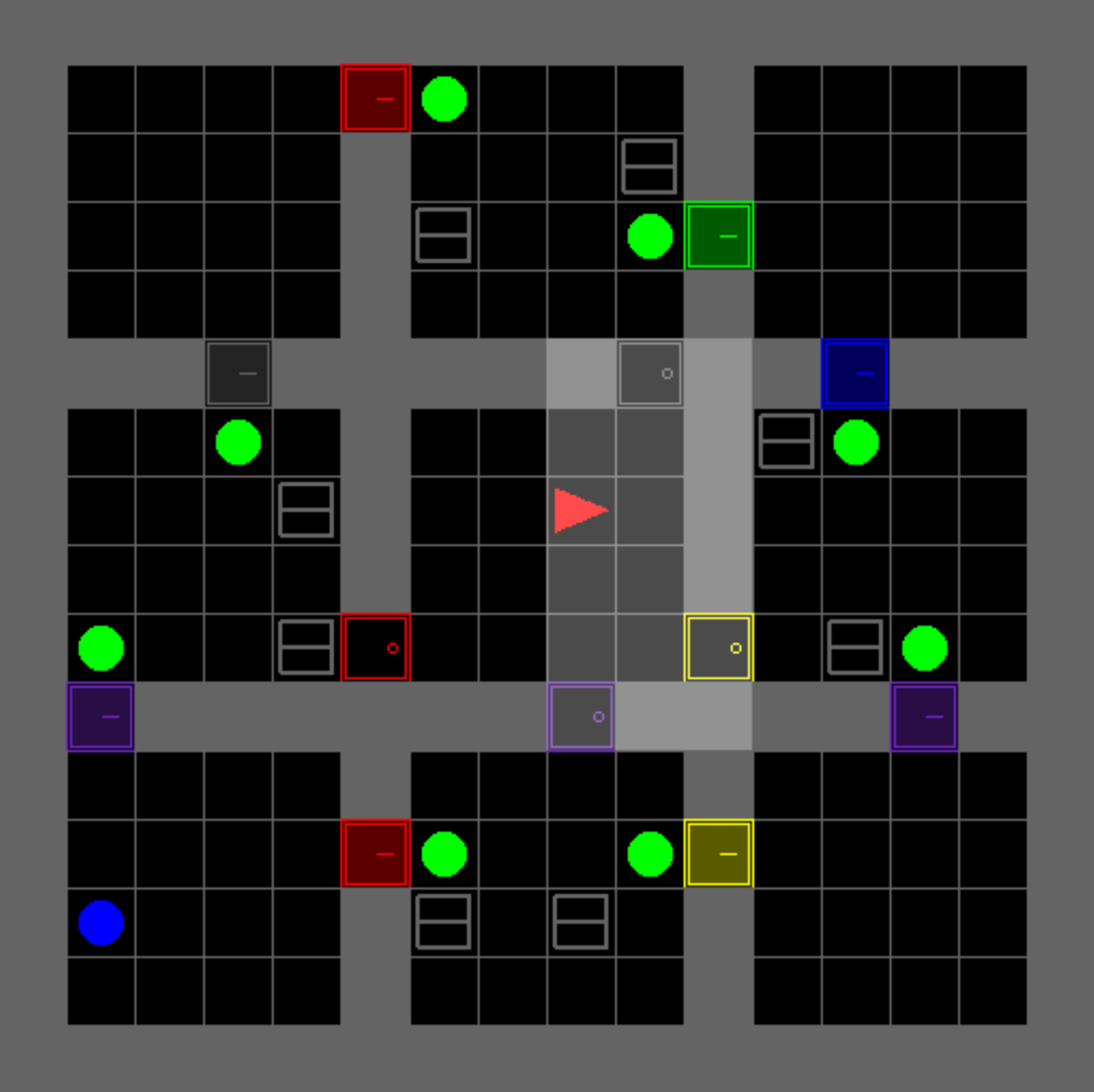}
          
    }%
     
     \subfigure[DoorKey-8x8-v0]{

         \includegraphics[height=3.2cm, width=4.5cm]{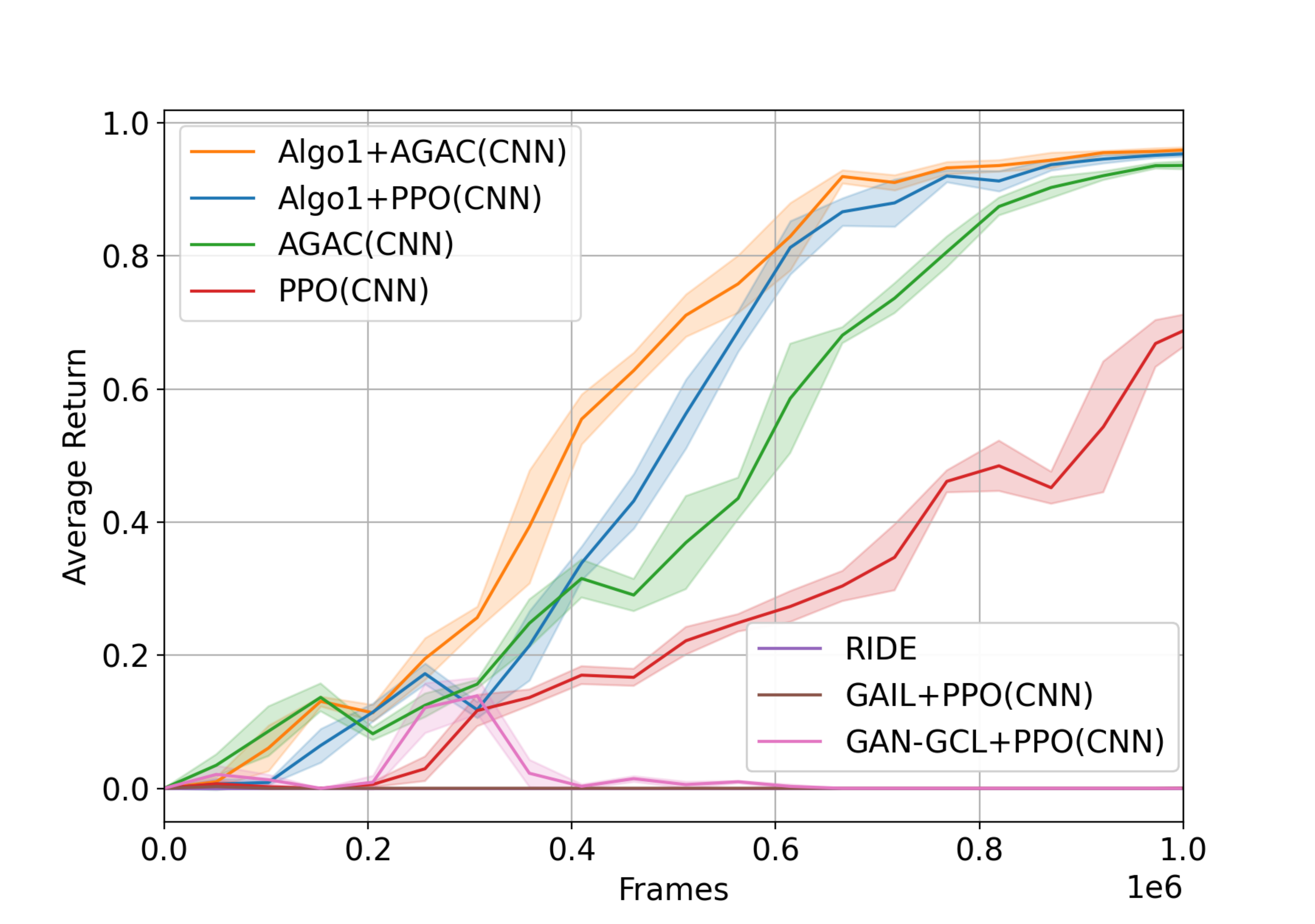}
          
    }%
     \subfigure[KeyCorridorS3R3]{

         \includegraphics[height=3.2cm, width=4.5cm]{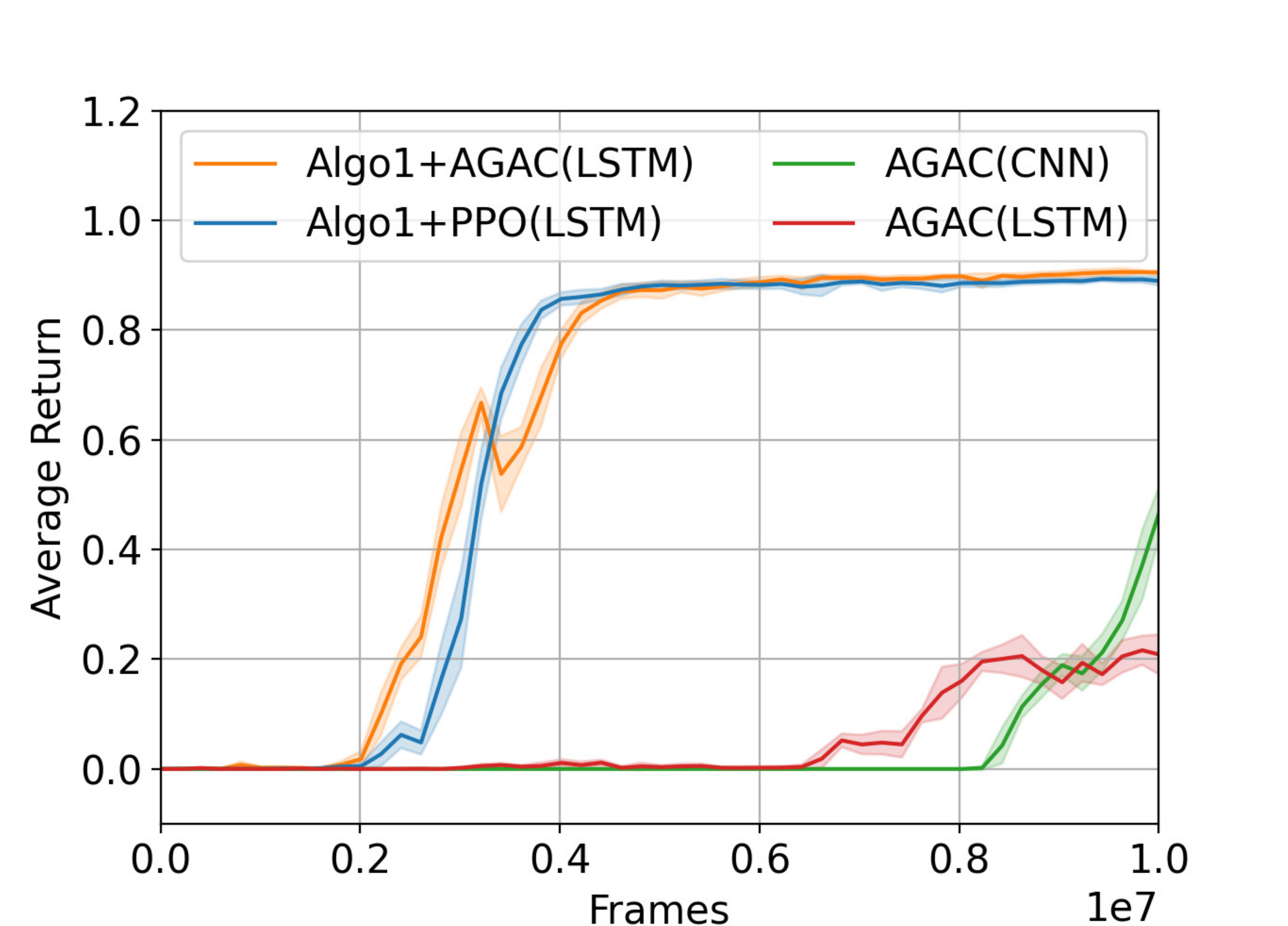}
          
    } %
     \subfigure[ObstructedMaze-2Dlhb]{

         \includegraphics[height=3.2cm, width=4.5cm]{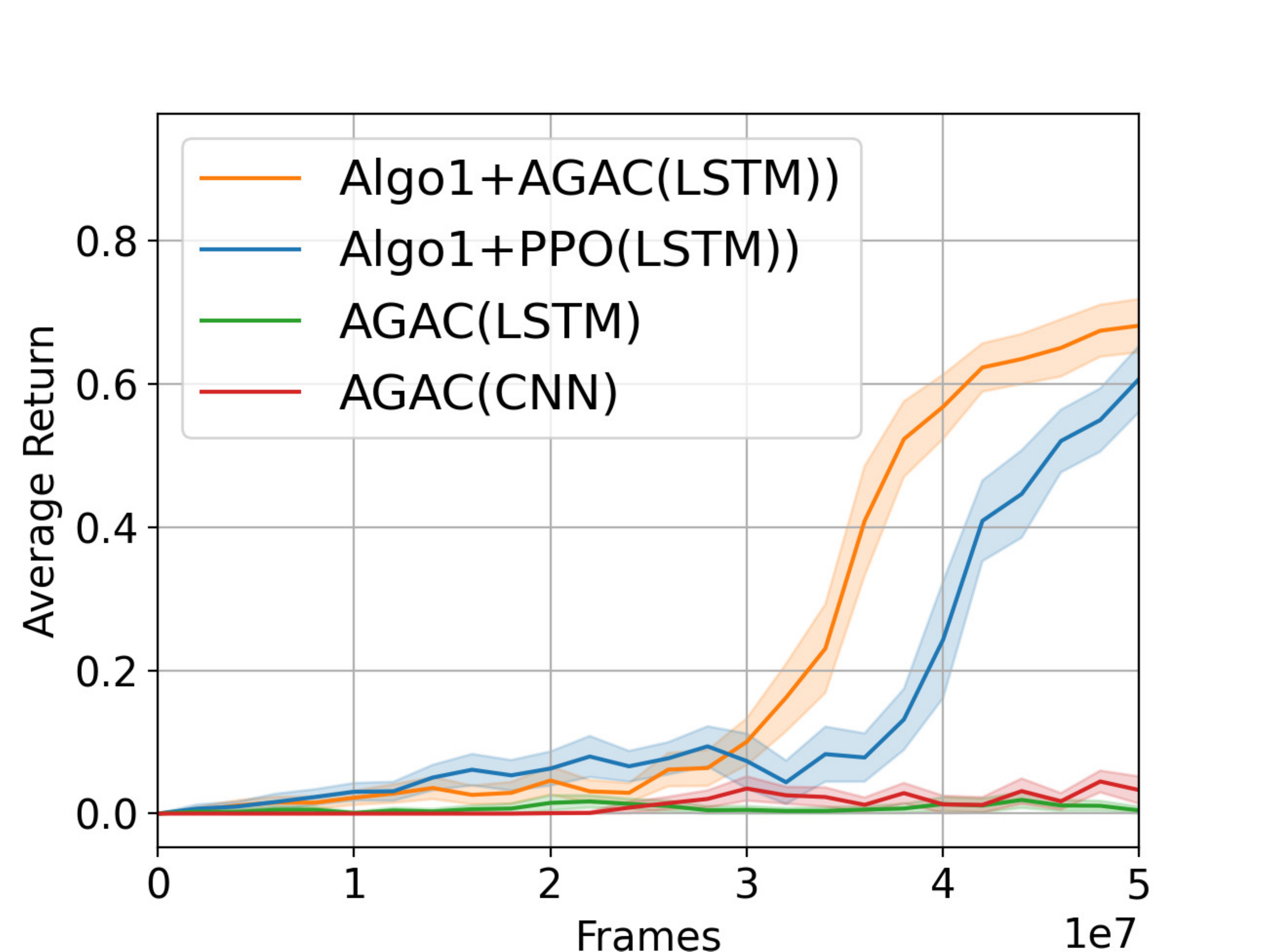}
          
    }

      \subfigure[DoorKey-16x16-v0]{
         \includegraphics[height=3.2cm, width=4.5cm]{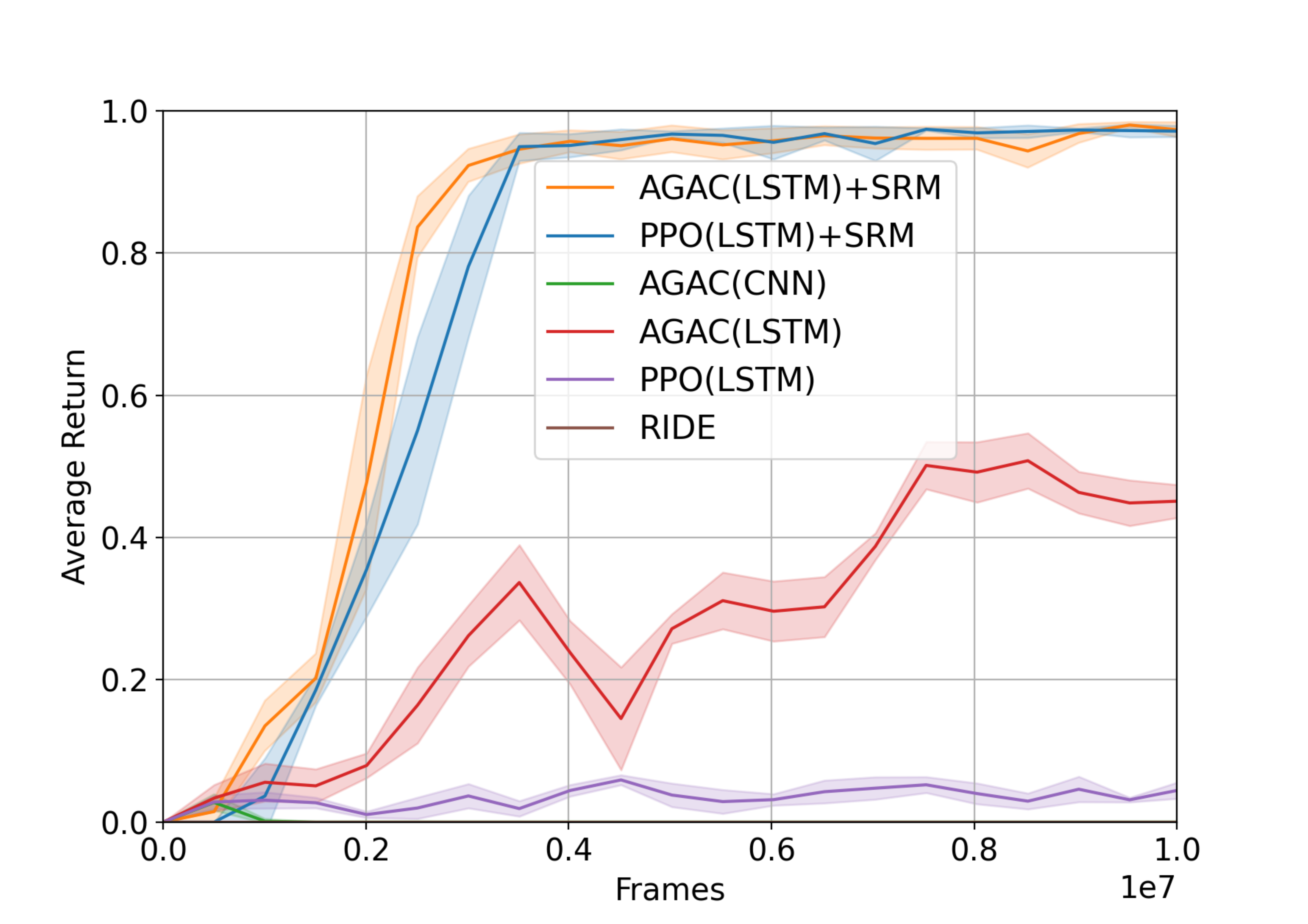}
         
    }
      \subfigure[KeyCorridorS4/S6R3]{

         \includegraphics[height=3.2cm, width=4.5cm]{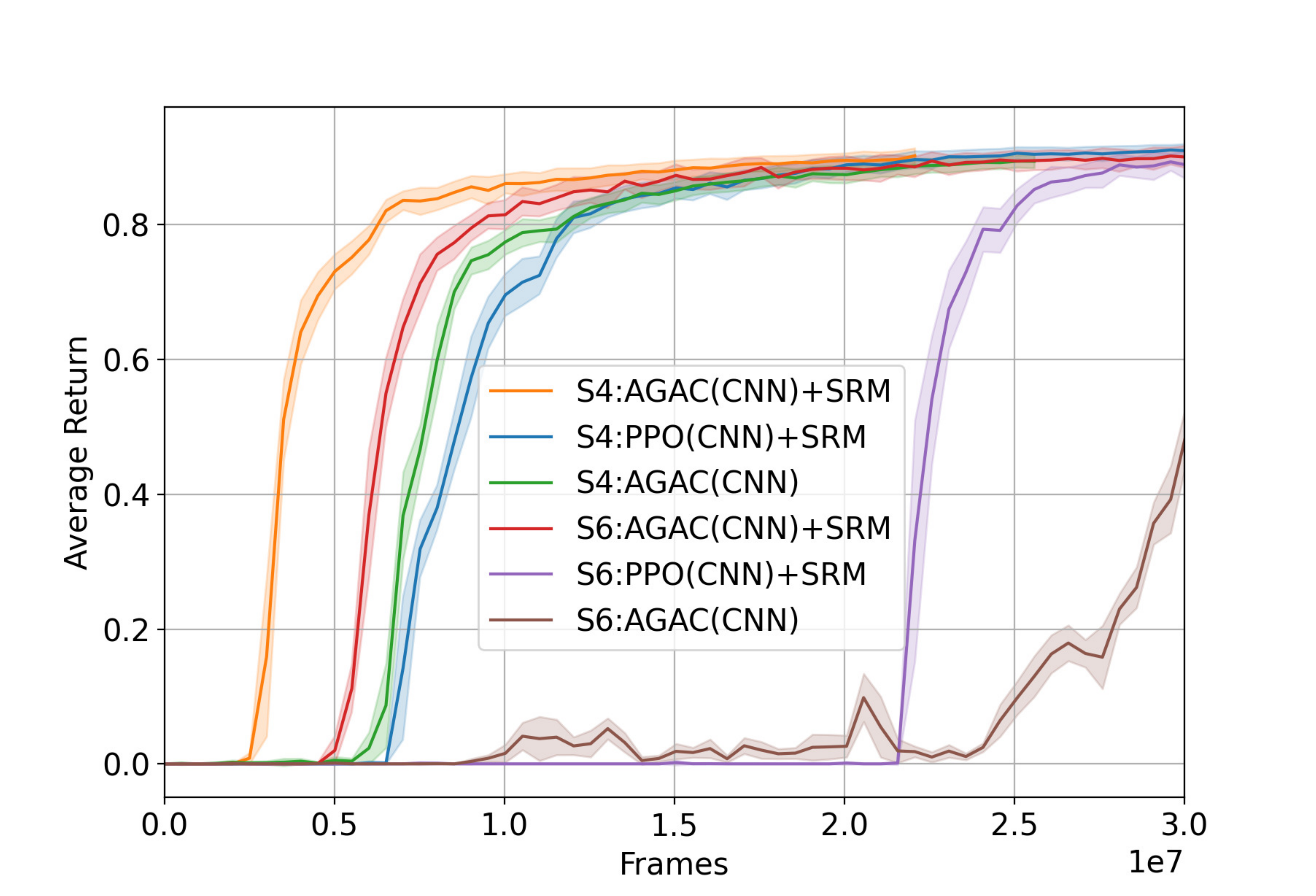}
         
    } 
     \subfigure[ObstructedMaze-Full]{
         \includegraphics[height=3.2cm, width=4.5cm]{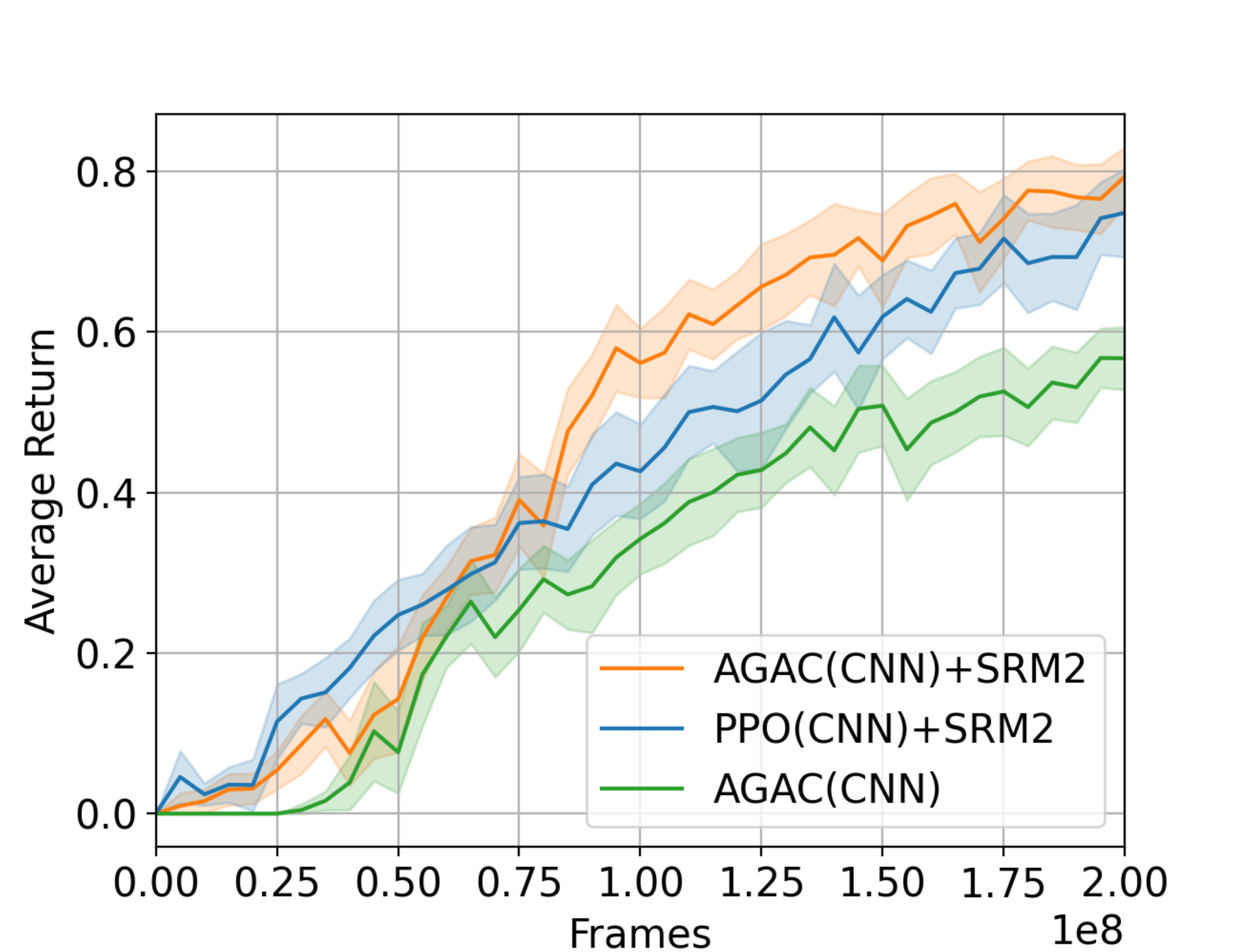}
         
    } 
     
     \caption{Algo1+AGAC/PPO indicates using AGAC or PPO as the policy learning algorithm in line 4 of Algorithm 1. AGAC/PPO+SRM indicates training an AGAC or PPO agent with the concretized SRM. CNN and LSTM in the parentheses indicate the versions of the actor-critic networks. S4 and S6 in (k) indicate respectively the results for KeyCorridorS4R3 and KeyCorridorS6R3.} \label{fig5_0}\label{fig5_1}\label{fig5_2}\label{fig5_7}\label{fig5_3}\label{fig5_6}\label{fig5_8}\label{fig5_4}\label{fig5_11}\label{fig5_9}\label{fig5_5}\label{fig5_10}
\end{figure*}
\begin{figure*}
     \centering
     \subfigure[DoorKey-8x8-v0]{

         \includegraphics[height=3.2cm, width=4.5cm]{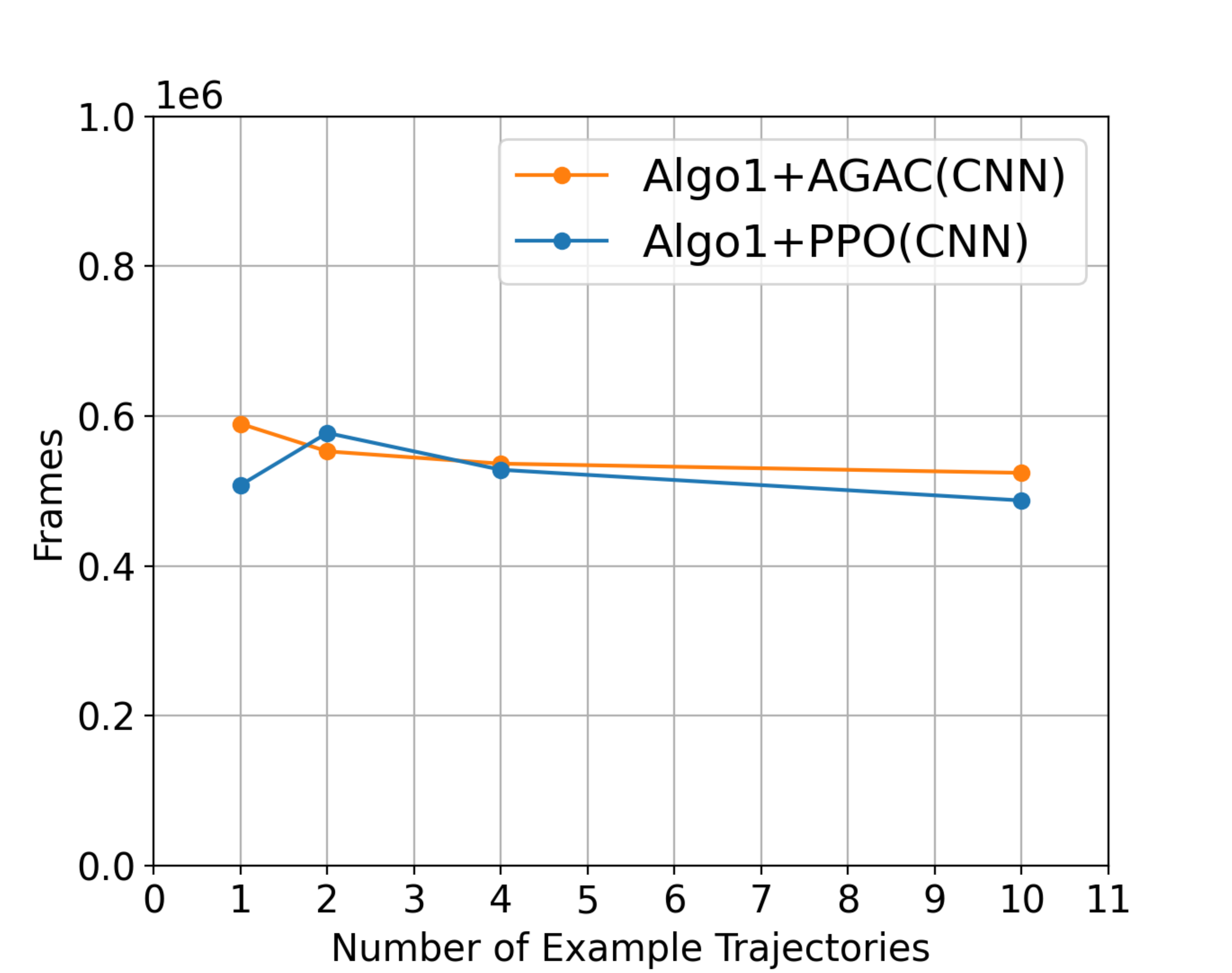}
         
    }
     \subfigure[KeyCorridorS3R3]{

         \includegraphics[height=3.2cm, width=4.5cm]{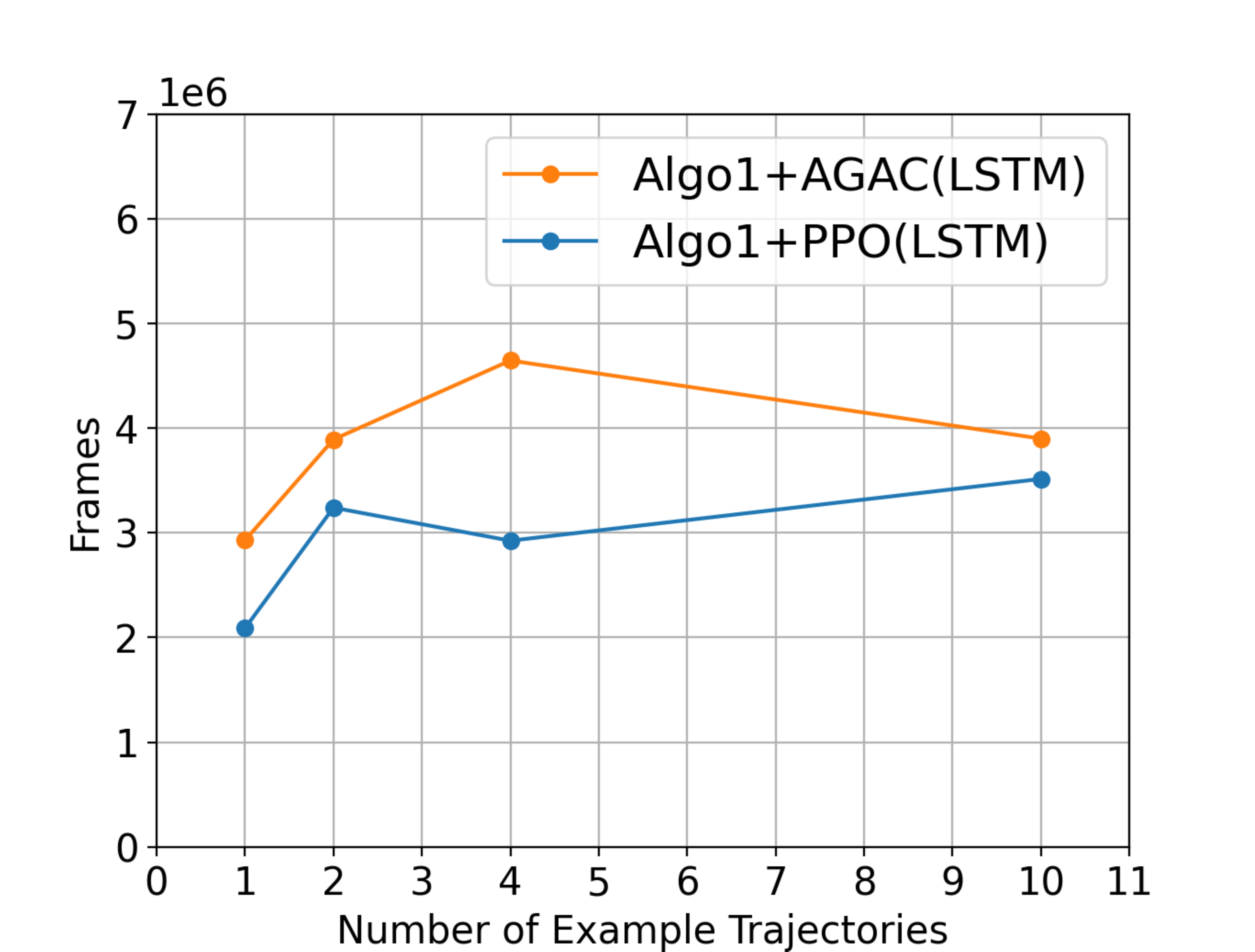}
          
    }
     \subfigure[ObstructedMaze-2Dlhb]{

         \includegraphics[height=3.2cm, width=4.5cm]{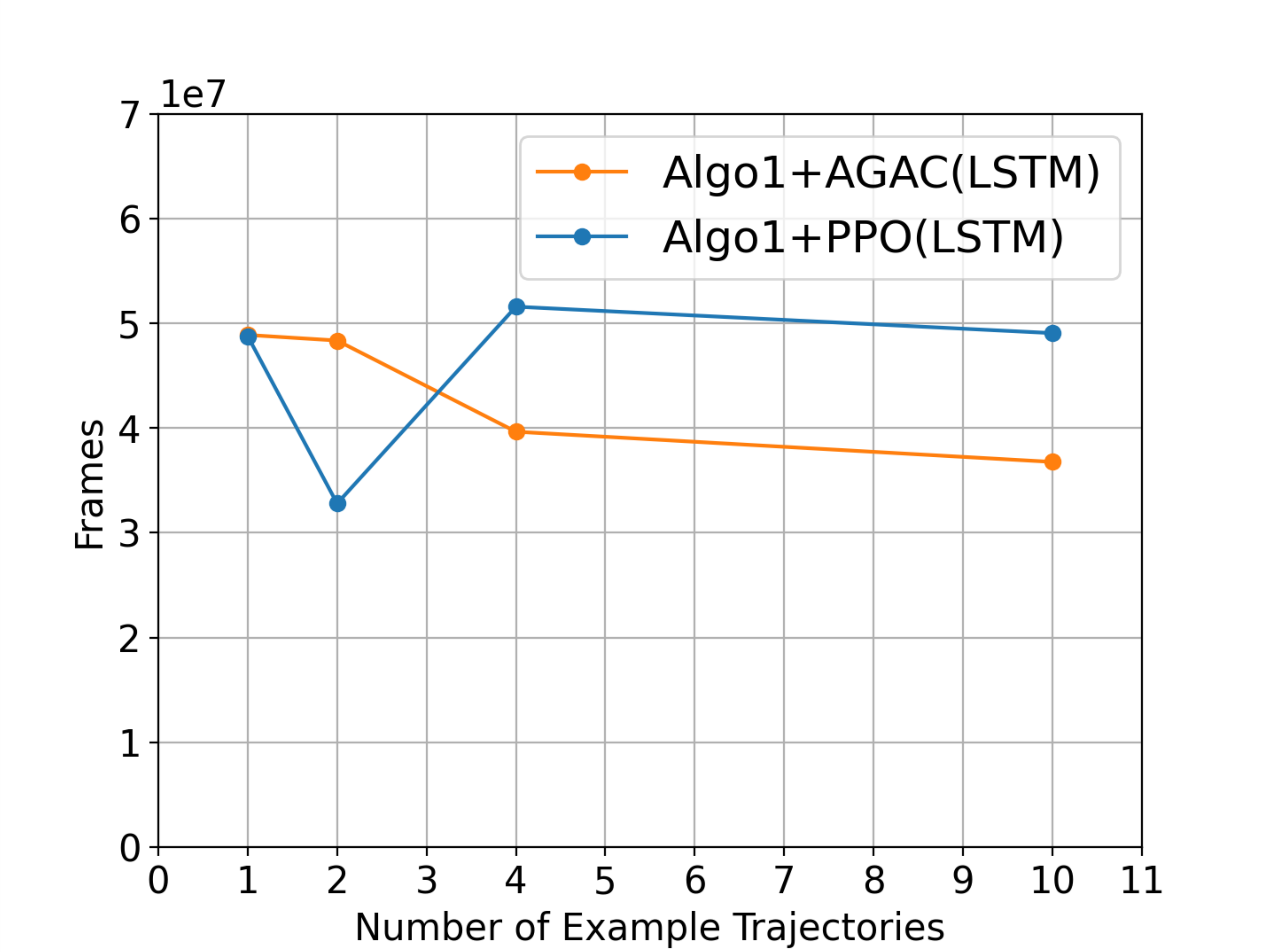}
         
    } 
    
     \subfigure[DoorKey-8x8-v0]{
         \centering
         \includegraphics[height=3.2cm, width=4.5cm]{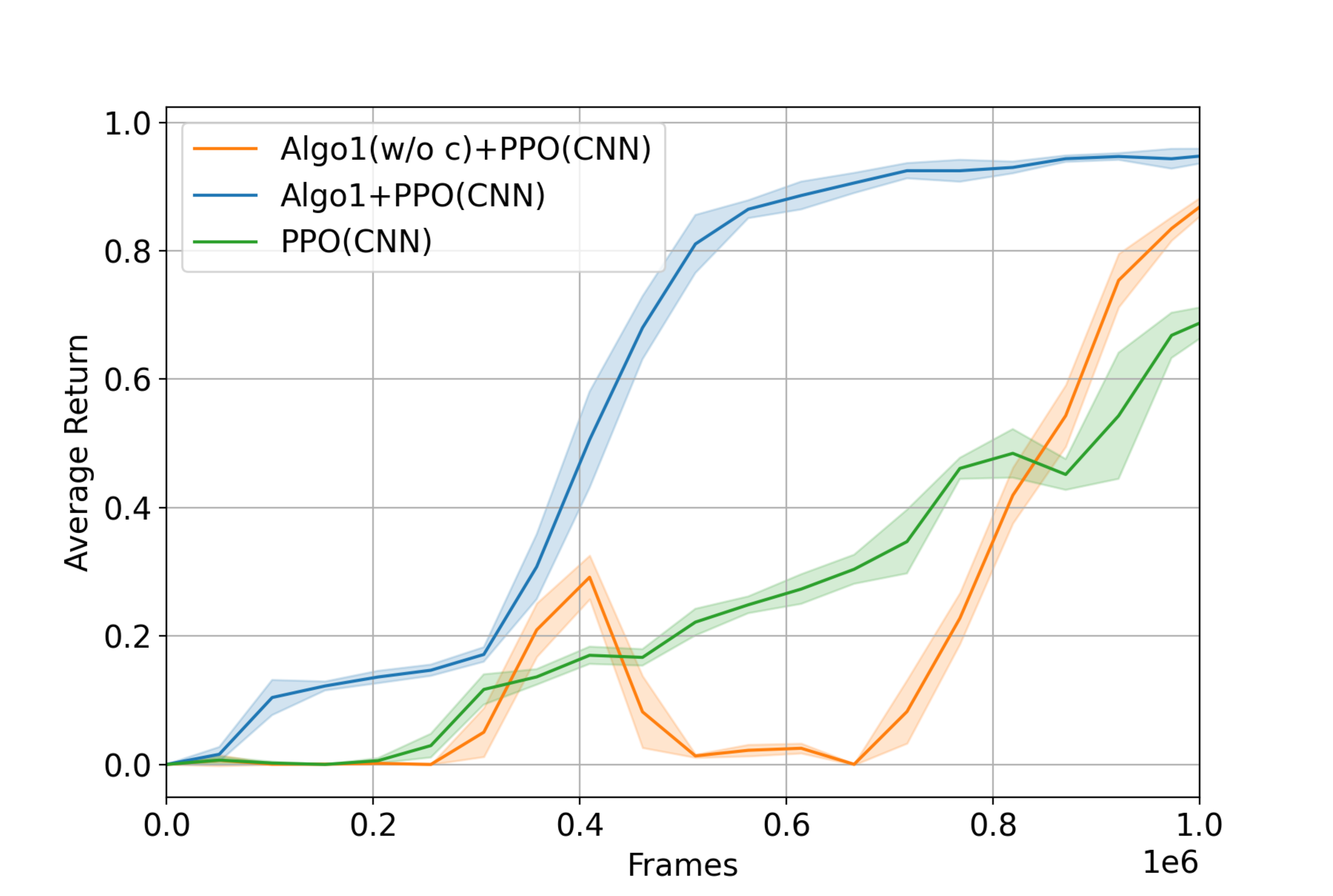}
         
    }
     \subfigure[KeyCorridorS3R3]{
         \centering
         \includegraphics[height=3.2cm, width=4.5cm]{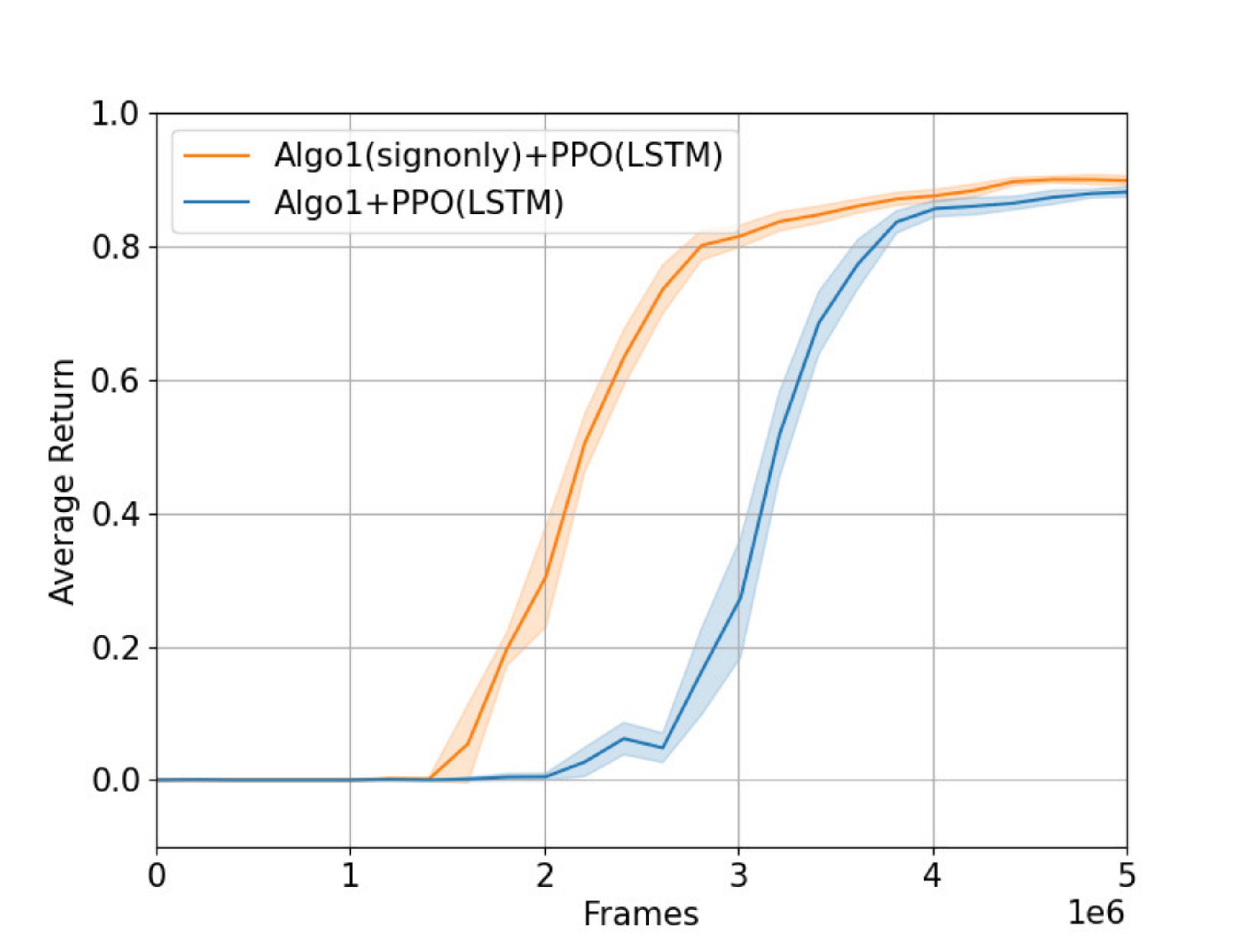}
         
    }
     \subfigure[ObstructedMaze-2Dlhb]{
         \centering
         \includegraphics[height=3.2cm, width=4.5cm]{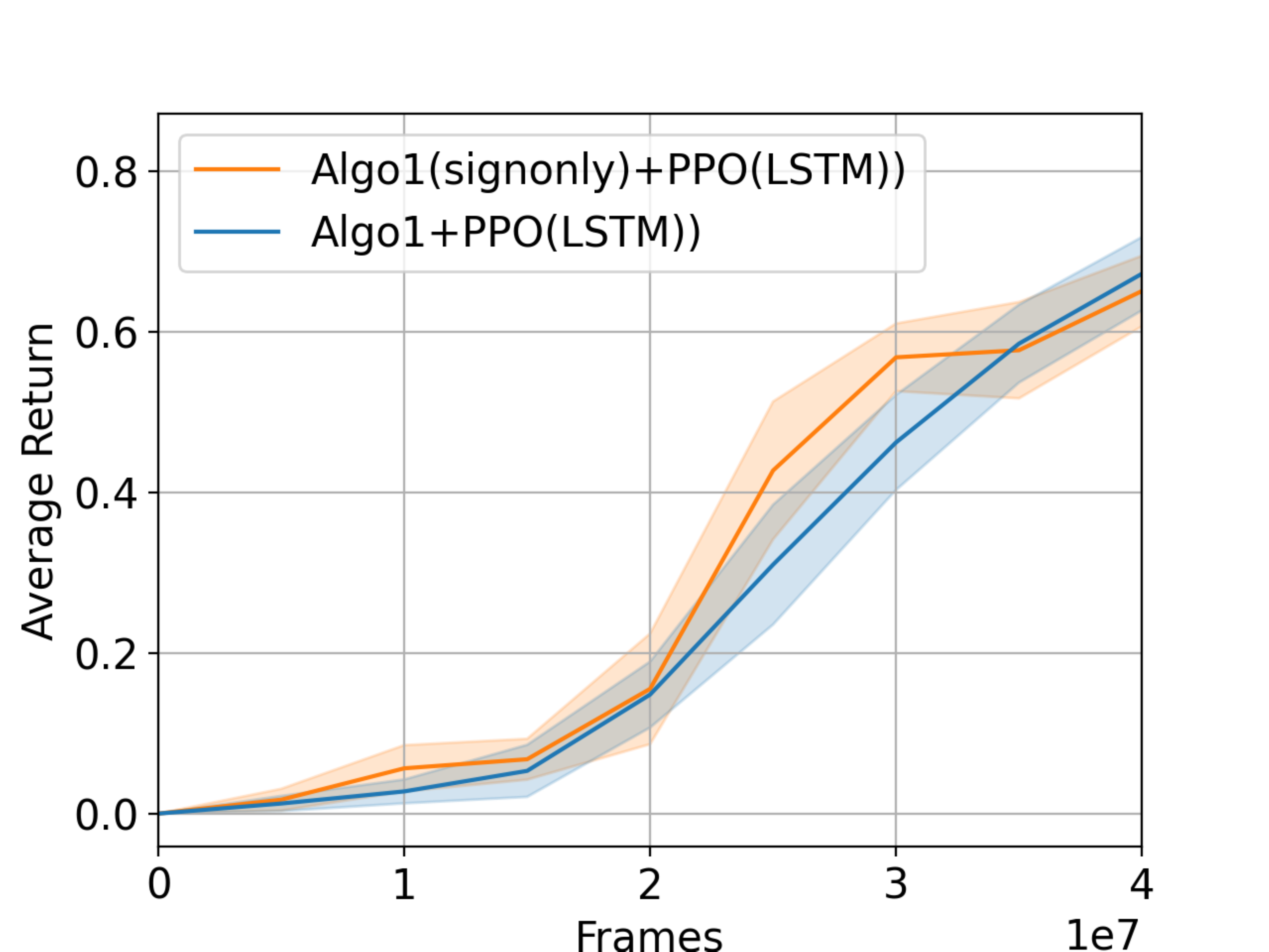}
          
    } 
     
     \caption{Algo1(w/o c)+PPO(CNN) indicates running Algorithm 1 without symbolic constraint; Algo1(signonly)+PPO(CNN) indicates running Algorithm 1 with symbolic constraint that only concerns the signs of the holes}     \label{fig6}
\end{figure*}
\noindent\textbf{Benchmarks.} Our benchmark includes three tasks of growing difficulty in the Mini-Grid environment: \textit{Door-Key}, \textit{KeyCorridor} and \textit{ObstructedMaze}. The second task has been introduced earlier. The first task is  shown in Fig.\ref{fig5_1}(a) where an agent needs to pick up a key, unlock a door and reach a target tile. The third task is shown in Fig.\ref{fig5_2}(c) in which the agent needs to pick up a targeted blue ball in a locked room. Unlike \textit{KeyCorridor}, in \textit{ObstructedMaze} some doors are locked; the keys for the locked doors are hidden in grey boxes; and each locked door is obstructed by a green ball. 
We note that despite the difficulty of these tasks, our designed SRMs do not carry out any motion planning and are solely based on  {reasoning the significant events}. The details are explained in \textit{Appendix A.2}. In all three tasks, the environments can vary in size by changing the number of rooms and tiles (e.g., DoorKey-8x8 vs. DoorKey-16x16). 
The placements of the objects and doors are randomized in each instance of an environment. By default, the agent is not rewarded until it finishes the tasks.
\subsection{Main Results}
In this section, 
we investigate the following questions: \textbf{A. Performance}: whether {Algorithm 1} can train an agent policy to achieve high 
average returns with a small number of environment interactions;  \textbf{B. Generalization}: whether the SRM concretized by Algorithm 1 for one environment can be used to improve the performance of RL agents on a different environment for the same task.

\noindent\textbf{Baselines.} We answer question \textbf{A} by comparing Algorithm 1 with generic IRL algorithms, including GAN-GCL from~\cite{fu2017learning} and GAIL from~\cite{ho2016generative}, both of which use neural networks to simulate the reward functions. We use a generic RL algorithm, PPO ~\cite{DBLP:journals/corr/SchulmanWDRK17}, and an exploration driven RL algorithms, AGAC~\cite{flet-berliac2021adversarially} for RL training in line 4 of Algorithm 1, to show how different RL algorithms affect the results. The results are annotated with Algo1+AGAC/PPO. We answer question \textbf{B} by using the SRM concretized via Algorithm 1 to train PPO and AGAC RL agents on the same tasks as in Algorithm 1 but in differently configured environments, e.g. concretizing an SRM in an 8x8 environment and training an RL agent in a 16x16 environment. The results are annotated with AGAC/PPO+SRM. Additionally, in all three tasks, we use the default reward to train RL agents with PPO, AGAC as well as another intrinsic-reward augmented RL algorithm, RIDE~\cite{Raileanu2020RIDE:}. We show the RL training results for reference, since AGAC and RIDE have been competitive and widely used as baselines in Mini-Grid tasks.
 
\textbf{Evaluation Setup.}
For each task, our basic setup includes $10$ demonstrated trajectories, an SRM, optionally a symbolic constraint, an actor-critic agent $\pi_\phi$, a neurally simulated reward function $f_\theta$ and a  sampler $q_\varphi$ that generates a multivariate Gaussian distribution. The actor-critic networks of $\pi_\phi$ have two versions, a non-recurrent CNN version and an LSTM version. In each task we only report the result of the one with higher performance between those two versions . The reward function $f_\theta$ is simulated by an LSTM network. For fair comparisons, we use identical hyperparameters and the same actor-critics and neurally simulated reward functions, if applicable, when comparing our approach with PPO, GAN-GCL, GAIL and AGAC. To measure training efficiency, we show how the average return, i.e. the average default reward achieved over a series of consecutive episodes by the agent policy, changes as the number of frames, i.e. the number of total interactions between the agent and the environment, increases. 

\textbf{Results.} We first run Algorithm 1 respectively in the 8X8 DoorKey, 7X7 KeyCorridorS3R3, 2-room ObstructedMaze2Dlhb environments. As shown in Fig.\ref{fig5_7}(c)(d)(e), using PPO and AGAC in line 4 of Algorithm 1 respectively produce policies with higher performance while needing fewer frames than by training PPO or AGAC with the default reward. As RIDE, GAN-GCL+PPO and GAIL+PPO fail with close-to-zero returns in all three tasks, we omit the results of those failed baselines in Fig.\ref{fig5_7}(d)(e). We hence answer question \textbf{A}: Algorithm 1 can train an agent policy to achieve high average returns with small number of environment interactions. The SRMs concretized by Algorithm 1 in the small DoorKey and KeyCorridor task environments are then used to train PPO and AGAC agents in a 16x16 DoorKey environment; a 10x10 KeyCorridorS4R3 and a 16x16 KeyCorridorS6R3 environment. As shown in Fig.\ref{fig5_0}(g)(h), the PPO and AGAC agents trained with SRMs achieve higher performances with significantly fewer frames than those trained with the default reward. Since RIDE fails to achieve non-zero average returns in the experiments, we omit its results except in Fig.\ref{fig5_0}(g). For the ObstructedMaze task, multiple versions of SRMs are concretized in the 2-room ObstructedMaze2Dlhb environment. We select the one (annotated by SRM2) to train PPO and AGAC agents in a 9-room ObstructedMazeFull environment. As shown in Fig.\ref{fig5_10}(l), the RL agent trained with SRM2 attains high performance more efficiently than that with the default reward. We provide more details on the experimental comparison of the different versions of SRMs in \textit{Appendix A.1} and \textit{A.2.3}.

\subsection{Ablation Study}
\textbf{Example Efficiency.} We investigate 
the impact of reducing the number of demonstrated trajectories on the performance of {Algorithm 1}.
We vary the number of demonstrated trajectories and observe how many frames that Algorithm 1 takes to pass a certain level of average return. As shown in Fig.\ref{fig5_6}(a)(b)(c), reducing the number of examples (to $1$) does not affect the performance of Algorithm 1 for producing a policy with average return of at least $0.8$ in the DoorKey and KeyCorridor task, and at least $0.7$ in the ObstructedMaze task, regardless of whether PPO or AGAC is used in line 4 of Algorithm 1. These results show that Algorithm 1 can be example efficient even with one single demonstration. \add{Such high example efficiency is not uncommon in IRL techniques. We refer to \cite{ho2016generative} where GAIL achieves similar efficiency in certain control tasks.}

\textbf{Symbolic Constraint.} We investigate the impact of modifying the symbolic constraints on the performance of Algorithm 1.  
First, we consider \textit{removing the symbolic constraint} in the SRM for the DoorKey-8x8 task. 
As shown in Fig.\ref{fig6}(d), Algorithm 1 still produces a high-performance policy despite spending more interactions with the environment than that with the symbolic constraint. 
However, for the KeyCorridorS3R3 and ObstructedMaze-2Dhlb tasks, Algorithm 1 does not produce high performance without the corresponding symbolic constraints. 
Hence, we consider \textit{weakening the symbolic constraint} next. 
All the previously used symbolic constraints include \textit{relational predicates} such as $\mathtt{?_8+?_5}\leq 0$ in the motivating example. 
We remove all the relational predicates and keep only \textit{sign constraints} such as $\mathtt{?_{id}}\leq 0$. Details of the difference can be found in \textit{Appendix A.2.2} and \textit{9.2.3}.
Fig.\ref{fig6}(e)(f) show that Algorithm 1 with the weakened constraint achieves similar levels of average returns and 
even in less amount of frames for KeyCorridorS3R3. 
This represents a potential trade-off between reducing the parameter search space and achieving high performance, i.e. 
adding the relational constraint in this case ends up ruling out some good parameters.

\section{Conclusion}
We propose symbolic reward machines to represent reward functions in RL tasks. 
SRMs complement policy learning methods by providing a structured way to capture high-level task knowledge. 
In addition, we develop an approach to concretize SRMs by learning from expert demonstrations. Experimental comparison with SOTA baselines on challenging benchmarks validates our approach. Future works will focus on reducing human efforts in the design of SRMs.

\nocite{langley00}

\bibliography{main}

\begin{thebibliography}{45}
\providecommand{\natexlab}[1]{#1}
\providecommand{\url}[1]{\texttt{#1}}
\expandafter\ifx\csname urlstyle\endcsname\relax
  \providecommand{\doi}[1]{doi: #1}\else
  \providecommand{\doi}{doi: \begingroup \urlstyle{rm}\Url}\fi

\bibitem[Abbeel \& Ng(2004)Abbeel and Ng]{Abbeel:2004:ALV:1015330.1015430}
Abbeel, P. and Ng, A.~Y.
\newblock Apprenticeship learning via inverse reinforcement learning.
\newblock In \emph{Proceedings of the Twenty-first International Conference on
  Machine Learning}, ICML '04, pp.\  1--, New York, NY, USA, 2004. ACM.
\newblock ISBN 1-58113-838-5.
\newblock \doi{10.1145/1015330.1015430}.
\newblock URL \url{http://doi.acm.org/10.1145/1015330.1015430}.

\bibitem[Abel et~al.(2021)Abel, Dabney, Harutyunyan, Ho, Littman, Precup, and
  Singh]{DBLP:journals/corr/abs-2111-00876}
Abel, D., Dabney, W., Harutyunyan, A., Ho, M.~K., Littman, M.~L., Precup, D.,
  and Singh, S.
\newblock On the expressivity of markov reward.
\newblock \emph{CoRR}, abs/2111.00876, 2021.
\newblock URL \url{https://arxiv.org/abs/2111.00876}.

\bibitem[Alshiekh et~al.(2017)Alshiekh, Bloem, Ehlers, K{\"{o}}nighofer,
  Niekum, and Topcu]{DBLP:journals/corr/abs-1708-08611}
Alshiekh, M., Bloem, R., Ehlers, R., K{\"{o}}nighofer, B., Niekum, S., and
  Topcu, U.
\newblock Safe reinforcement learning via shielding.
\newblock \emph{CoRR}, abs/1708.08611, 2017.
\newblock URL \url{http://arxiv.org/abs/1708.08611}.

\bibitem[Amodei et~al.(2016)Amodei, Olah, Steinhardt, Christiano, Schulman, and
  Man{\'{e}}]{DBLP:journals/corr/AmodeiOSCSM16}
Amodei, D., Olah, C., Steinhardt, J., Christiano, P., Schulman, J., and
  Man{\'{e}}, D.
\newblock Concrete problems in {AI} safety.
\newblock \emph{CoRR}, abs/1606.06565, 2016.
\newblock URL \url{http://arxiv.org/abs/1606.06565}.

\bibitem[Andre \& Russell(2001)Andre and Russell]{andre2001programmable}
Andre, D. and Russell, S.~J.
\newblock Programmable reinforcement learning agents.
\newblock In \emph{Advances in neural information processing systems}, pp.\
  1019--1025, 2001.

\bibitem[Andre \& Russell(2002)Andre and Russell]{andre2002state}
Andre, D. and Russell, S.~J.
\newblock State abstraction for programmable reinforcement learning agents.
\newblock In \emph{AAAI/IAAI}, pp.\  119--125, 2002.

\bibitem[Antoni \& Veanes(2017)Antoni and Veanes]{d39antoni2017the}
Antoni, L. and Veanes, M.
\newblock The power of symbolic automata and transducers.
\newblock In \emph{Computer Aided Verification, 29th International Conference
  (CAV'17)}. Springer, July 2017.

\bibitem[Baier \& Katoen(2008)Baier and Katoen]{10.5555/1373322}
Baier, C. and Katoen, J.-P.
\newblock \emph{Principles of Model Checking (Representation and Mind Series)}.
\newblock The MIT Press, 2008.
\newblock ISBN 026202649X.

\bibitem[Bellemare et~al.(2016)Bellemare, Srinivasan, Ostrovski, Schaul,
  Saxton, and Munos]{bellemare2016unifying}
Bellemare, M., Srinivasan, S., Ostrovski, G., Schaul, T., Saxton, D., and
  Munos, R.
\newblock Unifying count-based exploration and intrinsic motivation.
\newblock \emph{Advances in neural information processing systems},
  29:\penalty0 1471--1479, 2016.

\bibitem[Chevalier-Boisvert et~al.(2018)Chevalier-Boisvert, Willems, and
  Pal]{gym_minigrid}
Chevalier-Boisvert, M., Willems, L., and Pal, S.
\newblock Minimalistic gridworld environment for openai gym.
\newblock \url{https://github.com/maximecb/gym-minigrid}, 2018.

\bibitem[Devidze et~al.()Devidze, Radanovic, Kamalaruban, and
  Singla]{devidzeexplicable}
Devidze, R., Radanovic, G., Kamalaruban, P., and Singla, A.
\newblock Explicable reward design for reinforcement learning agents.

\bibitem[Finn et~al.(2016{\natexlab{a}})Finn, Christiano, Abbeel, and
  Levine]{DBLP:journals/corr/FinnCAL16}
Finn, C., Christiano, P., Abbeel, P., and Levine, S.
\newblock A connection between generative adversarial networks, inverse
  reinforcement learning, and energy-based models.
\newblock \emph{CoRR}, abs/1611.03852, 2016{\natexlab{a}}.
\newblock URL \url{http://arxiv.org/abs/1611.03852}.

\bibitem[Finn et~al.(2016{\natexlab{b}})Finn, Levine, and
  Abbeel]{DBLP:journals/corr/FinnLA16}
Finn, C., Levine, S., and Abbeel, P.
\newblock Guided cost learning: Deep inverse optimal control via policy
  optimization.
\newblock In \emph{International conference on machine learning}, pp.\  49--58.
  PMLR, 2016{\natexlab{b}}.

\bibitem[Flet-Berliac et~al.(2021)Flet-Berliac, Ferret, Pietquin, Preux, and
  Geist]{flet-berliac2021adversarially}
Flet-Berliac, Y., Ferret, J., Pietquin, O., Preux, P., and Geist, M.
\newblock Adversarially guided actor-critic.
\newblock In \emph{International Conference on Learning Representations}, 2021.
\newblock URL \url{https://openreview.net/forum?id=_mQp5cr_iNy}.

\bibitem[Fu et~al.(2018)Fu, Luo, and Levine]{fu2017learning}
Fu, J., Luo, K., and Levine, S.
\newblock Learning robust rewards with adverserial inverse reinforcement
  learning.
\newblock In \emph{International Conference on Learning Representations}, 2018.
\newblock URL \url{https://openreview.net/forum?id=rkHywl-A-}.

\bibitem[Goodfellow et~al.(2014)Goodfellow, Pouget-Abadie, Mirza, Xu,
  Warde-Farley, Ozair, Courville, and Bengio]{goodfellow2014generative}
Goodfellow, I., Pouget-Abadie, J., Mirza, M., Xu, B., Warde-Farley, D., Ozair,
  S., Courville, A., and Bengio, Y.
\newblock Generative adversarial nets.
\newblock In \emph{Advances in neural information processing systems}, pp.\
  2672--2680, 2014.

\bibitem[Hadfield-Menell et~al.(2017)Hadfield-Menell, Milli, Abbeel, Russell,
  and Dragan]{NIPS2017_32fdab65}
Hadfield-Menell, D., Milli, S., Abbeel, P., Russell, S.~J., and Dragan, A.
\newblock Inverse reward design.
\newblock In Guyon, I., Luxburg, U.~V., Bengio, S., Wallach, H., Fergus, R.,
  Vishwanathan, S., and Garnett, R. (eds.), \emph{Advances in Neural
  Information Processing Systems}, volume~30. Curran Associates, Inc., 2017.

\bibitem[Ho \& Ermon(2016)Ho and Ermon]{ho2016generative}
Ho, J. and Ermon, S.
\newblock Generative adversarial imitation learning.
\newblock In \emph{Advances in Neural Information Processing Systems}, pp.\
  4565--4573, 2016.

\bibitem[Icarte et~al.(2020)Icarte, Klassen, Valenzano, and
  McIlraith]{DBLP:journals/corr/abs-2010-03950}
Icarte, R.~T., Klassen, T.~Q., Valenzano, R.~A., and McIlraith, S.~A.
\newblock Reward machines: Exploiting reward function structure in
  reinforcement learning.
\newblock \emph{CoRR}, abs/2010.03950, 2020.
\newblock URL \url{https://arxiv.org/abs/2010.03950}.

\bibitem[Jeon et~al.(2018)Jeon, Seo, and Kim]{NEURIPS2018_943aa0fc}
Jeon, W., Seo, S., and Kim, K.-E.
\newblock A bayesian approach to generative adversarial imitation learning.
\newblock In Bengio, S., Wallach, H., Larochelle, H., Grauman, K.,
  Cesa-Bianchi, N., and Garnett, R. (eds.), \emph{Advances in Neural
  Information Processing Systems}, volume~31. Curran Associates, Inc., 2018.

\bibitem[Kingma \& Welling(2013)Kingma and Welling]{kingma2013auto}
Kingma, D.~P. and Welling, M.
\newblock Auto-encoding variational bayes.
\newblock \emph{arXiv preprint arXiv:1312.6114}, 2013.

\bibitem[Le et~al.(2018)Le, Jiang, Agarwal, Dud{\'{\i}}k, Yue, and
  III]{DBLP:journals/corr/abs-1803-00590}
Le, H.~M., Jiang, N., Agarwal, A., Dud{\'{\i}}k, M., Yue, Y., and III, H.~D.
\newblock Hierarchical imitation and reinforcement learning.
\newblock \emph{CoRR}, abs/1803.00590, 2018.
\newblock URL \url{http://arxiv.org/abs/1803.00590}.

\bibitem[Mnih et~al.(2015)Mnih, Kavukcuoglu, Silver, Rusu, Veness, Bellemare,
  Graves, Riedmiller, Fidjeland, Ostrovski, et~al.]{mnih2015human}
Mnih, V., Kavukcuoglu, K., Silver, D., Rusu, A.~A., Veness, J., Bellemare,
  M.~G., Graves, A., Riedmiller, M., Fidjeland, A.~K., Ostrovski, G., et~al.
\newblock Human-level control through deep reinforcement learning.
\newblock \emph{Nature}, 518\penalty0 (7540):\penalty0 529--533, 2015.

\bibitem[Ng \& Russell(2000)Ng and Russell]{Ng:2000:AIR:645529.657801}
Ng, A.~Y. and Russell, S.~J.
\newblock Algorithms for inverse reinforcement learning.
\newblock In \emph{Proceedings of the Seventeenth International Conference on
  Machine Learning}, ICML '00, pp.\  663--670, San Francisco, CA, USA, 2000.
  Morgan Kaufmann Publishers Inc.
\newblock ISBN 1-55860-707-2.
\newblock URL \url{http://dl.acm.org/citation.cfm?id=645529.657801}.

\bibitem[Ng et~al.(1999)Ng, Harada, and Russell]{Ng:1999:PIU:645528.657613}
Ng, A.~Y., Harada, D., and Russell, S.~J.
\newblock Policy invariance under reward transformations: Theory and
  application to reward shaping.
\newblock In \emph{Proceedings of the Sixteenth International Conference on
  Machine Learning}, ICML '99, pp.\  278--287, San Francisco, CA, USA, 1999.
  Morgan Kaufmann Publishers Inc.
\newblock ISBN 1-55860-612-2.
\newblock URL \url{http://dl.acm.org/citation.cfm?id=645528.657613}.

\bibitem[Parr \& Russell(1998)Parr and Russell]{parr1998reinforcement}
Parr, R. and Russell, S.~J.
\newblock Reinforcement learning with hierarchies of machines.
\newblock In \emph{Advances in neural information processing systems}, pp.\
  1043--1049, 1998.

\bibitem[Pathak et~al.(2017)Pathak, Agrawal, Efros, and
  Darrell]{pathak2017curiosity}
Pathak, D., Agrawal, P., Efros, A.~A., and Darrell, T.
\newblock Curiosity-driven exploration by self-supervised prediction.
\newblock In \emph{International conference on machine learning}, pp.\
  2778--2787. PMLR, 2017.

\bibitem[Peters \& Schaal(2008)Peters and Schaal]{peters2008reinforcement}
Peters, J. and Schaal, S.
\newblock Reinforcement learning of motor skills with policy gradients.
\newblock \emph{Neural networks}, 21\penalty0 (4):\penalty0 682--697, 2008.

\bibitem[Pierce(2002)]{10.5555/509043}
Pierce, B.~C.
\newblock \emph{Types and Programming Languages}.
\newblock The MIT Press, 1st edition, 2002.
\newblock ISBN 0262162091.

\bibitem[Raileanu \& Rocktäschel(2020)Raileanu and
  Rocktäschel]{Raileanu2020RIDE:}
Raileanu, R. and Rocktäschel, T.
\newblock Ride: Rewarding impact-driven exploration for procedurally-generated
  environments.
\newblock In \emph{International Conference on Learning Representations}, 2020.

\bibitem[Ramachandran \& Amir(2007)Ramachandran and
  Amir]{ramachandran2007bayesian}
Ramachandran, D. and Amir, E.
\newblock Bayesian inverse reinforcement learning.
\newblock \emph{Urbana}, 51\penalty0 (61801):\penalty0 1--4, 2007.

\bibitem[Ratliff et~al.(2006)Ratliff, Bagnell, and
  Zinkevich]{Ratliff:2006:MMP:1143844.1143936}
Ratliff, N.~D., Bagnell, J.~A., and Zinkevich, M.~A.
\newblock Maximum margin planning.
\newblock In \emph{Proceedings of the 23rd International Conference on Machine
  Learning}, ICML '06, pp.\  729--736, New York, NY, USA, 2006. ACM.
\newblock ISBN 1-59593-383-2.
\newblock \doi{10.1145/1143844.1143936}.
\newblock URL \url{http://doi.acm.org/10.1145/1143844.1143936}.

\bibitem[Riedmiller et~al.(2018)Riedmiller, Hafner, Lampe, Neunert, Degrave,
  Wiele, Mnih, Heess, and Springenberg]{riedmiller2018learning}
Riedmiller, M., Hafner, R., Lampe, T., Neunert, M., Degrave, J., Wiele, T.,
  Mnih, V., Heess, N., and Springenberg, J.~T.
\newblock Learning by playing solving sparse reward tasks from scratch.
\newblock In \emph{International Conference on Machine Learning}, pp.\
  4344--4353. PMLR, 2018.

\bibitem[Schulman et~al.(2017)Schulman, Wolski, Dhariwal, Radford, and
  Klimov]{DBLP:journals/corr/SchulmanWDRK17}
Schulman, J., Wolski, F., Dhariwal, P., Radford, A., and Klimov, O.
\newblock Proximal policy optimization algorithms.
\newblock \emph{CoRR}, abs/1707.06347, 2017.
\newblock URL \url{http://arxiv.org/abs/1707.06347}.

\bibitem[Silver et~al.(2016)Silver, Huang, Maddison, Guez, Sifre, van~den
  Driessche, Schrittwieser, Antonoglou, Panneershelvam, Lanctot, Dieleman,
  Grewe, Nham, Kalchbrenner, Sutskever, Lillicrap, Leach, Kavukcuoglu, Graepel,
  and Hassabis]{Silver:2016aa}
Silver, D., Huang, A., Maddison, C.~J., Guez, A., Sifre, L., van~den Driessche,
  G., Schrittwieser, J., Antonoglou, I., Panneershelvam, V., Lanctot, M.,
  Dieleman, S., Grewe, D., Nham, J., Kalchbrenner, N., Sutskever, I.,
  Lillicrap, T., Leach, M., Kavukcuoglu, K., Graepel, T., and Hassabis, D.
\newblock Mastering the game of go with deep neural networks and tree search.
\newblock \emph{Nature}, 529\penalty0 (7587):\penalty0 484--489, 01 2016.
\newblock URL \url{http://dx.doi.org/10.1038/nature16961}.

\bibitem[Tian et~al.(2020)Tian, Ellis, Kryven, and Tenenbaum]{tian2020learning}
Tian, L., Ellis, K., Kryven, M., and Tenenbaum, J.
\newblock Learning abstract structure for drawing by efficient motor program
  induction.
\newblock \emph{Advances in Neural Information Processing Systems}, 33, 2020.

\bibitem[Toro~Icarte et~al.(2019)Toro~Icarte, Waldie, Klassen, Valenzano,
  Castro, and McIlraith]{toro2019learning}
Toro~Icarte, R., Waldie, E., Klassen, T., Valenzano, R., Castro, M., and
  McIlraith, S.
\newblock Learning reward machines for partially observable reinforcement
  learning.
\newblock \emph{Advances in Neural Information Processing Systems},
  32:\penalty0 15523--15534, 2019.

\bibitem[Veanes et~al.(2012)Veanes, Hooimeijer, Livshits, Molnar, and
  Bjorner]{10.1145/2103621.2103674}
Veanes, M., Hooimeijer, P., Livshits, B., Molnar, D., and Bjorner, N.
\newblock Symbolic finite state transducers: Algorithms and applications.
\newblock \emph{SIGPLAN Not.}, 47\penalty0 (1):\penalty0 137--150, January
  2012.
\newblock ISSN 0362-1340.
\newblock \doi{10.1145/2103621.2103674}.
\newblock URL \url{https://doi.org/10.1145/2103621.2103674}.

\bibitem[Verma et~al.(2018)Verma, Murali, Singh, Kohli, and
  Chaudhuri]{verma2018programmatically}
Verma, A., Murali, V., Singh, R., Kohli, P., and Chaudhuri, S.
\newblock Programmatically interpretable reinforcement learning.
\newblock In \emph{International Conference on Machine Learning}, pp.\
  5045--5054. PMLR, 2018.

\bibitem[Wainwright \& Jordan(2008)Wainwright and Jordan]{MAL-001}
Wainwright, M.~J. and Jordan, M.~I.
\newblock Graphical models, exponential families, and variational inference.
\newblock \emph{Foundations and Trends® in Machine Learning}, 1\penalty0
  (1–2):\penalty0 1--305, 2008.
\newblock ISSN 1935-8237.
\newblock \doi{10.1561/2200000001}.
\newblock URL \url{http://dx.doi.org/10.1561/2200000001}.

\bibitem[Yang et~al.(2021)Yang, Inala, Bastani, Pu, Solar{-}Lezama, and
  Rinard]{DBLP:journals/corr/abs-2102-11137}
Yang, Y., Inala, J.~P., Bastani, O., Pu, Y., Solar{-}Lezama, A., and Rinard, M.
\newblock Program synthesis guided reinforcement learning.
\newblock \emph{CoRR}, abs/2102.11137, 2021.

\bibitem[Zhou \& Li(2018)Zhou and Li]{zhou2018safety}
Zhou, W. and Li, W.
\newblock Safety-aware apprenticeship learning.
\newblock In \emph{International Conference on Computer Aided Verification},
  pp.\  662--680. Springer, 2018.

\bibitem[Zhou \& Li(2021)Zhou and Li]{zhou2021programmatic}
Zhou, W. and Li, W.
\newblock Programmatic reward design by example.
\newblock \emph{arXiv preprint arXiv:2112.08438}, 2021.

\bibitem[Zhu et~al.(2019)Zhu, Xiong, Magill, and
  Jagannathan]{DBLP:journals/corr/abs-1907-07273}
Zhu, H., Xiong, Z., Magill, S., and Jagannathan, S.
\newblock An inductive synthesis framework for verifiable reinforcement
  learning.
\newblock In \emph{Proceedings of the 40th ACM SIGPLAN Conference on
  Programming Language Design and Implementation}, pp.\  686--701, 2019.

\bibitem[Ziebart et~al.(2008)Ziebart, Maas, Bagnell, and
  Dey]{Ziebart:2008:MEI:1620270.1620297}
Ziebart, B.~D., Maas, A., Bagnell, J.~A., and Dey, A.~K.
\newblock Maximum entropy inverse reinforcement learning.
\newblock In \emph{Proceedings of the 23rd National Conference on Artificial
  Intelligence - Volume 3}, AAAI'08, pp.\  1433--1438. AAAI Press, 2008.
\newblock ISBN 978-1-57735-368-3.
\newblock URL \url{http://dl.acm.org/citation.cfm?id=1620270.1620297}.

\end{thebibliography}
\bibliographystyle{icml2022}

\newpage
\appendix

\onecolumn
\clearpage
\section{Appendix}

In this appendix, we will present additional experimental results; the design details of the SRMs and the symbolic constraints used in the experiments; a detailed experimental setup including the hyperparameters. 

\subsection{Additional Results}
We show some addition experimental results in this section to answer the following questions.

\quad \textbf{E.} Can arbitrarily concretized SRM effectively train RL agents?

\quad \textbf{F.} How much do the performance of Algorithm 1 depend on the designs of the SRMs?

For question \textbf{E}, we randomly generate hole assignments that satisfy the symbolic constraints for the SRMs of the DoorKey and KeyCorridor tasks. The SRMs are shown in Fig.\ref{fig1_2_} and \ref{fig2_2}. The symbolic constraints contain the relational predicates as shown in Table.\ref{tab2} and \ref{tab6_1}. Those SRMs and symbolic constraints produce the main results in the main text. Now the assignments are generated by only optimizing the supervised objective $J_{con}$ mentioned in the main text. The concretized SRMs are used for training RL policies in the following large DoorKey and KeyCorridor environments.
\begin{figure}[ht]
     \centering
    \subfigure[DoorKey-16x16]{
         \includegraphics[height=5.2cm,width=6.5cm]{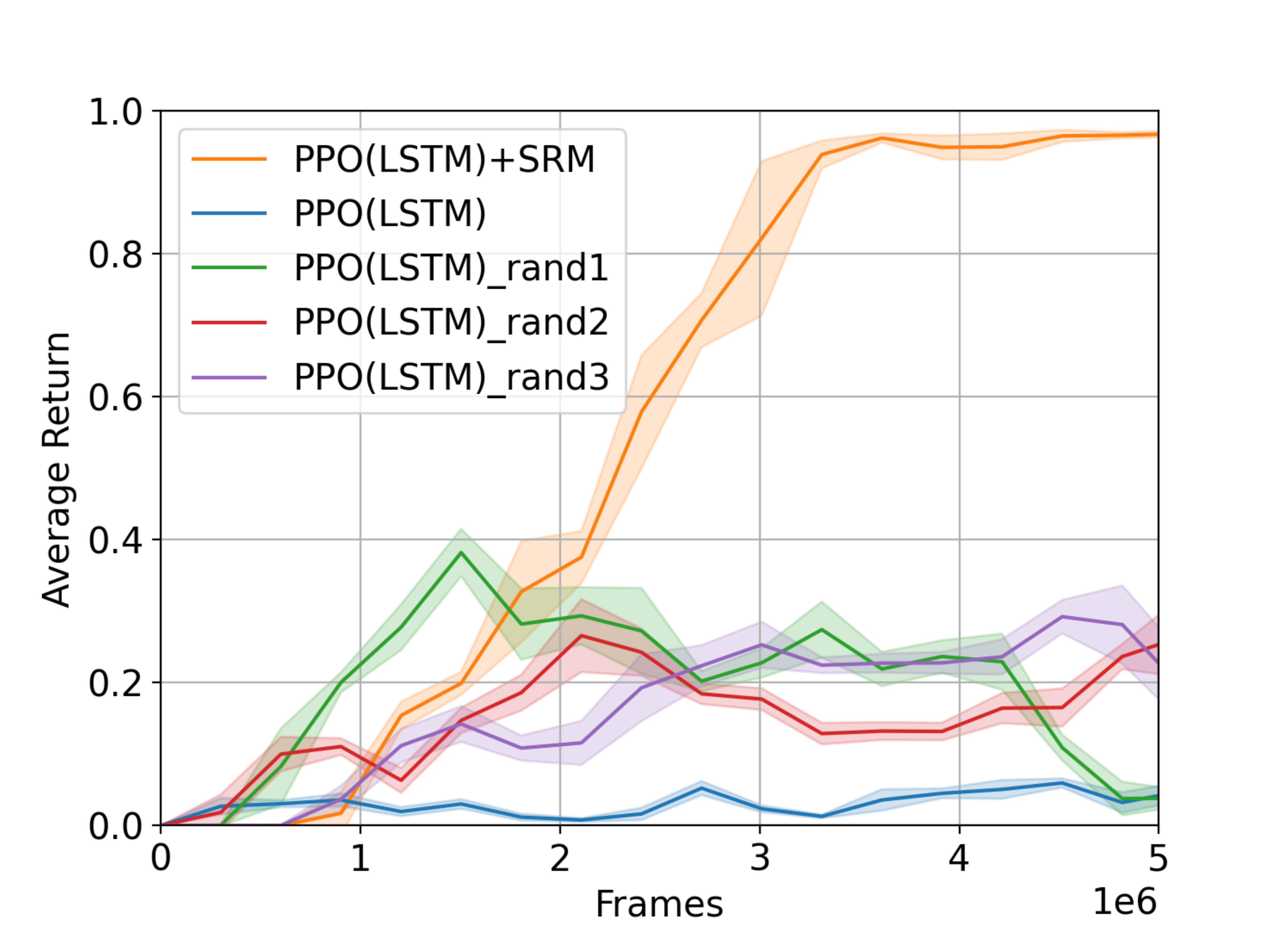}
        }     
    \subfigure[KeyCorridorS4R3]{
        \includegraphics[height=5.2cm,width=6.5cm]{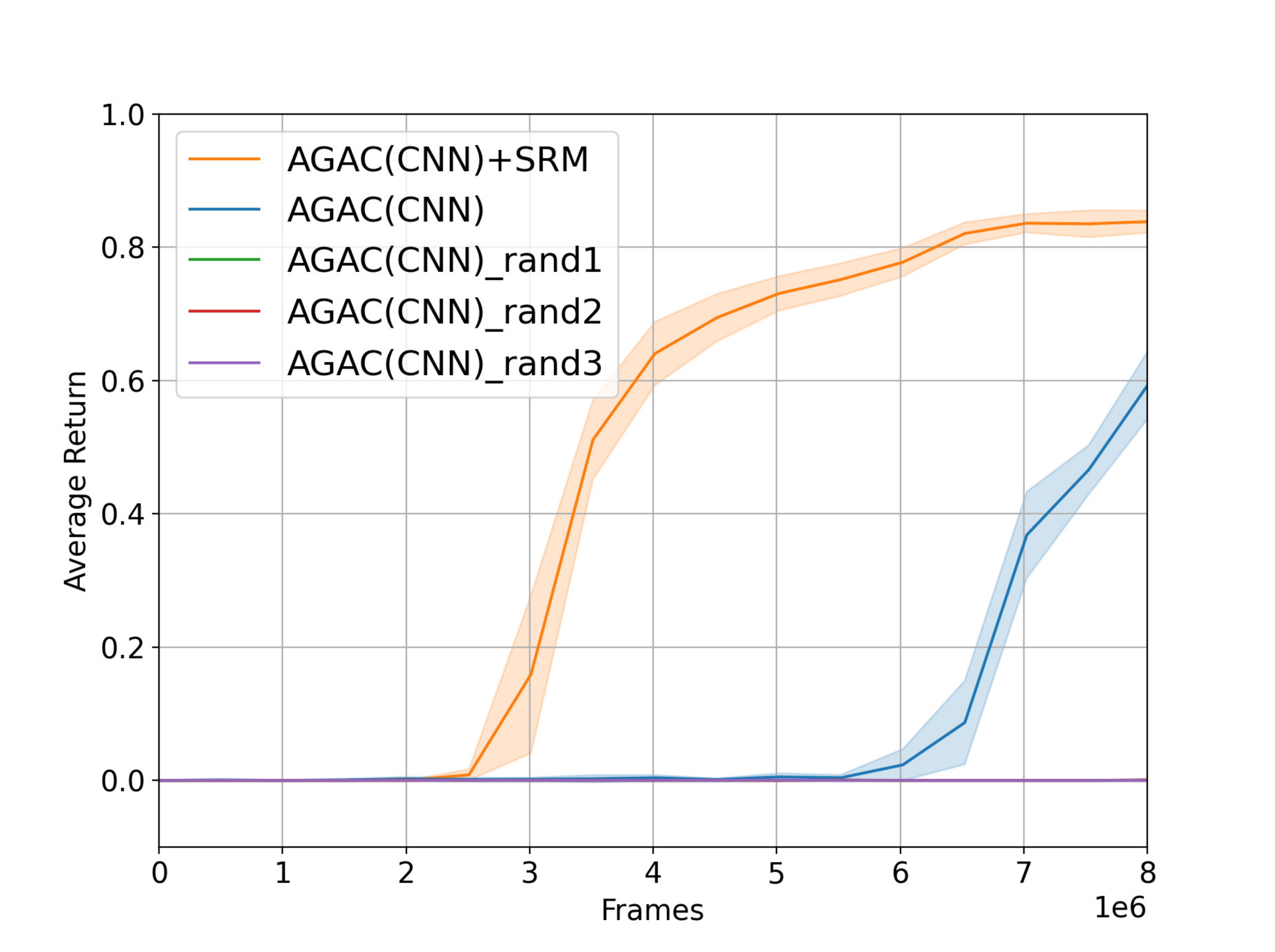}
         
     }    
     
     \caption{AGAC/PPO+SRM indicates that the hole assignments are learned via Algorithm 1; AGAC/PPO\_rand\# with an index \# indicates that the holes are randomly assigned with some values that satisfy the symbolic constraint for that task. CNN and LSTM indicate the versions of the actor-critic networks.} \label{fig6_4}\label{fig6_5}
\end{figure}
\begin{figure}[ht]
 \centering
    \subfigure[ObstructedMaze-2Dhlb]{
         \includegraphics[height=5.2cm,width=6.5cm]{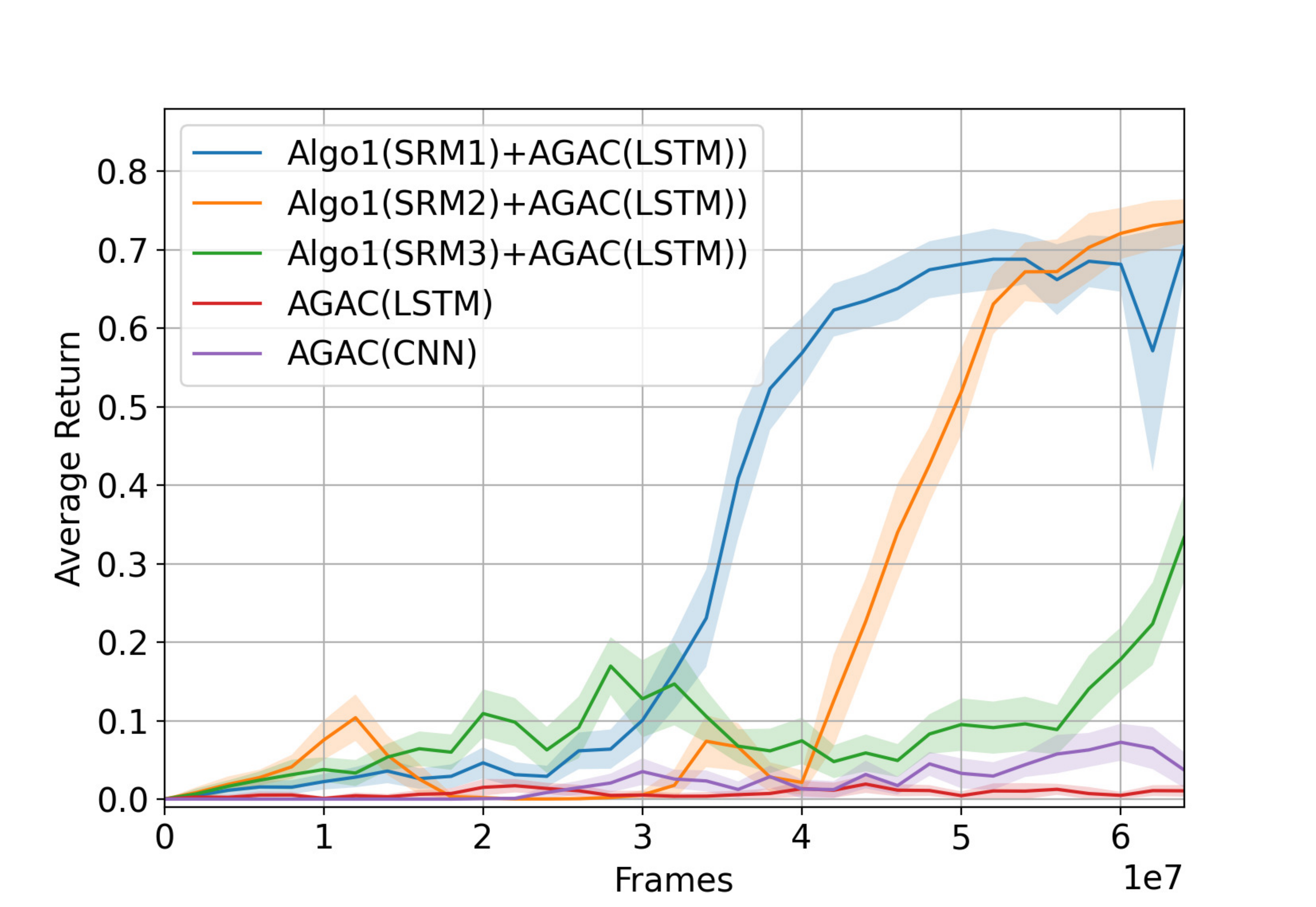}
        
     }
    \subfigure[ObstructedMaze-Full]{
           \includegraphics[height=5.2cm,width=6.5cm]{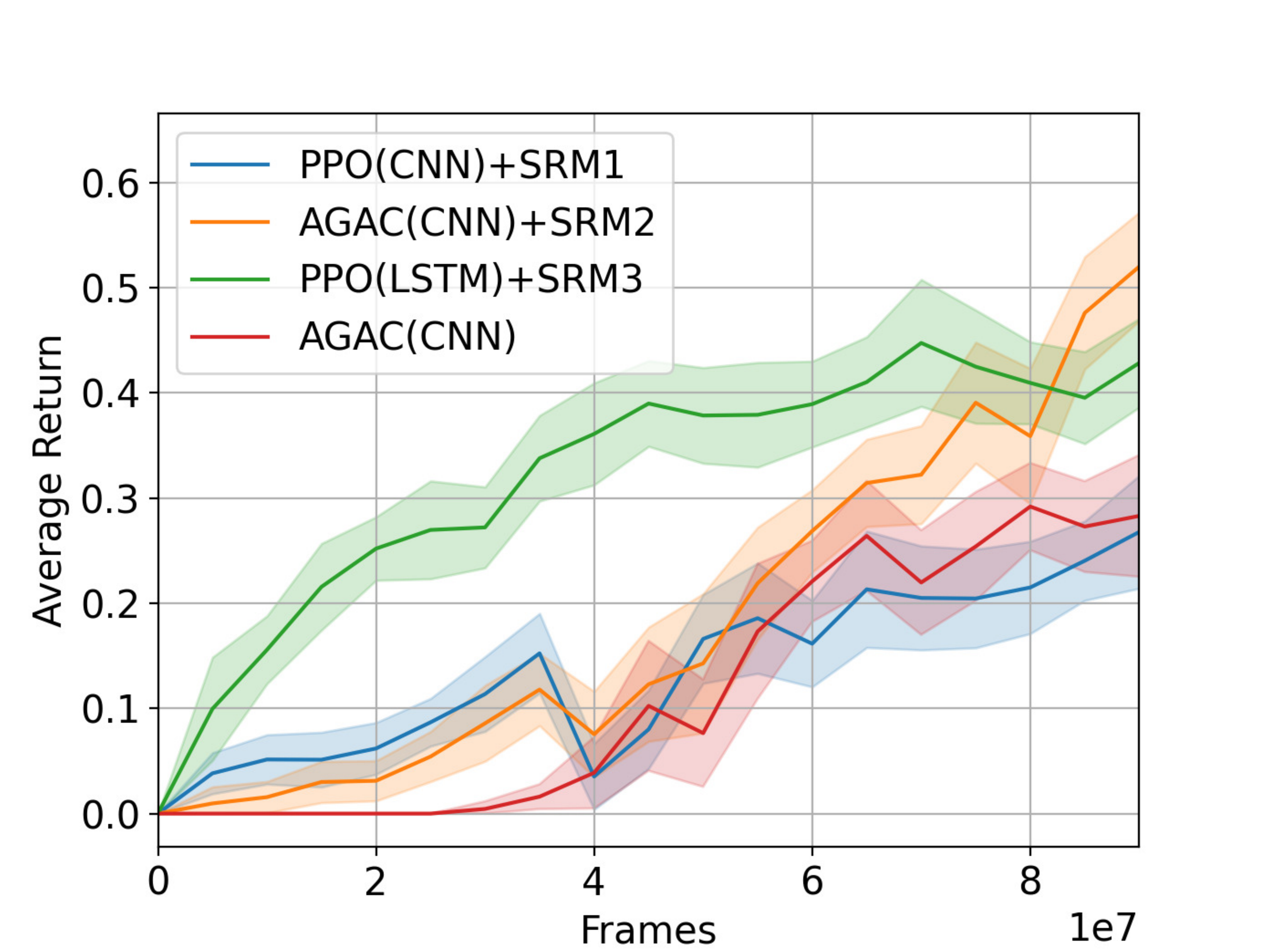}
         
     }
 \caption{ Algo1(SRM\#)+AGAC(LSTM) with an index $\#=1\sim 3$ indicates running Algorithm 1 with those three designed SRMs and by using AGAC in line 4 of Algorithm 1. PPO/AGAC+SRM\# indicates training RL agents with SRM\# by using PPO or AGAC algorithm. CNN and LSTM indicate the versions of the actor-critic networks.} \label{fig6_6}\label{fig6_9}\label{fig6_10}
\end{figure}

\begin{itemize}
\item \textbf{DoorKey-16x16 }. In Fig.\ref{fig6_4}(a), we test three 3 randomly generated hole assignments for the SRM, each annotated by PPO(LSTM)\_rand\# . The PPO(LSTM) agents trained with those SRMs achieve certain level of performance than that trained with the default reward. However, the SRM concretized with a learned hole assignment, annotated by PPO(LSTM)+SRM, enables the agent to attain much higher performance with much lower amount of frames.
\item \textbf{KeyCorridorS4R4 }. We test 3 randomly generated hole assignments for  the SRMs, each  annotated by AGAC(CNN)\_rand\#.   As in Fig.\ref{fig6_5}(b), the agents trained with the SRMs with random assignments do not perform at all. In contrast, the agent trained with the SRM that is concretized with a learned hole assignment achieves high performance with comparable amount of frames to that trained with the default reward. 
\end{itemize} 

For question \textbf{F}, as mentioned in the main text we design three SRMs for the ObstructedMaze task. We will describe the difference between these SRMs in the next section. We run Algorithm 1 with those SRMs in the ObstructedMaze-2Dhlb environment and compare the results in Fig.\ref{fig6_6}(a). In Fig.\ref{fig6_10}(b), we use those concretized SRMs to train RL agents in ObstructedMaze-Full. However, the SRM1 that achieves highest performance in Fig.\ref{fig6_10}(c) is outperformed by two others.

Besides answering those two questions, we recall that we run Algorithm 1 in DoorKey and KeyCorridor tasks without symbolic constraint and with weaker symbolic constraint in the ablation study of the main text. Under the same conditions, we vary the number of demonstrations and check the number of frames needed for $\pi_A$ to attain high performance.  In Fig.\ref{fig6_11}(a) and Fig.\ref{fig6_11}(b), we show that when  the number of examples is reduced from $10$ to $1$,  number of frames that Algorithm 1 needs to produce a policy with average return of at least $0.8$ are not severely influenced. 
\begin{figure}[ht]
     \centering
    \subfigure[DoorKey-8x8]{ 
         \centering
        \includegraphics[height=5.2cm,width=6.5cm]{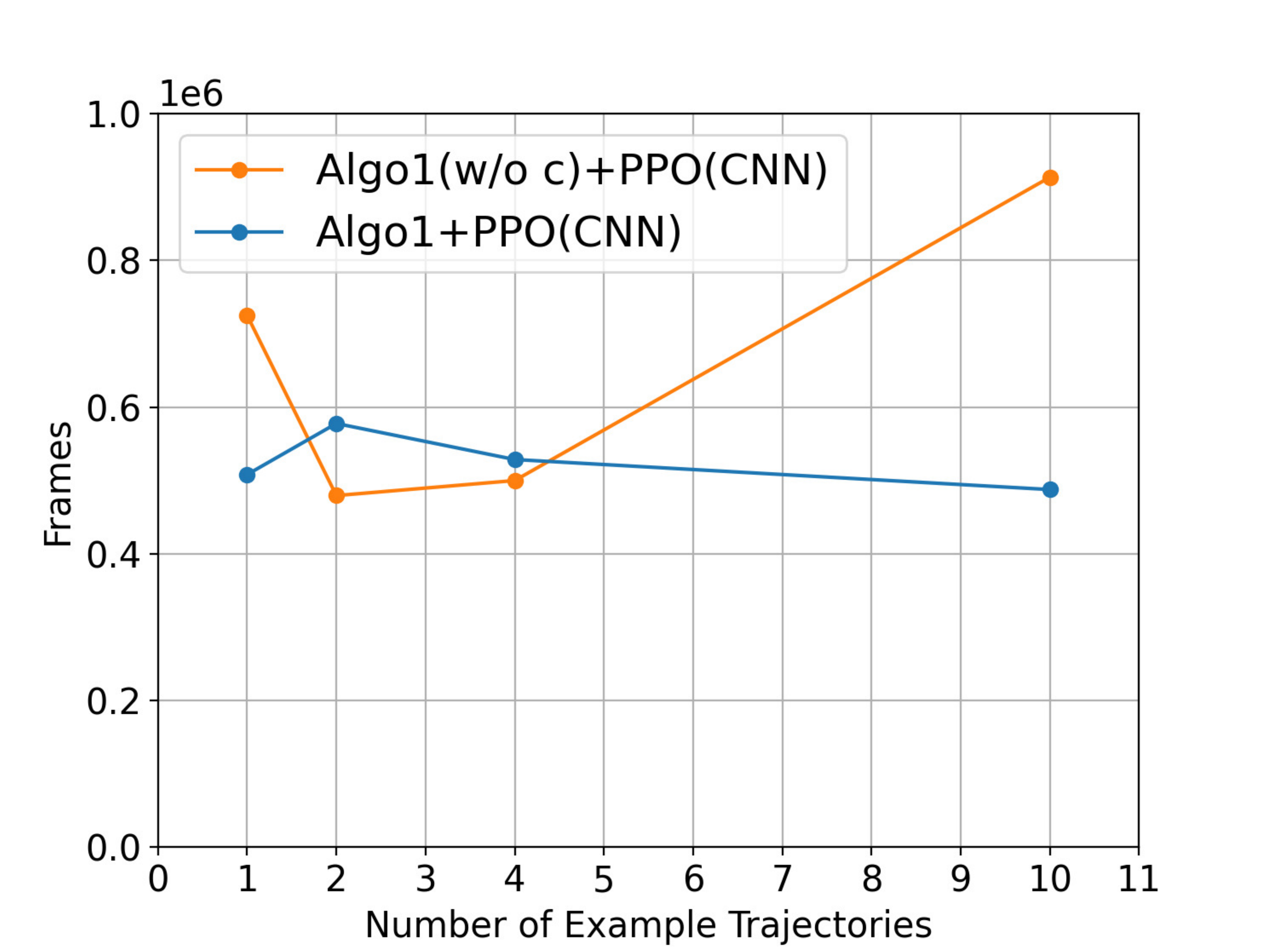}

     }
    \subfigure[KeyCorridorS4R3]{ 
         \centering
        \includegraphics[height=5.2cm,width=6.5cm]{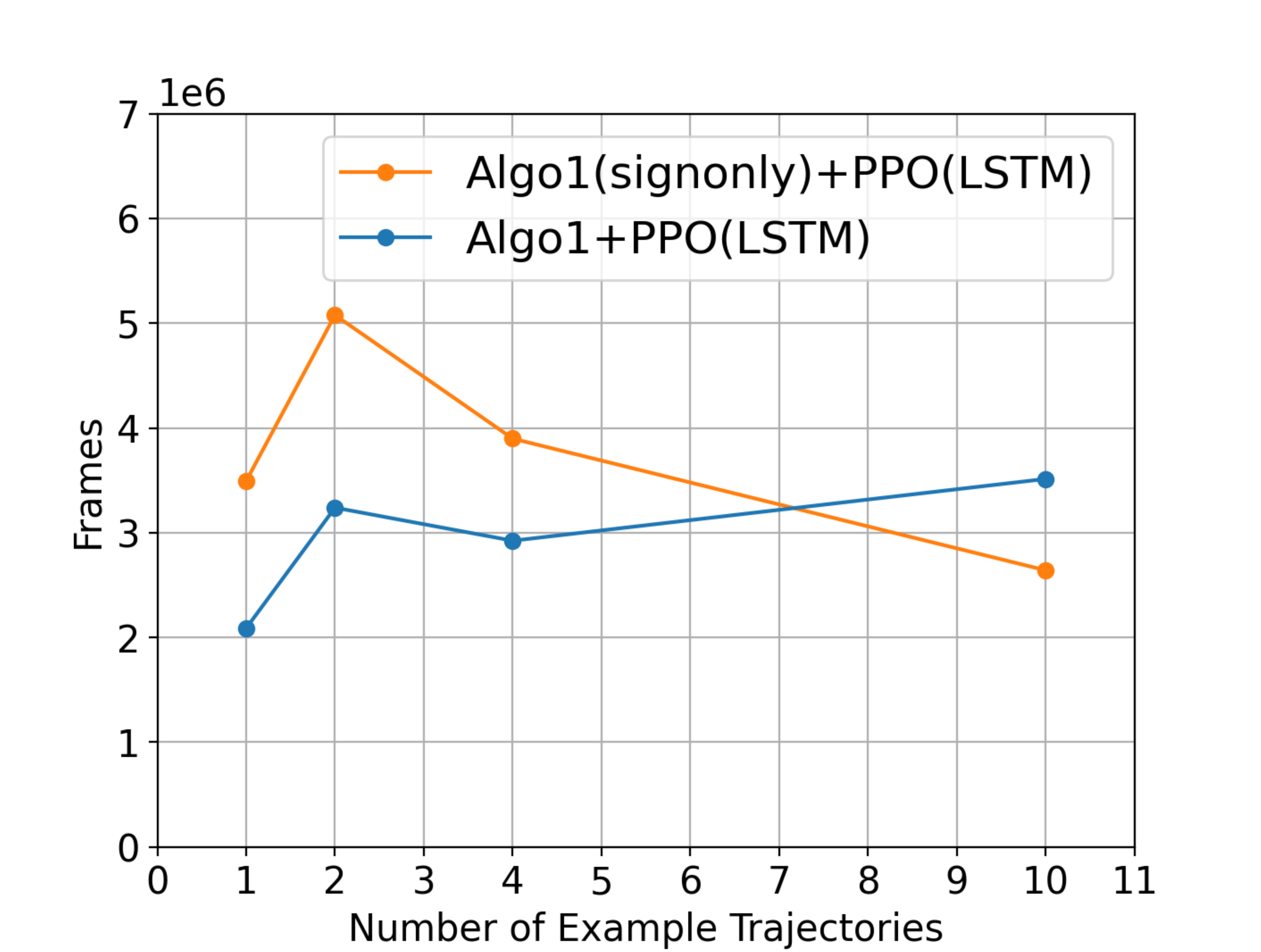}
         
     }
     \caption{Algo1+PPO(CNN) indicates using PPO as the policy learning algorithm in line 4 of Algorithm 1; Algo1(w/o c)+PPO(CNN) indicates that running Algorithm 1 without symbolic constraint while using PPO(CNN) in line 4; Algo1(signonly)+PPO(CNN) indicates that running Algorithm 1 without symbolic constraint while using PPO(CNN) in line 4;CNN indicates CNN version of the actor-critic networks.}\label{fig6_11} 
         
\end{figure}

\subsection{Design Details of the SRMs}
In this section, we show the diagrams of the SRMs as well as the symbolic constraints designed for the tasks. We will explain the design patterns in those SRMs in detail.

\subsubsection{DoorKey Task}

\begin{figure*}[ht]
         \centering
         \includegraphics[height=3.2cm,width=12cm]{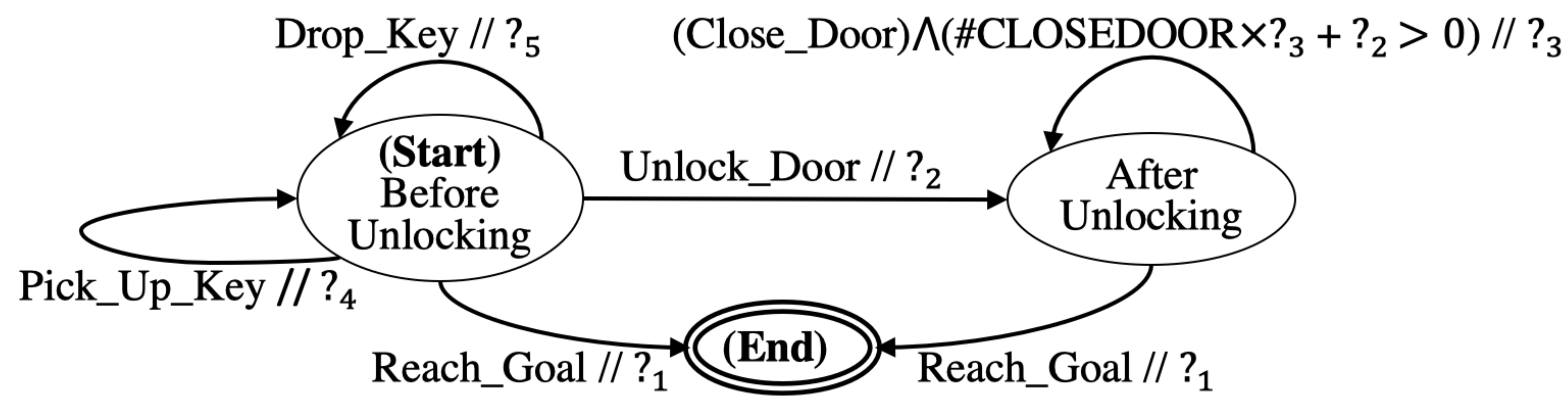}
     \caption{The diagram of the SRM designed for the DoorKey task.}\label{fig1_2_}
\end{figure*}

\begin{table*}[ht]
    \centering
    \begin{tabular}{l|l}
            \textbf{Properties} & \textbf{Predicates} \\
            $[\mu_1]$Reward reaching the goal & $\mathtt{\bigwedge\limits^{5}_{\mathtt{id}=1}(\mathtt{?_{id}} \leq \mathtt{?_1})}$ \\
            $[\mu_2]$Penalize dropping unused key & $\mathtt{?_5+?_4\leq 0}$\\
            $[\mu_3]$Reward unlocking door & $\mathtt{\bigwedge\limits^5_{\mathtt{id}=2}({\mathtt{?_{id}}}\leq \mathtt{?_2})}$\\
            $[\mu_4]$Penalty for closing door & $\mathtt{\mathtt{?_3}\leq 0}$\\
            $[\mu_5]$Mildly penalize door toggling &$\mathtt{\mathtt{?_3}+\mathtt{?_2}\leq 0}$
    \end{tabular}
    \caption{The correspondence between properties and atomic predicates for the DoorKey SRM in Fig.\ref{fig1_2}}
    \label{tab2}
\end{table*}

We show the diagram of the SRM for DooKey in Fig.\ref{fig1_2_}. This SRM implicitly identifies an unlocking-door sub-task with two internal states ``$\mathtt{Before\ Unlocking}$" and ``$\mathtt{After\ Unlocking}$". The transitions are designed mostly based on high level human insights represented in first order logic: a) $\mathtt{(Reach\_Goal@t)}\mapsto\mathtt{\exists t_1<t.\exists t_2<t_1.}\mathtt{(Unlock\_Door @t_1)} \wedge\mathtt{(Pick\_up\_Key @t_2)}$ where $\mathtt{@t}$ indicates that the predicate preceding it, e.g., $\mathtt{Reach\_Goal}$, operates on the time step $t$ of the trajectory $\tau$; b) $\mathtt{\forall t\in[ t_1, t_2].}$ $\mathtt{(Drop\_Key@t_1,Before\ Unlocking)}\wedge\mathtt{(\neg Pick\_up\_Key @t)}\wedge \mathtt{(Pick\_up\_Key@t_2)}\mapsto\mathtt{\forall t'\in[ t_1, t_2].}\mathtt{(\neg Unlock\_Door@t')}$ where we additionally integrate the internal state, i.e., ``$\mathtt{Before\ Unlocking}$", next to $\mathtt{@t_1}$, to indicate the internal state at the time step $t_1$. The predicate $\mathtt{\#CLOSEDOOR\times ?_3 + ?_2}>0$ in Fig.\ref{fig1_2} is introduced with due consideration of avoiding overly penalizing the agent for closing the door, which behavior is redundant for the task. The underlying idea is: \textit{if the reward function penalized an under-trained RL agent for every door closing behavior with some high penalty $\mathtt{?_3}<0$ for a total of $\mathtt{\#CLOSEDOOR}$ amount of times, and the accumulated penalty $\mathtt{\#CLOSEDOOR\times ?_3}$ outweighed the reward $\mathtt{?_2}>0$ for unlocking the door, then the agent in practice might be inclined to reside away from the door for good.} The SRM in Fig.\ref{fig1_2} simply upper-bounds the accumulated penalty to avoid negative effects in practice. Then we show the atomic predicates in the symbolic constraint for this task in Table.\ref{tab2}. The final symbolic constraint is $c=\bigwedge^5_{i=1}\mu_i$. We omit the explanation for the  symbolic constraint since the atomic predicates are self-explanatory.

\subsubsection{KeyCorridor Task}

\begin{figure*}[ht]
         \centering
         \includegraphics[height=4.5cm,width=14cm]{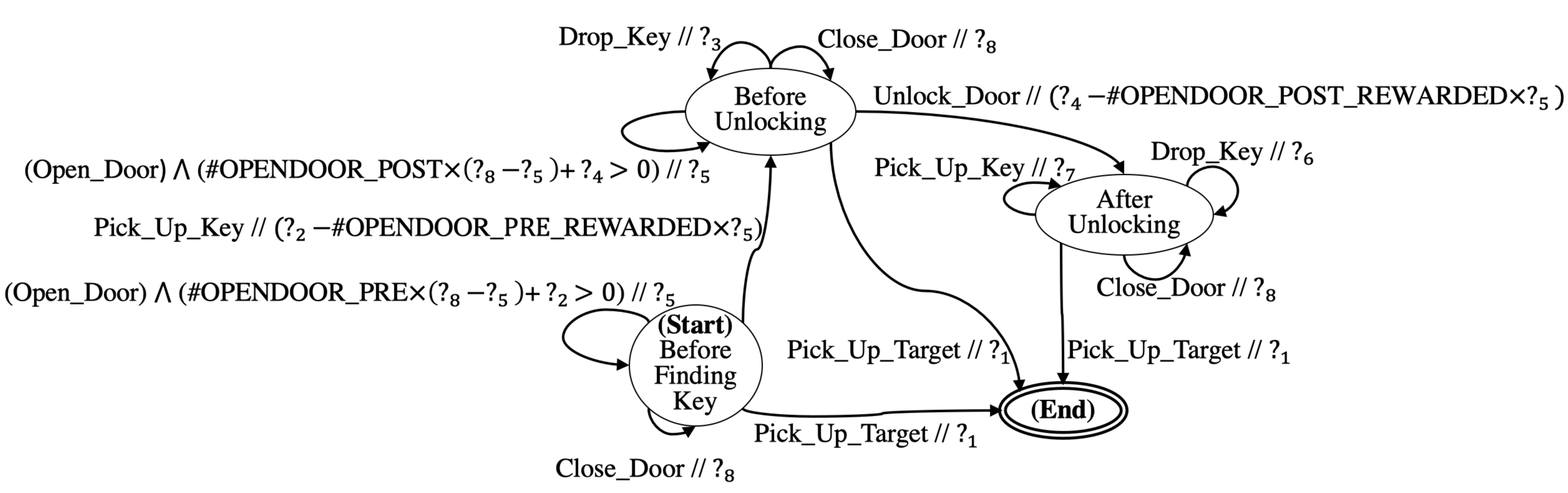}
         \caption{The diagram of the SRM designed for the KeyCorridor task.}
         \label{fig2_2}
\end{figure*}

\begin{table*}[ht]
        \centering
        \begin{tabular}{l|l|l}
            \textbf{Properties} & \textbf{(Retaional) Predicates} & \textbf{(Non-Relational) Predicates}\\
            $[\mu_1]$Reward picking up ball & $\mathtt{\bigwedge\limits^8_{id=2}(\mathtt{?_{id}}\leq \mathtt{?_1})}$ & $\mathtt{?_1}\geq 0$\\
            $[\mu_2]$Reward picking up key & $\mathtt{ \mathtt{?_2}\geq 0}$ & $\mathtt{ \mathtt{?_2}\geq 0}$\\
            $[\mu_3]$Reward dropping used key & $\mathtt{ \mathtt{?_3}\geq 0}$  & $\mathtt{ \mathtt{?_3}\geq 0}$\\
            $[\mu_4]$Reward unlocking door & $\mathtt{ \mathtt{?_4}\geq 0}$ & $\mathtt{ \mathtt{?_4}\geq 0}$ \\
            $[\mu_5]$Encourage opening door & $\mathtt{ \mathtt{?_5}\geq 0}$  & $\mathtt{ \mathtt{?_5}\geq 0}$\\
            $[\mu_6]$Penalize meaningless move & $\mathtt{\mathtt{?_8}\leq0}$ & $\mathtt{\mathtt{?_8}\leq0}$\\
            $[\mu_7]$Moderately reward opening door & $\mathtt{\mathtt{?_5}-\mathtt{?_8}\leq \mathtt{?_2}}$ &\\
            $[\mu_8]$Penalize dropping unused key &$\mathtt{\mathtt{?_2}+\mathtt{?_6}\leq 0}$ &$\mathtt{?_6}\leq 0$\\
            $[\mu_9]$Penalize picking up used key &$\mathtt{\mathtt{?_3}+\mathtt{?_7}\leq 0}$ & $\mathtt{?_7\leq 0}$\\
        \end{tabular}
        \caption{The correspondence between properties and the relational and non-relational atomic predicates for the SRM of KeyCorridor in Fig.\ref{fig2_2}}
        \label{tab6_1}
    \end{table*}
We depict in Fig.\ref{fig2_2} the diagram of the SRM designed for this task. Due to the added complexity in this task in comparison with the DoorKey task, two sub-tasks, finding-key and unlocking-door, are implicitly established by using three internal states ``$\mathtt{Before\ Finding\ Key}$", ``$\mathtt{Before\ Unlocking}$" and ``$\mathtt{After\_Unlocking}$''. Some important first order logic formulas that hold in most situations in the KeyCorridor task include: a) $\mathtt{(Pick\_Up\_Target@t_1)}\wedge\mathtt{(Unlock\_Door@t_2)}\mapsto\mathtt{\exists t\in[t_2, t_1].}\mathtt{(Drop\_Key@t)}$; b) $\mathtt{\forall t'<t.}\mathtt{(Pick\_Up\_Key@t)}\wedge\mathtt{(\neg Pick\_Up\_Key@t')}\mapsto\mathtt{\exists t''<t.}\mathtt{(Open\_a\_Door@t'')}$; c) $\mathtt{(Unlock\_Door@t)}\mapsto\mathtt{\exists t'<t.}\mathtt{(Open\_a\_Door@t'')}$. Regarding the implication a, two predicates $\mathtt{Pick\_Up\_Key@t}$ and $\mathtt{Drop\_Key@t}$ are added at the internal state ``$\mathtt{After\ Unlocking}$" to govern the rewards returned for their respectively concerned behaviors after the door is unlocked. As for the implications b and c, the caveat is to determine the utility of each door opening behavior. A designer may go to one extremity by rewarding every door opening behavior with some constant, which, however, either represses exploration by penalizing opening door, or oppositely raises reward hacking, i.e., agent accumulates reward by exhaustively searching for doors to open.  Alternatively, the designer may go to another extremity by carrying out a motion planning and specify the solution in the SRM, which, however, is cumbersome and cannot be generalized. In this paper, we highlight a economical design pattern to circumvent such non-determinism.

As shown in Fig.\ref{fig2_2}, before the agent accomplishes the finding-key sub-task, i.e., in the ``$\mathtt{Before\ Finding\ Key}$" internal state, once the agent opens a door, the predicate $\mathtt{\#OPENDOOR\_PRE\times(?_8-?_5)+?_2> 0?}$ checks whether the total reward gained from opening doors is about to exceed a threshold. 
The counter $\mathtt{\#OPENDOOR\_PRE}$ counts the number of times that agent opens doors prior to the agent finding the key; the variable $\mathtt{?_8}$ is expected to be a penalty for the agent closing a door, which is redundant. By introducing $\mathtt{?_8}$, we specify that even if the agent closed doors instead of opening doors for equal number $\mathtt{\#CLOSEDOOR\_PRE\equiv \#OPENDOOR\_PRE}$ of times, the agent could still gain positive net reward by finishing the finding-key sub-task, i.e., $\mathtt{\#CLOSEDOOR\_PRE\times?_8+?_2\geq \#OPENDOOR\_PRE\times?_5}$. When the agent accomplishes the finding-key sub-task, i.e., transitioning to the ``$\mathtt{Before\ Unlocking}$" internal state, the reward $\mathtt{?_2-\#OPENDOOR\_PRE\_REWARDED\times ?_3}$ subtracts the reward hitherto gained from opening doors with $\mathtt{\#OPENDOOR\_PRE\_REWARDED\leq \#OPENDOOR\_PRE}$ counting the number of times that door opening behaviors are indeed awarded prior to the agent finding the key. In some sense, this approach amortizes the reward $\mathtt{?_2}$ for finishing the finding-key sub-task over the door opening behaviors. The amortized reward $\mathtt{ \#OPENDOOR\_PRE\_REWARDED\times ?_3}$ cannot exceed $\mathtt{?_2}$ and should be deducted from $\mathtt{?_2}$. The same idea is adopted to award the door opening behaviors prior to the agent unlocking the door. The counter $\mathtt{\#OPENDOOOR\_POST}$ in Fig.\ref{fig2_2} counts the number of times that agent opens doors after the agent finding the key prior to the agent unlocking the door; $\mathtt{\#OPENDOOR\_PRE\_REWARDED}$ counts the number of times that door opening behaviors are awarded within that time interval. Apparently, such design pattern is convenient enough to be implemented via symbolic means. The challenge, however, remains to properly determine values for $\mathtt{?_{id}}$'s. Then we show the atomic predicates in the symbolic constraint for this task in Table.\ref{tab6_1}. The final symbolic constraint is $c=\bigwedge^9_{i=1}\mu_i$.

\subsubsection{ObstructedMaze Task}

Fig.\ref{fig3_2} shows the diagram of the reward function designed for the ObstructedMaze task. Despite of the complexity of task, there are only three internal states, ``$\mathtt{\textbf{(Start)}}$", ``$\mathtt{After\ Seeing\ the\ Target}$" and ``$\mathtt{\bf{(End)}}$". This is because only specifying the the sub-tasks is far from  adequate for this task.
\begin{figure*}[ht]
     \centering
         \centering
         \includegraphics[height=5.3cm,width=14. cm]{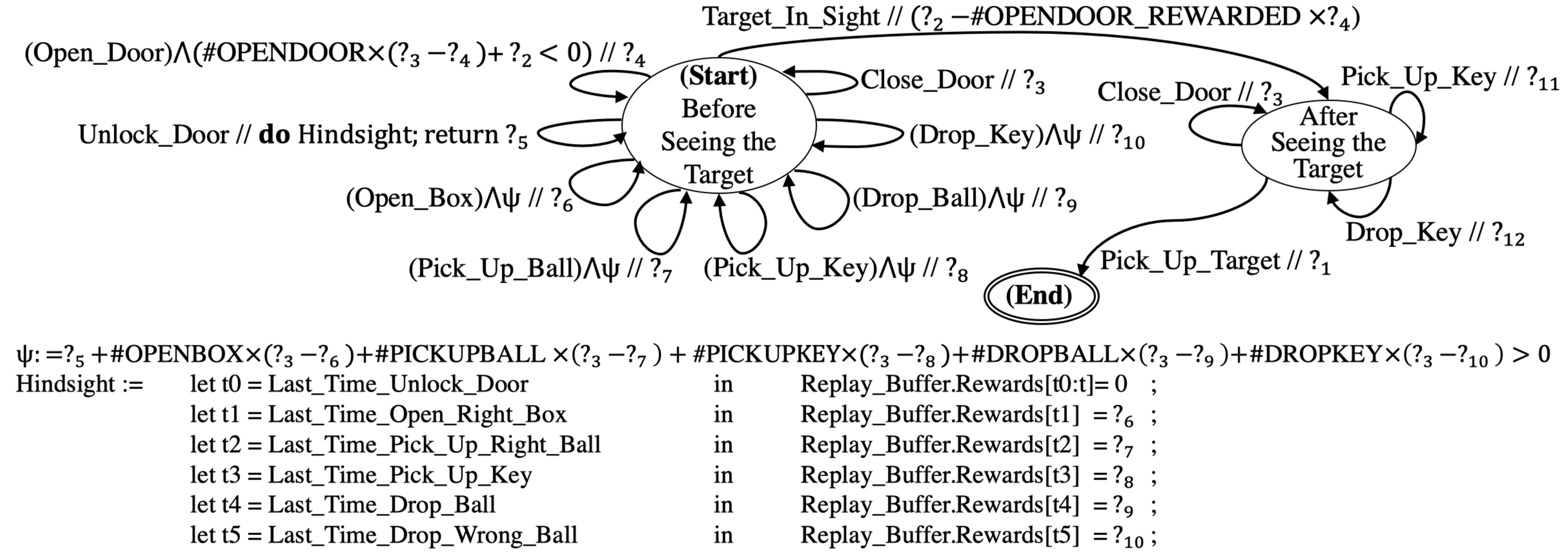}
         \caption{The diagram of the SRM designed for the ObstructedMaze task.}
         \label{fig3_2}
\end{figure*}
Once the ``$\mathtt{After\ Seeing\ the\ Target}$" state is reached, the reward function only concerns whether the agent drops or picks up a key or the target. In the ``$\mathtt{\textbf{(Start)}}$" state,  the reward function views each door unlocking behavior as a milestone. If the agent does not unlock a door at the present time step, the SRM awards the following agent behaviors: opening a box, picking up a ball, picking up a key, dropping a ball, dropping a key through a proposition $\mathtt{?_5+\#OPENBOX\times(?_3-?_6)+\#PICKUPBALL\times(?_3-?_7)+\ldots > 0}$ which bounds the number of times that the those behaviors are awarded. The counters $\mathtt{\#OPENBOX},\mathtt{\#PICKUPBALL}\ldots$ only count the number of respectively concerned behaviors between two successive door unlocking behaviors by resetting themselves to $\mathtt{0}$ once the agent unlocks a door. Thus far the design pattern is still similar to that adopted in the KeyCorridor tasks. What makes a difference here is that we assume the SRM to have access to the replay buffer of the agent policy, annotated as $\mathtt{Replay\_Buffer}$. Suppose that in some time step $\mathtt{t}$ the agent unlocks a $\mathtt{C}$ colored door located at coordinate $\mathtt{X}$, the SRM locates the last time step when the agent unlocked a door. Then it reassigns the rewards to $\mathtt{0}$ for all the opening a box, picking up a ball, picking up a key, dropping a ball, dropping a key behaviors stored in  $\mathtt{Replay\_Buffer}$ ever since that last door locking time step till the present time step. Then it identifies the time steps of four milestone behaviors based on the following human insights represented in first order logic: a) $\mathtt{(Unlock\_C\_Colored\_Door@t)}\mapsto \mathtt{\exists t_1< t}.\mathtt{(Open\_Box@t_1)}\wedge\mathtt{(Find\_C\_Colored\_Key@t_1
+1)}$, i.e.,  in time step $\mathtt{t_1}$ the agent opened the box that contains the key for this $C$ colored door; b) $\mathtt{(Unlock\_X\_Located\_Door@t)}\mapsto \mathtt{\exists t_2< t.}\mathtt{\forall t_2'>t_2.}\mathtt{(Pick\_Up\_X\_Located\_Ball@t_2)}\wedge\mathtt{(\neg Pick\_Up\_X\_Located\_Ball@t_2')}$, i.e., in time step $\mathtt{t_2}$ the agent picked up the ball obstructing this door at position $X$ for the last time; c) $\mathtt{(Unlock\_C\_Colored\_Door@t)}\mapsto \mathtt{\exists t_3< t.}\mathtt{\forall t_3'\in[t_3, t].}\mathtt{(Pick\_Up\_C\_Colored\_Key@t_3)}\wedge\mathtt{(\neg Pick\_Up\_C\_Colored\_Key@t_3')}$, i.e., in time step $\mathtt{t_3}$ the agent picked up the key for this $C$ colored door for the last time; d) $\mathtt{(Unlock\_Door@t)}\mapsto \mathtt{\exists t_4< t.}\mathtt{\forall t_4'\in[t_4, t].}\mathtt{(Drop\_Ball@t_4)}\wedge\mathtt{(\neg Drop\_Ball @t_4')}$, i.e., in time step $\mathtt{t_4}$ the agent dropped a ball for the last time. After identifying those milestone time steps, the SRM rewards the behaviors at the corresponding time steps. The intuition behind such design pattern is that the reward function simply encourages all those behaviors if it is unclear what outcome those behavior will lead to; once the agent unlocks a door, the reward function is able to identify the milestone behaviors that are most closely related to the door unlocking outcome. Then we show the atomic predicates in the symbolic constraint for this task in Table.\ref{tab6_2}. The final symbolic constraint is $c=\bigwedge^{12}_{i=1}\mu_i$.

\begin{table*}[ht]
        \centering
        \begin{tabular}{l|l|l}
            \textbf{Properties} & \textbf{(Retaional) Predicates} & \textbf{(Non-Relational) Predicates}\\
            $[\mu_1]$Reward picking up target & $\mathtt{\bigwedge\limits^{12}_{id=2}(\mathtt{?_{id}}\leq \mathtt{?_1})}$ & $\mathtt{?_1\geq 0}$ \\
            $[\mu_2]$Reward finding target& $\mathtt{ \mathtt{?_2}\geq ?_4+?_5-2 ?_3}$ & $\mathtt{?_2}\geq 0$\\
            $[\mu_3]$Reward opening door & $\mathtt{ \mathtt{?_3}\leq 0}$ & $\mathtt{?_3}\leq 0$\\
            $[\mu_4]$Reward opening door & $\mathtt{ \mathtt{?_4}\geq 0}$ & $\mathtt{?_4}\geq 0$\\
            $[\mu_5]$Reward unlocking door & $\mathtt{\mathtt{?_{5}}\geq \sum\limits^{10}_{id=6}\mathtt{?_{id}}}$ & $\mathtt{?_5}\geq 0$\\
            $[\mu_6]$Penalize meaningless move & $\mathtt{\mathtt{?_3}\leq0}$ & $\mathtt{?_3}\geq 0$\\
            $[\mu_7]$Penalize picking up used key &$\mathtt{?_{11}+?_{12}}\leq 0$ & $\mathtt{?_{11}}\leq 0$\\
            $[\mu_8]$Reward opening box & $\mathtt{\mathtt{?_6}\geq0}$ & $\mathtt{?_6}\geq 0$\\
            $[\mu_9]$Reward picking up ball & $\mathtt{\mathtt{?_7}\geq 0}$ & $\mathtt{?_7}\geq 0$\\
            $[\mu_10]$Reward picking up key &$\mathtt{\mathtt{?_8}\geq 0}$& $\mathtt{?_8}\geq 0$\\
            $[\mu_{11}]$Reward dropping ball &$\mathtt{\mathtt{?_{9}}\geq 0}$& $\mathtt{?_9}\geq 0$\\
            $[\mu_{12}]$Reward dropping used key & $\mathtt{?_{12}}\geq 0$ & $\mathtt{?_{12}}\geq 0$
        \end{tabular}
        \caption{The correspondence between properties and predicates for the SRM of ObstructedMaze task in Fig.\ref{fig3_2}}
        \label{tab6_2}
    \end{table*} 
Note that the stored reward is not to be confused with the reward output at the present time. The syntax of sequencing in the $\mathtt{Hindsignt}$ code block depends on the language of the $r$ term specified in the background theory.  For the other two SRMs annotated by SRM2 and SRM3 as mentioned earlier, we remove the $\mathtt{Hindsight}$ block. Especially, in SRM2, we restrict that door unlocking and door opening behaviors are rewarded if only the accumulated rewards gained from those two behaviors do not exceed $\mathtt{?_2}$. Otherwise, none of the behaviors correlated with the self-looping transitions at state ``$\mathtt{Before\_Seeing_\_the\_Target}$" in Fig.\ref{fig3_2} will ever be rewarded. A possible reason for the policies trained by SRM1 do not generalize well in larger environment is that due to the hindsight reward modification, the reward output is too sparse in the large environment for the agent to learn. As shown by the experimental results of SRM2 and SRM3, once the $\mathtt{Hindsight}$ block is removed, the training performance in large environment is improved.

\subsection{Training details}
\begin{itemize}
\item \textbf{Training Overhead}. We note that all the designed SRMs require checking hindsight experiences, or maintaining memory or other expensive procedures. However, line 5 of Algorithm 1 requires running all $K$ candidate programs on all $m$ sampled trajectories, which may incur a substantial overhead during training. Our solution is that, before sampling any program as in line 5 of Algorithm 1, we evaluate the result of $[\![\srm]\!](\tau_{A,i})$, which keeps holes $\mathtt{\textbf{?}}$ unassigned, for all the $m$ trajectories. By doing this, we only need to execute the expensive procedures that do not involve the holes once, such as the counter $\mathtt{\#OPENDOOR}$ and the reward modification steps in the $\mathtt{Hindsight}$ block in Fig.\ref{fig3_2}. Then we use $q_\varphi$ to sample $K$ hole assignments $\mathtt{\{\mathtt{\bf{h}}_{k}\}}^K_{k=1}$ from $ {\textbf{H}}$ and feed them to $\{[\![\srm]\!](\tau_{A,i})\}^m_{i=1}$ to obtain $\{\{[\![l_k:=\srm[\mathtt{\bf{h}}_k/\mathtt{\textbf{?}}]]\!](\tau_{A,i})\}^m_{i=1}\}^K_{k=1}$. By replacing line 2 and line 5 with those two steps in Algorithm 1, we significantly reduce the overhead. 
\item \textbf{Supervised Learning Loss}. In Algorithm 1, a supervised learning objective $J_{cons}$ is used to penalize any sampled hole assignment for not satisfying the symbolic constraint. In practice, since our sampler $q_\varphi$ directly outputs the mean and log-variance of a multivariate Gaussian distribution for the candidate hole assignments, we directly evaluate the satisfaction of the mean. Besides, as mentioned earlier, in our experiments we only consider symbolic constraint as a conjunction of atomic predicates, e.g., $c=\wedge^n_{i=1}\mu_i$ with each $\mu_i$ only concerning linear combinations of the holes, we reformulated each $\mu_i$ into a form $u_i(\mathtt{\textbf{?}})\leq 0$ where $u_i$ is some linear function of the holes $\mathtt{\textbf{?}}$. We make sure that $(u_i(\mathtt{\bf{h}})\leq 0) \leftrightarrow ([\![\mu_i]\!](\mathtt{\bf{h}})=\top)$ for any hole assignment $\mathtt{\bf{h}}$. After calculating each $u_i(\mathtt{\bf{h}})$, which is now a real number, we let  $J_{cons}(q_\varphi)$ be a negative binary cross-entropy loss for $Sigmoid(ReLU([u_i(\mathtt{\bf{h}}), \ldots, u_n(\mathtt{\bf{h}})]^T)))$ with $0$ being the ground truth. This loss penalizes any $\mathtt{\bf{h}}$ that makes $u_i(\mathtt{\bf{h}})>0$. In this way $J_{cons}(q_\varphi)$ is differentiable w.r.t $\varphi$. Besides, we retain the entropy term $\mathcal{H}(q_{\varphi})$ extracted from the KL-divergence to regularize the variance output by $q_\varphi$.

\noindent\textbf{Network Architectures}. Algorithm 1 involves an agent policy $\pi_\phi$, a neural reward function $f_\theta$ and a sampler $q_\varphi$. Each of the three is composed of one or more neural networks.
\begin{itemize}
    \item \textbf{Agent policy $\pi_\phi$}. We prepare two versions of actor-critic networks, a CNN version and an LSTM version. For the CNN version, we adopt the actor-critic network from the off-the-shelf implementation of AGAC~\cite{flet-berliac2021adversarially}. It has 3 convolutional layers each with 32 filters, 3$\times$3 kernel size, and a stride of 2. A diagram of the CNN layers can be found in~\cite{flet-berliac2021adversarially}. For the LSTM version, we concatenate 3 identically configured convolutional layers with a LSTM cell of 32-size state vector. The LSTM cell is then followed by multiple fully connected layers each to simulate the policy, value and advantage functions. While AGAC contains other components~\cite{flet-berliac2021adversarially}, the PPO agent solely consists of the actor-critic networks. 
    \item \textbf{Neural reward function $f_\theta$}. The network is recurrent. It has 3 convolutional layers each with 16, 32 and 64 filters, 2$\times$2 kernel size and a stride of 1. The last convolutional layer is concatenated with an LSTM cell of which the state vector has a size of 128. The LSTM cell is then followed by a 3-layer fully connected network where each hidden layer is of size 64. Between each hidden layer we use two $tanh$ functions and one Sigmoid function as the activation functions. The output of the Sigmoid function is the logit for each action in the action space $\mathcal{A}$. Finally, given an action in a state, we use softmax and a Categorical distribution output the log-likelihood for the given action as the reward.
    \item \textbf{Sampler $q_\varphi$}. The input to  $q_\varphi$ is a constant $[1, \ldots, 1]^T$ of size 20. The sampler is a fully-connected network with 2 hidden layers of size 64. The activation functions are both $tanh$.  Suppose that there are $|\textbf{?}|$ holes in the SRM. Then the output of  $q_\varphi$ is a vector of size no less than $2|\textbf{?}|$.  The $|\textbf{?}|$ most and the $|\textbf{?}|$  least significant elements in the output vector will respectively be used as the mean of the Gaussian and constitute a diagonal log-variance matrix. Besides, we let $q_\varphi$ to output a value as the constant reward for the dummy transitions. While we still return $0$ instead of this constant as the reward to the agent, we subtract every sampled $\mathtt{\bf{h}}$ with this constant to compute $[\![l]\!](\tau)$ for $J_{soft}$. This subtraction simulates normalizing $[\![l]\!](\tau)$ in order to match the outputs of $f_\theta$, which, as mentioned earlier, is always non-positive in order to match $\log \pi_E$.
    \end{itemize}
    \item{\textbf{Hyperparameters}}. 
Most of the hyperparameters that appear in Algorithm 1 are summarized as in Table.\ref{tab0}. All hyperparameters relevant to AGAC are identical as those in~\cite{flet-berliac2021adversarially} although we do not present all of them in Table.\ref{tab0} in order to avoid confusion. The hyperparameter $\eta$ is made large to heavily penalize $q_\varphi$ when its output violates the symbolic constraint $c$. Besides, we add an entropy term $\mathcal{H}(q_{\varphi})$ multiplied by $1e-2$ in addition to $J_{soft}$ and $J_{con}$ to regularize the variance output by $q_\varphi$. The item $Entropy$ refers to the multiplier for the entropy term $\mathcal{H}(q_{\varphi})$ as introduced earlier.
\end{itemize}

\begin{table}[ht]
\centering
\begin{tabular}{l|l}
    \textbf{Parameter} & \textbf{Value} \\
    \# Epochs & 4 \\
    \# minibatches ($\pi_\phi$) & 8\\
    \# batch size ($f_\theta, q_\varphi$) & 128\\
    \# frames stacked (CNN $\pi_\phi$) & 4 \\
    \# reccurence (LSTM $\pi_\phi$) & 1\\
    \# recurrence ($f_\theta$) & 8\\
    Discount factor $\gamma$ & 0.99\\
    GAE parameter $\lambda$ & 0.95\\
    PPO clipping parameter $\epsilon$ & 0.2\\
    $K$ & 16\\
    $\alpha$ & 0.001\\
    $\beta$ & 0.0003\\
    $\eta$ & 1.e8\\
    $Entropy$ & 1.e-2
\end{tabular}
\caption{Hyperparameters used in the training processes}
\label{tab0}
\end{table}

\subsection{Derivation Of the Objective Functions}
First, we derive the lower-bound of $\log p(0_A,1_E|\pi_A, E, l)$ in Eq.\ref{eq4_3_3} as follows.
\begin{eqnarray}
&&\log p(0_A,1_E|\pi_A, E, l)\nonumber\\
&=&\log \sum_{\tau_A,\tau_E} p(\tau_A|\pi_A)p(\tau_E|E)\iint\limits_{f_{\tau_A},f_{\tau_E}} p(0_A|\tau_A; \pi_A, f_{\tau_A})p(1_E|\tau_E; \pi_A, f_{\tau_E})\nonumber\\
&&\qquad\qquad\qquad\qquad \qquad \qquad \qquad \qquad \qquad \qquad p(f_{\tau_E}| \tau_E; l)p(f_{\tau_A}|\tau_A; l)\nonumber\\
&\geq &\underset{\substack{\tau_A\sim\pi_A\\\tau_E\sim E}}{\mathbb{E}}\Big[\log \iint\limits_{f_{\tau_A},f_{\tau_E}} p(0_A|\tau_A; \pi_A, f_{\tau_A})p(1_E|\tau_E; \pi_A, f_{\tau_E}) p(f_{\tau_E}|\tau_E; l)p(f_{\tau_A}|\tau_A; l)\Big]\quad  \nonumber\\
&=&\underset{\substack{\tau_A\sim\pi_A\\\tau_E\sim E}}{\mathbb{E}}\Big[ \log\iint\limits_{f_{\tau_A},f_{\tau_E}} p(0_A|\tau_A;\pi_A,f_{\tau_A})p(1_E|\tau_E;\pi_A,f_{\tau_E})  p(f_{\tau_A}|\tau_A;l) p(f_{\tau_E}|\tau_E;l)\nonumber\\
&&\qquad \qquad \qquad \qquad \qquad \qquad \qquad \qquad \qquad \qquad \frac{p(f_{\tau_A}|\tau_A; f)p(f_{\tau_E}|\tau_E; f)}{p(f_{\tau_A}|\tau_A; f)p(f_{\tau_E}|\tau_E; f)}\Big]\nonumber\\
&\geq &\underset{f}{\max}\ \underset{\substack{\tau_A\sim\pi_A\\\tau_E\sim E}}{\mathbb{E}}\Big[ \underset{\substack{f_{\tau_A}\sim p(\cdot|\tau_A;f)\\ f_{\tau_E}\sim p(\cdot|\tau_E;f))}}{\mathbb{E}}\Big( \log p(0_A|\tau_A;\pi_A,f_{\tau_A}) p(1_E|\tau_E; \pi_A, f_{\tau_E})\nonumber\\
&& \qquad \qquad \qquad \qquad \qquad \qquad \qquad \qquad \qquad \qquad  \frac{p(f_{\tau_A}|\tau_A; l)p(f_{\tau_E}|\tau_E; l)}{p(f_{\tau_A}|\tau_A; f)p(f_{\tau_E}|\tau_E; f)}\Big)\Big] \nonumber\\
&=&\underset{f}{\max} \underset{{\epsilon\sim \mathcal{N}}}{\mathbb{E}}\Big[J_{adv}(D_\epsilon)\Big]-\quad  \underset{\mathclap{\tau\sim \pi_A, E}}{\mathbb{E}}\quad \Big[D_{KL}(p_{f(\tau)}||p_{l(\tau)})\Big]\nonumber
\end{eqnarray}

We justify the usage of the stochastic version $\mathbb{E}_{\epsilon\sim\mathcal{N}(0,1)}[J_{adv}(D_\epsilon)]$, rather than the conventional generative adversarial objective $J_{adv}(D)$, by showing that one of the saddle point of $\underset{\pi_A}{\min}\ \underset{f}{\max}\ \mathbb{E}_{\epsilon\sim\mathcal{N}(0,1)}[J_{adv}(D_\epsilon)]$ is attained when $f\equiv\log \pi_E\equiv\log\pi_A$ where we write $\pi_E$ in proxy of $E$ by assuming that the distribution of state-action pairs satisfies $p(s,a|\pi_E)\equiv p(s,a|E)$.
\begin{theorem}
Given a $\pi_E$,  $\underset{\pi_A}{\min}\ \underset{f}{\max}\ \mathbb{E}_{\epsilon\sim \mathcal{N}(0,1)}\big\{\mathbb{E}_{(s,a)\sim \pi_E}\big[\log D_{\epsilon}(s,a)]+\mathbb{E}_{(s,a)\sim \pi_A}\big[\log (1 - D_{\epsilon}(s,a)\big]\big\}$, where $D_{\epsilon}(s,a):=\frac{\exp(f(s,a)+\epsilon)}{\exp(f(s,a)+\epsilon)+\pi_A(a|s))}$,  is optimal when $f\equiv \log \pi_E\equiv\log\pi_A$.
\end{theorem}
\begin{proof}
Firstly, we consider optimizing $f$ under the condition of $\pi_A\equiv \pi_E$. Inspired by the proof of optimality condition of Generative Adversarial Nets in~\cite{goodfellow2014generative}, we introduce two variables $x_{s,a}=y_{s,a}\in(0, 1]$ to simulate $p(s,a|\pi_A)=p(s,a|\pi_E)$ for any $s,a\in\mathcal{S\times A}$. Then we prove that  $\mathbb{E}_{\epsilon\sim\mathcal{N}(0,1)}\big[x_{s,a}\log\frac{\hat{x}_{s,a}\cdot \exp(\epsilon)}{\hat{x}_{s,a}\cdot\exp(\epsilon)+y_{s,a}}+y_{s,a}\log \frac{y_{s,a}}{\hat{x}_{s,a}\cdot\exp(\epsilon)+y_{s,a}}\big]$ as a function of $\hat{x}_{s,a}$ has x stationary point at $\hat{x}_{s,a}=x_{s,a}$ by computing its gradient w.r.t $\hat{x}_{s,a}$ as follows.
\begin{eqnarray}
&&\nabla_{\hat{x}_{s,a}}\mathbb{E}_{\epsilon\sim\mathcal{N}(0,1)}\big[x_{s,a}\log\frac{\hat{x}_{s,a}\cdot \exp(\epsilon)}{\hat{x}_{s,a}\cdot\exp(\epsilon)+y_{s,a}}+y_{s,a}\log \frac{y_{s,a}}{\hat{x}_{s,a}\cdot\exp(\epsilon)+y_{s,a}}\big]\nonumber\\
&=&\mathbb{E}_{\epsilon\sim\mathcal{N}(0,1)}\big[x_{s,a}\cdot\frac{\hat{x}_{s,a}\cdot\exp(\epsilon)+y_{s,a}}{\hat{x}_{s,a}\cdot \exp(\epsilon)}\cdot \frac{ \exp(\epsilon)(\hat{x}_{s,a}\cdot\exp(\epsilon)+y_{s,a})-\hat{x}_{s,a}\exp(2\epsilon)}{(\hat{x}_{s,a}\cdot\exp(\epsilon)+y_{s,a})^2}+\nonumber\\
&&\qquad\qquad\qquad\qquad\qquad\qquad\qquad\qquad y_{s,a}\cdot \frac{\hat{x}_{s,a}\cdot\exp(\epsilon)+y_{s,a}}{y_{s,a}}\cdot\frac{-y_{s,a}\exp(\epsilon)}{(\hat{x}_{s,a}\cdot\exp(\epsilon)+y_{s,a})^2}\big]\nonumber\\
&=&\mathbb{E}_{\epsilon\sim\mathcal{N}(0,1)}\big[\frac{x_{s,a}y_{s,a}/\hat{x}_{s,a}-y_{s,a}\exp(\epsilon)}{\hat{x}_{s,a}\cdot\exp(\epsilon)+y_{s,a}}\big]\label{app1}
\end{eqnarray}
When $\hat{x}_{s,a}=x_{s,a}=y_{s,a}$, Eq.\ref{app1} equals $\mathbb{E}_{\epsilon\sim\mathcal{N}(0,1)}\big[\frac{1-\exp(\epsilon)}{1+\exp(\epsilon)}\big]$. Note that the probabilities of sampling $\epsilon$ and $-\epsilon$ equal each other, and $\frac{1-\exp(\epsilon)}{1+\exp(\epsilon)}=-\frac{1-\exp(-\epsilon)}{1+\exp(-\epsilon)}$.  Hence, $\mathbb{E}_{\epsilon\sim\mathcal{N}(0,1)}\big[\frac{1-\exp(\epsilon)}{1+\exp(\epsilon)}\big]=0$. It can trivially proved that the gradient of Eq.\ref{app1} w.r.t $\hat{x}$ is non-positive. Therefore, $f\equiv \log\pi_E$ is a local maximum. 

Next, we consider optimizing $\pi_A$ under the condition of $f\equiv \log \pi_E$. We denote $p(s,a|\pi_E)$ and $p(s,a|\pi_A)$ for any $(s,a)\in\mathcal{S\times A}$ as $x_{s,a},y_{s,a}\in(0,1]$ for short. Then we show that $\mathbb{E}_{\epsilon\sim\mathcal{N}(0,1)}\big[\sum\limits_{(s,a)\in\mathcal{S\times A}}x_{s,a}\log\frac{x_{s,a}\cdot \exp(\epsilon)}{x_{s,a}\cdot\exp(\epsilon)+y_{s,a}}+\frac{y_{s,a}}{x_{s,a}+y_{s,a}}\log \frac{y_{s,a}}{x_{s,a}\cdot\exp(\epsilon)+y_{s,a}}\big]$ as a function of $\{y_{s,a}|(s,a)\in \mathcal{S\times A}\}$ s.t. $\sum\limits_{(s,a)\in\mathcal{S\times A}}y_{s,a}=1$ has a stationary point at $x_{s,a}\equiv y_{s,a}$ by computing the gradient of the Lagrangian of this constrained function as follows.
\begin{eqnarray}
L&=&\underset{\epsilon\sim\mathcal{N}(0,1)}{\mathbb{E}}\big[\sum\limits_{(s,a)\in\mathcal{S\times A}}x_{s,a}\log\frac{x_{s,a}\cdot \exp(\epsilon)}{x_{s,a}\cdot\exp(\epsilon)+y_{s,a}}+y_{s,a}\log \frac{y_{s,a}}{x_{s,a}\cdot\exp(\epsilon)+y_{s,a}}\big]-\nonumber\\
&&\qquad\qquad\qquad\qquad\qquad\qquad\qquad\qquad\qquad\qquad\qquad\qquad\lambda(\sum\limits_{(s,a)\in\mathcal{S\times A}}y_{s,a} -1)\nonumber\\
\nabla_{y_{s,a}}L&=&\underset{\epsilon\sim\mathcal{N}(0,1)}{\mathbb{E}}\big[\log\frac{y_{s,a}}{x_{s,a}\cdot\exp(\epsilon)+y_{s,a}} + \frac{x_{s,a}(\exp(\epsilon)-1)}{x_{s,a}\cdot\exp(\epsilon)+y_{s,a}}\big]-\lambda=0\nonumber\\
&\Rightarrow&\forall(s,a)\in\mathcal{S\times A}.\underset{\epsilon\sim\mathcal{N}(0,1)}{\mathbb{E}}\big[\log\frac{y_{s,a}}{x_{s,a}\cdot\exp(\epsilon)+y_{s,a}} + \frac{x_{s,a}(\exp(\epsilon)-1)}{x_{s,a}\cdot\exp(\epsilon)+y_{s,a}}\big]=\lambda\label{app2}\\
\nabla_\lambda L&=&\sum\limits_{(s,a)\in\mathcal{S\times A}}y_{s,a} -1=0\label{app3}
\end{eqnarray}
Suppose that $\lambda=\mathbb{E}_{\epsilon\sim\mathcal{N}(0,1)}[\log \frac{1}{1+\exp(\epsilon)}]$ and $x_{s,a}=y_{s,a}$ holds for any $(s,a)\in\mathcal{S\times A}$. Then both Eq.\ref{app2} and Eq.\ref{app3} hold. Hence, $x_{s,a}\equiv y_{s,a}$ is a stationary point. It is also trivially provable that the gradient of Eq.\ref{app2} w.r.t $y_{s,a}$ is non-negative. Therefore, $\pi_A\equiv \pi_E$ is a local minimum. In conclusion, $f\equiv \log\pi_E\equiv \log\pi_A$ is a saddle point.
\end{proof}

We derive the lower-bound of the ELBO in \textit{Section 5} as follows.
\begin{eqnarray}
&& ELBO(q)= D_{KL}\Big[q (l)|| p(l)\Big] +  \underset{\mathclap{l\sim q}}{\mathbb{E}}\ \Big[ \log p(0_A, 1_E|\pi_A, E, l)\Big]\nonumber\\
&=&\underset{l\sim q}{\mathbb{E}}\Big\{\log \underset{\substack{\\ \tau_A\sim\pi_A\\\tau_E\sim E}}{\mathbb{E}}\Big[\iint\limits_{f_{\tau_A},f_{\tau_E}} p(0_A|\tau_A; \pi_A, f_{\tau_A})p(1_E|\tau_E; \pi_A, f_{\tau_E}) p(f_{\tau_E}|\tau_E; l)p(f_{\tau_A}|\tau_A; l)\Big]\Big\} - \nonumber\\
&&\qquad\qquad \qquad\qquad \qquad\qquad \qquad\qquad \qquad\qquad  \qquad\qquad \qquad\qquad\qquad D_{KL}\Big[q (l)|| p(l)\Big]\nonumber\\
&\geq &\underset{f}{\max}  \underset{{\epsilon\sim \mathcal{N}}}{\mathbb{E}}\Big[ J_{adv}(D_\epsilon)\Big]-\underset{\mathclap{\substack{l\sim q\\ \tau\sim \pi_A, E}}}{\mathbb{E}}\quad \Big[D_{KL}(p_{f(\tau)}||p_{l(\tau)})\Big]- D_{KL}\Big[q (l)|| p(l)\Big]\nonumber\nonumber\\
&=&J_{soft}(q, f) + J_{con}(q)\nonumber
\end{eqnarray}


\end{document}